\documentclass{article}

\usepackage{arxiv}

\usepackage[utf8]{inputenc} 
\usepackage[T1]{fontenc}    
\usepackage{hyperref}       
\usepackage{url}            
\usepackage{booktabs}       
\usepackage{amsfonts}       
\usepackage{nicefrac}       
\usepackage{microtype}      
\usepackage{lipsum}		
\usepackage{graphicx}
\usepackage{natbib}
\usepackage{doi}

\usepackage{amsmath}
\usepackage[noabbrev,capitalise]{cleveref}
\usepackage{amsfonts}
\usepackage{amssymb}
\usepackage{array}
\usepackage{amsthm}
\usepackage{algorithm}
\usepackage{algorithmic}
\usepackage{mathrsfs}

\newtheorem{assumption}{Assumption}
\newtheorem{proposition}{Proposition}
\newtheorem{theorem}{Theorem}
\newtheorem{definition}{Definition}

\newcommand{\kron}{\boxtimes}  
\makeatletter
\newcommand{\argmin}{\mathop{\text{~argmin~}}}

\DeclareMathOperator{\vecn}{vec}

\newcommand{\tens}[1]{\boldsymbol{\mathcal{#1}}}

\newcommand{\tT}{\tens{T}}

\title{Efficient Algorithms for Regularized Nonnegative Scale-invariant Low-rank Approximation Models}

\date{} 					

\author{ \href{https://orcid.org/0000-0001-8319-8566}{\includegraphics[scale=0.06]{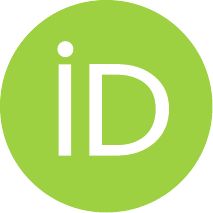}\hspace{1mm}Jeremy E. Cohen}\thanks{Jeremy Cohen is supported by ANR JCJC LoRAiA
ANR-20-CE23-0010.} \\
	CREATIS UMR 5220, U1294\\
	INSA Lyon, UCBL, CNRS, Inserm\\
	Lyon, France \\
	\texttt{jeremy.cohen@cnrs.fr} \\
	\And
	\href{https://orcid.org/0000-0002-3313-1547}{\includegraphics[scale=0.06]{orcid.pdf}\hspace{1mm}Valentin~Leplat} \\
	Faculty of Computer Science and Engineering\\
	Innopolis University\\
	Innopolis, Russia \\
	\texttt{V.Leplat@innopolis.ru} \\
}

\renewcommand{\headeright}{ }
\renewcommand{\undertitle}{ }
\renewcommand{\shorttitle}{EFFICIENT ALGORITHMS FOR REGULARIZED NONNEGATIVE SCALE-INVARIANT LRA}

\hypersetup{
pdftitle={Efficient Algorithms for Regularized Nonnegative Scale-invariant Low-rank Approximation Models},
pdfsubject={Low-rank approximations},
pdfauthor={Jérémy E. Cohen, Valentin Leplat},
pdfkeywords={Constrained Low-rank approximations, non-convex optimization, sparse approximation},
}

\begin{document}
\maketitle

\begin{abstract}

Regularized nonnegative low-rank approximations, such as sparse Nonnegative Matrix Factorization or sparse Nonnegative Tucker Decomposition, form an important branch of dimensionality reduction models known for their enhanced interpretability. From a practical perspective, however, selecting appropriate regularizers and regularization coefficients, as well as designing efficient algorithms, remains challenging due to the multifactor nature of these models and the limited theoretical guidance available.
This paper addresses these challenges by studying a more general model, the Homogeneous Regularized Scale-Invariant model. We prove that the scale-invariance inherent to low-rank approximation models induces an implicit regularization effect that balances solutions. This insight provides a deeper understanding of the role of regularization functions in low-rank approximation models, informs the selection of regularization hyperparameters, and enables the design of balancing strategies to accelerate the empirical convergence of optimization algorithms.

Additionally, we propose a generic Majorization-Minimization (MM) algorithm capable of handling $\ell_p^p$-regularized nonnegative low-rank approximations with non-Euclidean loss functions, with convergence guarantees. Our contributions are demonstrated on sparse Nonnegative Matrix Factorization, ridge-regularized Nonnegative Canonical Polyadic Decomposition, and sparse Nonnegative Tucker Decomposition.

\end{abstract}


\section{Introduction}\label{sec:introduction}

\subsection{Context}

Unsupervised machine learning heavily relies on Low-Rank Approximation (LRA) models, which are fundamental tools for dimensionality reduction. These models exploit the inherent multilinearity in datasets represented as matrices or tensors. One of the most commonly used LRA methods is Principal Component Analysis (PCA), which identifies orthogonal principal components (\textit{i.e.}, rank-one matrices) within a dataset. However, the orthogonality constraint can be detrimental when a physical interpretation of the principal components is required~\cite{harshman1970foundations, Comon2010Handbook}.

Over the past two decades, researchers have explored various constrained LRA models with enhanced interpretability. Notably, Nonnegative Matrix Factorization (NMF) replaces PCA’s orthogonality constraint with an elementwise nonnegativity constraint, allowing for a part-based decomposition of the data~\cite{Lee1999Learning,gillis2020bk}. For tensorial data, models such as Nonnegative Canonical Polyadic Decomposition (NCPD) have received significant attention. The Nonnegative Tucker Decomposition (NTD), a multilinear dictionary-learning model, has also gained popularity~\cite{morup2008algorithms}. Unconstrained LRA models face interpretation challenges due to inherent rotation ambiguities. Consequently, the interpretability of LRA models depends largely on the constraints or regularizations imposed on their parameters.

While the theoretical and algorithmic aspects of regularized Low-Rank Approximation have been extensively explored, several practical challenges remain. As regularized LRA are inherently multi-factor problems, these challenges include selecting appropriate regularization functions and hyperparameters for each parameter matrix, as well as developing efficient algorithms for non-Euclidean loss functions that offer convergence guarantees. In this work, we address these challenges by introducing a broader framework, which we term the Homogeneous Regularized Scale-Invariant (HRSI) problem. Through this framework, we examine the implicit regularization effects arising from the scale invariance of explicitly regularized LRA.

The following sections of this introduction formalize the general framework of our study, outline its current limitations, summarize our contributions, and provide an overview of the manuscript's structure.

\subsection{Notations}
Vectors, matrices and tensors are denoted $x$, $X$ and $\mathcal{X}$ respectively. The set $\mathbb{R}_+^{\times_{i=1}^{d} m_i}$ contains $d$-order tensors with dimensions $m_1\times \ldots \times m_d$ and nonnegative entries. Product $\mathcal{X} \times_i U$ is the n-mode product along the $i$-th mode between tensor $\mathcal{X}$ and matrix $U$.  The Frobenius norm of matrix or tensor $X$ is denoted as $\|X\|_F$. The entry indexed by $(i,j)$ and the column indexed by q in matrix $X$ are respectively written as $X[i,j]$ and $X[:,q]$. $\{X_i\}_{i\leq n}$ is a set of matrices or tensors $X_i$ indexed by $i$ from $1$ to $n$. The set of vectors $x$ such that $g(x)<+\infty$ is denoted by $\text{dom}(g)$, and $\eta_{\mathcal{C}}$ is the characteristic function of convex set $\mathcal{C}$. Matrix $I_r$ denotes the identity matrix of size $r\times r$,
$x\otimes y$ the tensor (outer) product, $A \odot B$ and $A \kron B$ respectively the Hadamard and Kronecker product between two tensors $A$ and $B$ and   $X^\alpha$ the elementwise power of tensor $X$ with exponent $\alpha$.
More details on multi-linear algebra can be found in \cite{Kolda2009Tensor}.

\subsection{The Homogeneous Regularized Scale-Invariant model}\label{sec:HRSI_intro}
Consider the following optimization problem
\begin{equation}\label{eq:hrsi}
  \underset{\forall i\leq n,\; X_i\in\mathbb{R}^{m_i\times r}}{\argmin}f(\{X_i\}_{i\leq n}) + \sum_{i=1}^{n} \mu_{i} \sum_{q=1}^{r} g_i(X_i[:,q]),
\end{equation}
where $f$ is a continuous map from the Cartesian product $\times_{i=1}^{n} \mathbb{R}^{m_i\times r} \cap \text{dom}(g_i)$ to $\mathbb{R}_+$, $\{\mu_i\}_{i\leq n}$ is a set of positive regularization hyperparameters, and $\{g_i\}_{i\leq n}$ is a set of lower semi-continuous regularization maps from $\mathbb{R}^{m_i}$ to $\mathbb{R}_+$. We assume that the total cost is coercive\footnote{Coercivity is satisfied by the regularization functions considered in this work.}, ensuring the existence of a minimizer. Furthermore, we assume the following assumptions hold:

\begin{itemize}
    \item [\textbf{A1}] The function $f$ is invariant to column-wise scaling of the parameter matrices. For any sequence of positive diagonal matrices $\{\Lambda_i\}_{i\leq n}$,
    \begin{equation}\label{eq:h1finv}
        f(\{X_i\Lambda_i\}_{i\leq n}) = f(\{X_i\}_{i\leq n})\;\;\text{if}\; \prod_{i=1}^{n}\Lambda_i = I_r,
    \; \Lambda_i>0.
    \end{equation}
    In other words, scaling any two columns $X_{i_1}[:,q]$ and $X_{i_2}[:,q]$ inversely proportionally has no effect on the value of $f$. 
    \item [\textbf{A2}] Each regularization function $g_i$ is positive homogeneous of degree $p_i$, that is for all $i\leq n$ there exists $p_i>0$ such that for any $\lambda > 0$ and any $x\in \mathbb{R}^{m_i}$,
    \begin{equation}
        g_i(\lambda x) = \lambda^{p_i} g_i(x).
    \end{equation}
    This property holds in particular for any $\ell_p^p$ norm in the ambient vector space.
    \item [\textbf{A3}] Each function $g_i$ is positive-definite, meaning that for any $i\leq n$, $g_i(x)=0$ if and only if $x$ is null.\footnote{This assumption limits the framework's applicability, excluding positive homogeneous regularizations like Total Variation.} 
\end{itemize}
We refer to the approximation problem~\eqref{eq:hrsi} with assumptions \textbf{A1}, \textbf{A2} and \textbf{A3} the Homogeneous Regularized Scale-Invariant problem (HRSI). 

HRSI is quite general and encompasses many instances of regularized LRA problems. To illustrate this, consider the computation of solutions to a (double) sparse NMF (sNMF) model, where an $\ell_1$-norm regularization promotes sparsity in the matrices $X_1$ and $X_2$ and the Kullback-Leibler divergence serves as the loss function. This problem can be formalized as the following optimization problem:
\begin{equation}\label{eq:sNMF1}
    \underset{X_1 \geq 0, X_2\geq 0}{\argmin} \text{KL}(M, X_1X_2^T) + \mu_1\|X_1\|_1 + \mu_2 \|X_2\|_1,
\end{equation}
where $M\in\mathbb{R}^{m_1\times m_2}$ is the input data matrix. This formulation for sNMF fits within the HRSI framework because the cost function is scale-invariant and the regularizations are homogeneous, positive-definite, and columnwise separable. Specifically, sNMF as defined in Equation\eqref{eq:sNMF1} is an HRSI problem with $f(\{X_i\}_{i\leq n})=\text{KL}(M, X_1X_2^T)$, $g_1(x) = \|x\|_1 + \eta_{\mathbb{R}_+^{m_1}}(x)$ and $g_2(x) = \|x\|_1 + \eta_{\mathbb{R}_+^{m_2}}(x)$ where $\|x\|_1$ denotes the $\ell_1$ norm in $\mathbb{R}^{m_i}$ applied to each column of $X_i$. 


\subsection{Current limitations}
Let us consider conventional variational problems, focusing on the well-known LASSO problem~\cite{tibshiraniRegressionShrinkageSelection2011}. A single regularization hyperparameter is used for this problem, although more complex models have been proposed~\cite{bogdanSLOPEADAPTIVEVARIABLE2015}. Representer theorems state that solutions to such problems are combinations of extreme points within the regularization unit ball~\cite{boyerRepresenterTheoremsConvex2019,unserUnifyingRepresenterTheorem2021}. A maximum value for the regularization hyperparameter exists beyond which all solutions become degenerate~\cite{tibshiraniLassoProblemUniqueness2013}, and it is simple to compute. Additionally, homotopy methods establish a connection between the regularity of solutions and the values of the regularization hyperparameters~\cite{efronLeastAngleRegression2004}. In simpler terms, the relationship between the penalty function choice, the regularization hyperparameter, and the solution set is well understood in standard cases. 
For scale-invariant multifactor models however, there are no known representer theorems or maximum regularization hyperparameter characterizing the regularity of the solutions to the best of our knowledge. The presence of multiple regularization hyperparameters also complicates model selection and evaluation.

Furthermore, the design of algorithms for HRSI with non-Euclidean loss functions --- particularly in the context of nonnegative parameter matrices and beta-divergence loss~\cite{cichockiCsiszarDivergencesNonnegative2006,fevotte2011algorithms} --- has been underexplored. The design of efficient optimization algorithms for NMF with non-Euclidean loss is a relatively mature area. 
Many algorithms, including the popular Majorization-Minimization (MM) algorithm, have been proposed to solve the unregularized problem~\cite{Lee1999Learning, fevotte2011algorithms}, but these methods do not trivially extend for regularized factorization. Many regularized LRA models rely on normalization constraints to simplify the tuning of multiple hyperparameters and mitigate degeneracy (see~\cref{sec:illposed}). This approach makes the problem more challenging to solve due to nonconvexity issues. To the best of our knowledge, MM algorithms have no known convergence guarantees in this context.

\subsection{Contributions and Impact}
Our main contribution is to show that solutions to the HRSI problem inherently benefit from implicit regularization due to their scale-invariance. We show the following:
\begin{itemize}
    \item Solutions to HRSI are balanced with respect to the regularization functions, as detailed in~\cref{sec:scale-optimality}.
    \item Explicit regularization on matrices $X_i$ induces implicit regularization on the rank-one components, providing insight into the properties of the solutions. For instance, applying ridge regularization to all parameter matrices in a ridge CPD implicitly enforces a group LASSO regularization on the rank-one components, revealing the tensor CP rank.
    \item The hyperparameters for each matrix $X_i$ can be chosen simultaneously, as discussed in~\cref{sec:choicemu}.
\end{itemize}
These results extend the known properties of specific regularized LRA models. For further details, see~\cref{subsec:balancingPrinciple} for two examples,  and~\cref{sec:sota} for an overview of existing results.

On the front of algorithm design, focusing on the case where the loss function $f$ is a $\beta$-divergence, we make the following contributions. 
\begin{itemize}
    \item A general alternating MM Meta-algorithm is introduced in~\cref{sec_algorithms} which provably computes a stationary point of the regularized LRA problem and accommodates various regularization functions (in particular any $\ell_p^p$ regularization). In particular, 
    our proposed MM update rules for various combinations of $\beta$-divergences and $\ell_p^p$ regularizations are an original contribution, as they guarantee convergence. To the best of our knowledge, existing works typically propose heuristic adaptations of unregularized MM updates in the presence of regularizations, but without convergence guarantees for the iterates~\cite{leroux2015sparse}.
    \item This Meta-algorithm features a balancing strategy that scales the columns of the parameter matrices to minimize the regularization terms. We show formally on a toy problem and experimentally on several practical examples that this procedure improves the convergence speed for the iterates, which otherwise may be sublinear.
\end{itemize}

We illustrate these implicit regularization effects and the proposed Meta-algorithm through the lens of three models: sNMF, ridge-regularized NCPD (rNCPD), and sparse NTD (sNTD). These models are tested on simple synthetic and more complex real-world datasets in~\cref{sec:Showcases}.

These contributions have the potential to improve the usability of regularized LRA models across a wide range of applications, such as spectral unmixing in remote sensing~\cite{Bioucas-Dias2012Hyperspectral} or medical microscopy~\cite{careddaSeparableSpectralUnmixing2023}, music information retrieval~\cite{marmoret2020uncovering}, source separation in chemometric~\cite{Bro1998Multi}, neuroimaging~\cite{roaldAOADMMApproachConstraining2022} and antenna array processing~\cite{Sidiropoulos2000Parallel}. See~\cite{gillis2020bk, Kolda2009Tensor} for surveys about matrix and tensor decomposition applications, respectively. In these applications, regularization is often crucial for ensuring or enhancing the interpretability of the results. While CPD enjoys uniqueness properties under mild conditions~\cite{Kruskal1977Three}, both NMF and NTDare generally not unique without further regularization~\cite{Cichocki2009Nonnegative, gillis2020bk}. Homogeneous regularizations also reshape the optimization landscape by making the cost function coercive, thereby avoiding issues such as swamps that can arise in tensor decompositions~\cite{Lim2009Nonnegative, mohlenkampDynamicsSwampsCanonical2019}. This work guides how to choose such regularization functions, select individual regularization hyperparameters, and design efficient optimization algorithms with convergence guarantees for a broad range of loss functions, accommodating various noise models. Finally, we show that in many scenarios, explicit regularization of all matrices in regularized LRA leads to group sparsity on the components. This helps in selecting the rank, a significant practical issue with these models. This work also includes Python code\footnote{\url{https://github.com/vleplat/NTD-Algorithms}} providing off-the-shelf algorithms for sNMF, rCPD and sNTD.

\section{Scale invariance in HRSI leads to implicit balancing}\label{sec:Sparsity_scale_inva_models}

We begin by discussing the ill-posedness of HRSI, emphasizing the importance of regularizing all the modes. After addressing this issue, we move on to our main contribution: demonstrating that the scale-invariance in the cost function leads to optimal scaling with respect to the regularization, effectively balancing the solutions. An underlying optimization problem emerges that is essentially equivalent\footnote{We use ``essential equivalence'' to describe two optimization problems where any solution to one is also a solution to the other, up to a scaling that complies with assumption \textbf{A1}.} to HRSI, exhibiting an implicit regularization effect. This implicit HRSI formulation provides valuable insights for selecting regularization functions $\{g_i\}_{i\leq n}$ and regularization hyperparameters.

\subsection{Ill-posedness of partially regularized homogeneous scale-invariant problems}\label{sec:illposed}
In Equation~\eqref{eq:hrsi} which defines the HRSI problem, we assumed that the regularization hyperparameters are nonnegative. It is known in the literature~\cite{marminMajorizationMinimizationSparseNonnegative2023}
that if the parameter $\mu_j$ for a specific index $j$ is zero, the scale-indeterminacy of the loss leads to degenerate solutions for HRSI. A stronger result is shown in this work: a partially regularized HRSI problem yields the same optimal cost as an unregularized problem, but with unattainable solutions. In other words, the regularization functions do not change the minimum cost unless all parameter matrices are regularized.
\begin{proposition}\label{prop:ill_posed}
Let $\mu_j=0$ for at least one index $j\leq n$. Then  the infimum of the function 
\begin{equation}\label{eq:l1prop11}
f(\{X_i\}_{i\leq n}) + \sum_{i=1}^{n} \mu_{i} \sum_{q=1}^{r} g_i(X_i[:,q])
\end{equation}
is equal to $\underset{X_i\in\mathbb{R}^{m_i\times r}}{\inf}~f(\{X_i\}_{i\leq n})$. Moreover, the infimum is not attained unless f attains its minimum at $X_i=0$ for all $i\leq n$.
\end{proposition}
A proof is provided in~\cref{AppendB}.
In the following, we assume that all matrices $X_i$ are regularized, with associated regularization hyperparameter $\mu_i>0$.

\subsection{Optimal balancing of parameter matrices in HRSI}\label{sec:scale-optimality}
We now turn to the central discussion of this section: the effect of scaling invariance on regularization. Since the loss function $f$ is scale-invariant, it follows that one can minimize the whole cost by adjusting the scaling of the columns with respect to the regularization functions only. A closed-form solution for the optimal scaling derives from the homogeneity and the positive-definiteness of the regularization functions. This solution offers valuable insights into the implicit regularizations and invariances present within the HRSI framework.

\subsubsection{Optimal scaling of regularization terms}
Let $\{X_i\}_{i\leq n}$ be a set where all $X_i$ have nonzero columns, and $q\leq r$ be a column index. To simplify the notation, we let \(a_i=\mu_ig_i(X_i[:,q])\) where $a_i>0$ due to assumption \textbf{A3}. Consider the optimization problem obtained by considering the HRSI problem~\eqref{eq:hrsi} with respect to only the scales $\{\lambda_i\}_{i\leq n}$ of the $q$-th columns of all matrices $X_i$,
\begin{equation}\label{eq:scale_cost}
    \underset{\forall i\leq n,\;\lambda_i\geq 0}{\min} \sum_{i=1}^{n} \lambda_i^{p_i} a_i  \text{ such that } \prod_{i=1}^{n}\lambda_i = 1.
\end{equation}
The scale-invariant loss function $f$ and the regularizations over other columns $g_i(X_i[:,p])$ for $p\neq q$, which are constant with respect to a scaling of the $q$-th columns represented by the parameters $\{\lambda_i\}_{i\leq n}$, have been removed. Our objective is to determine the optimal scales.

One can observe that Problem~\eqref{eq:scale_cost} admits a unique optimal solution. 
The optimal scales can provably be computed from the weighted geometrical mean $\beta$ of $\{p_ia_i, \frac{1}{p_i} \}$ as follows: 
\begin{equation}
    \lambda_i^{\ast} = \left(\frac{\beta}{p_ia_i}\right)^{\frac{1}{p_i}}, \text{ with }  \beta = \left(\prod_{i\leq n} \left(p_ia_i\right)^{\frac{1}{p_i}}\right)^{\frac{1}{\sum_{i\leq n} \frac{1}{p_i}}}.
\end{equation}
Here, the set $\{ \lambda_i^\ast\}_{i\leq n}$ contains the solutions to the scaling problem~\eqref{eq:scale_cost}.

Since at optimality the balancing procedure should have no effect, given an optimal solution ${X_i^\ast}$ to HRSI, it holds that $\lambda_i^\ast=1$ for all $i\leq n$, where $a_i = \mu_ig_i(X_i[:,q]^\ast)$. In that case, we can infer that all additive terms in~\eqref{eq:scale_cost} must be equal up to a multiplicative factor. Specifically, for any $i\leq j\leq n$, we have
   $ p_ia_i = p_ja_j = \beta $.
We can now state the following proposition characterizing the solutions of HRSI.
\begin{proposition}\label{prop:opt_scal}
    If $\mu_i>0$ for all $i$, any solution $\{X_i^\ast\}_{i\leq n}$ to Problem~\eqref{eq:hrsi} satisfies $p_i\mu_ig_i(X_i^{\ast}[:,q]) = \beta_q$ for all $i\leq n$ and $q\leq r$, where 
    \begin{equation}\label{eq:def_beta}
        \beta_q = \left(\prod_{i\leq n} \left(p_i\mu_ig_i(X^\ast_i[:,q])\right)^{\frac{1}{p_i}}\right)^{\frac{1}{\sum_{i\leq n} \frac{1}{p_i}}} .
    \end{equation}
\end{proposition}
A proof for~\cref{prop:opt_scal} is provided in~\cref{app:balancing_eq}. 
If all regularization hyperparameters $\mu_i$ are nonzero but a column $X_{i_0}[:,q]$ is null for a given index $i_0\leq n$, by scaling to zero columns $X_i[:,q]$ for $i\neq i_0$, the penalty functions are trivially minimized with a constant cost. Therefore, the optimal balancing in the presence of a null column $X_{i_0}[:,q]$ is to set the corresponding columns $X_{i\neq i_0}[:,q]$ to zero as well. Hence,~\cref{prop:opt_scal} allows for solutions with zero columns.

\subsubsection{Balancing algorithm and essentially equivalent scale-free implicit HRSI}\label{subsec:balancingPrinciple}
In the previous section, we demonstrated that it is possible to compute the optimal scaling of parameter matrices to minimize the regularization terms while keeping the cost function $f$ constant. In simpler terms, at any specific point within the parameter space, we can reduce the cost function in \eqref{eq:hrsi} by applying a scaling procedure, which we refer to as the balancing procedure.

Formally, given a set of parameter matrices $\{X_i\}_{i\leq n}$ with nonzero columns, and assuming that all regularization hyperparameters $\mu_i$ are positive, the balancing procedure is defined as follows:
\begin{equation}\label{eq:scale_aast_final}
    \tilde{X}_i[:,q] = \left(\frac{\beta_q}{p_i\mu_ig_i(X_i[:,q])}\right)^{1/p_i} X_i[:,q],
\end{equation}
where $\beta_q$ is given by~\cref{prop:opt_scal}. The balancing procedure results in a new set of parameters $\{\tilde{X}_i\}_{i\leq n}$ which solves the optimal scaling problem
\begin{equation}\label{eq:scale_pb_big}
  \underset{\forall i\leq n,\; \Lambda_i\in\mathbb{R}_+^{r\times r}}{\min}f(\{X_i\Lambda_i\}_{i\leq n}) + \sum_{i=1}^{n} \mu_{i} \sum_{q=1}^{r} g_i(\Lambda_i[q,q]X_i[:,q])
  \text{ s.t. } \Lambda_i \text{ are diagonal and } \prod_{i\leq n} \Lambda_i = I_r.
\end{equation}
The balancing procedure is summarized in~\cref{alg:balancing-Algo}.

\algsetup{indent=2em}
\begin{algorithm}[ht!]
\caption{Balancing algorithm \label{alg:balancing-Algo}}
\begin{algorithmic}[1] 
\REQUIRE A collection of columns $\{X_i[:,q]\}_{i\leq n}$ for a fixed index $q\leq r$, regularization hyperparameters $\mu_i>0$, regularization functions $g_i$ and their homogeneity degree $p_i$.
\IF{$X_i[:,q]=0$ for any $i\leq n$}
    \STATE Set $X_i[:,q]=0$ for all $i\leq n$.
\ELSE
    \STATE Compute the geometric mean $\beta_q = \left(\prod_{i\leq n} \left(p_i\mu_ig_i(X_i[:,q])\right)^{\frac{1}{p_i}}\right)^{\frac{1}{\sum_{i\leq n} \frac{1}{p_i}}} $  
    \FOR{$i$=1 : $n$}
        \STATE Set $\tilde{X}_i[:,q] = \left(\frac{\beta_q}{p_i\mu_ig_i(X_i[:,q])}\right)^{1/p_i} \tilde{X}_i[:,q]$
    \ENDFOR
\ENDIF
\RETURN $\{X_i[:,q]\}_{i\leq n}$
\end{algorithmic}  
\end{algorithm} 

As noted earlier, a global solution $\{X^\ast_i\}_{i\leq n}$  of HRSI must adhere to the balancing condition. Consequently, the solution space of HRSI can be restricted to balanced solutions only.
The core insight of our work is that this restriction transforms the columnwise sum of regularization terms into a columnwise product of regularization.

Since all regularization terms evaluated at a balanced solution are equal up to factors $p_i$, the regularization terms
    $\sum_{q=1}^{r} \sum_{i=1}^{n} \mu_i g_i(X_i[:,q])$
after optimal balancing becomes $
   \sum_{i=1}^{n}\frac{1}{p_i} \sum_{q=1}^{r}  \beta_q$.
This expression can be rewritten as an implicit regularization term 
\begin{equation}\label{eq:implicit_reg}
   \tilde{\mu} \sum_{q=1}^{r}  \left(\prod_{i=1}^{n} g_i(X_{i}[:,q])^{\frac{1}{p_i}}\right)^{\frac{1}{\sum_{i=1}^{n} \frac{1}{p_i}}},
\end{equation}
where the penalty weight $\tilde{\mu}$ is defined as  
$\tilde{\mu} = \left(
\prod_{i=1}^{n}(p_i\mu_i)^{\frac{1}{p_i}}
    \right)^{\frac{1}{\sum_{i=1}^{n} \frac{1}{p_i}}}\left(\sum_{i=1}^{n}\frac{1}{p_i}\right).$ 
When all $p_i$-s are equal, the implicit regularization coefficient conveniently simplifies to $\tilde{\mu}= n\left(\prod_{i=1}^{n}\mu_i\right)^{\frac{1}{n}}$, and the regularization term simplify to $\tilde{\mu}\sum_{q=1}^{r} \left( \prod_{i=1}^{n} g_{i}(X_{i}[:,q]) \right)^{\frac{1}{n}}$.

The implicit regularization~\eqref{eq:implicit_reg} is of significant interest. The optimization obtained by replacing the explicit regularization with the implicit one,
\begin{equation}\label{eq:hrsi_implicit}
  \underset{\forall i\leq n,\; X_i\in\mathbb{R}^{m_i\times r}}{\argmin}f(\{X_i\}_{i\leq n}) + \tilde{\mu} \sum_{q=1}^{r}  \left(\prod_{i=1}^{n} g_i(X_{i}[:,q])^{\frac{1}{p_i}}\right)^{\frac{1}{\sum_{i=1}^{n} \frac{1}{p_i}}},
\end{equation}
is fully scale-invariant in the sense of assumption \textbf{A1}. We can show that the explicit and implicit regularized problems are in fact essentially equivalent. 
\begin{proposition}\label{prop:implicit_equiv}
    Given strictly positive regularization hyperparameters $\lambda_i$, the HRSI problem~\eqref{eq:hrsi} is essentially equivalent to the implicit HRSI problem~\eqref{eq:hrsi_implicit}. 
\end{proposition}
The proof of this proposition relies on the fact that, denoting $\phi(\{X_i\}_{i\leq n})=f(\{X_i\}_{i\leq n}) + \sum_{i=1}^{n} \mu_{i} \sum_{q=1}^{r} g_i(X_i[:,q])$, 
\begin{equation}\label{eq:hrsi_proof}
  \underset{\forall i\leq n,\; X_i\in\mathbb{R}^{m_i\times r}}{\min} \phi(\{X_i\}_{i\leq n}) =\underset{\forall i\leq n,\; X_i\in\mathbb{R}^{m_i\times r},~\prod_{i\leq n}\Lambda_i=1}{\min}~\phi(\{X_i\Lambda_i \}_{i\leq n}).
\end{equation}
This shows that there is a trivial mapping between the minimizers of the explicit and implicit HRSI problems, where the solution to one can be transformed into the solution to the other through a scaling of the columns of the parameter matrices.

The implicit HRSI formulation provides an alternative approach to the original HRSI problem and could be leveraged in practice to design more efficient optimization algorithms. While we have not pursued this direction in this work,
implicit HRSI has already served as the foundation for the development of an efficient algorithm for NMF, as proposed in recent work~\cite{marminMajorizationMinimizationSparseNonnegative2023}.

\paragraph{Example 1: Ridge penalization applied to all parameter matrices induces low-rankness.}
It has been established in the matrix low-rank approximation literature that applying squared $\ell_2$ norm penalizations to the Burer-Montero representation of the low-rank approximation matrix induces implicit low-rank regularization~\cite{srebroMaximumMarginMatrixFactorization2004}.

Consider the matrix LRA problem with ridge penalizations, using notations from Equation~\eqref{eq:sNMF1}:
\begin{equation}
    \underset{X_1\in\mathbb{R}^{m_1\times r}, X_2\in\mathbb{R}^{m_2\times r}}{\min} ~ \|M - X_1X_2^T\|_F^2 + \mu \left(\|X_1\|_F^2 + \|X_2\|_F^2\right).
\end{equation}
\Cref{prop:implicit_equiv} shows that this problem is essentially equivalent to
\begin{equation}
    \underset{X_1\in\mathbb{R}^{m_1\times r}, X_2\in\mathbb{R}^{m_2\times r}}{\min} ~ \|M - X_1X_2^T\|_F^2 + 2\mu \sum_{q=1}^{r} \|X_1[:,q]\|_2\|X_2[:,q]\|_2.
\end{equation}
Noticing that the product of the $\ell_2$ norms of two vectors is the $\ell_2$ norm of their tensor product, we define $L_q= X_1[:,q]\otimes X_2[:,q]$ where $a\otimes b=ab^T$ is the outer product. The problem can thus be rewritten as
\begin{equation}
    \underset{L_q\in\mathbb{R}^{m_1\times m_2},\; \text{rank}(L_q)\leq 1}{\min} ~ \|M - \sum_{q=1}^{r}L_q\|_F^2 + 2\mu \sum_{q=1}^{r} \|L_q\|_2.
\end{equation}
This is a group LASSO regularized low-rank approximation problem. Group LASSO regularization is known to promote sparsity in groups of variables~\cite{yuanModelSelectionEstimation2006}, in this case, the rank-one terms $L_q$. The expected behavior of this regularization is semi-automatic rank selection, driven by the value of $\mu$.

These derivations are not specific to matrix problems. The squared $\ell_2$ norm penalties also lead to a group norm implicit regularization on the rank-one terms of tensor low-rank approximations. This is illustrated in the case of rNCPD in~\cref{sec:Showcases}, where ridge regularization induces low nonnegative CP-rank approximations. However, since the group lasso regularization applies atop the rank-one constraints, we were unable to extend existing guarantees on the sparsity level of the solution. Additionally, to make the nuclear norm apparent in the matrix LRA problem and recover the result of Srebro and co-authors~\cite{srebroMaximumMarginMatrixFactorization2004}, we would need to rely on the rotation ambiguity of the $\ell_2$ norm, which falls outside of the framework proposed here. In other words, whether there are scenarios where the (implicit) group lasso penalty does not induce low-rank solutions remains an open problem. 

\paragraph{Example 2: Sparse $\ell_1$ penalizations may yield sparse and group sparse low-rank approximations, but individual sparsity levels cannot be tuned}
In certain applications of regularized LRA, promoting sparsity across all parameter matrices $\{ X_i\}_{i\leq n}$ is desirable. For example, in the context of NMF with $n=2$, a common formulation to achieve this is:
\begin{equation}\label{eq:l1l1nmf}
    \underset{X_1\in\mathbb{R}_+^{m_1\times r}, X_2\in\mathbb{R}_+^{m_2\times r}}{\min} ~ \|M - X_1X_2^T\|_F^2 + \mu_1 \|X_1\|_1 + \mu_2 \|X_2\|_1.
\end{equation}
This model has been studied with a perspective similar to ours~\cite{Papalexakis2013From}, hence the observations we make below are known for sNMF.

One might intuitively expect that $\mu_1$ and $\mu_2$ control the sparsity levels of $X_1$ and $X_2$, respectively. However, this expectation must be tempered. Solutions to~\eqref{eq:l1l1nmf} are equal, up to scaling, to minimizers of the implicit problem
\begin{equation}\label{eq:l1l1nmfimplicit}
    \underset{L_q\in\mathbb{R}_+^{m_1\times m_2},\; \text{rank}(L_q)\leq 1}{\min} ~ \|M - \sum_{q=1}^{r}L_q\|_F^2 + 2\sqrt{\mu_1\mu_2}\sum_{q=1}^{r} \sqrt{\|L_q\|_1}.
\end{equation}
We may observe that only the product of parameters $\mu_1$ and $\mu_2$  is significant at the implicit level. Consequently, one cannot independently tune $\mu_1$ or $\mu_2$ to adjust the sparsity levels of $X_1$ and $X_2$ individually.
While the intuition that stronger regularization leads to sparser solutions holds, it applies only to the product $\mu_1\mu_2$, affecting the sparsity level of rank-one components $L_q$, as confirmed experimentally in~\cref{sec:Showcases}. Further details on hyperparameter design are provided in~\cref{sec:choicemu}.

Another observation is that the implicit regularization term in~\eqref{eq:l1l1nmfimplicit} is a group sparsity regularization with a mixed $\ell^{0.5}_{0.5,1}$ norm. While this regularization is unconventional, it intuitively promotes both sparsity in the parameter matrices and group sparsity in the rank-one components, \textit{i.e.} perform automatic rank detection. However, these two effects are not controlled independently. This interdependence creates ambiguity in the resulting sparsity patterns: it is unclear whether individual sparsity within parameter matrices or group sparsity across rank-one components will dominate.

\paragraph{What happens with $\ell_1+\ell_2$ regularizations} Another commonly encountered pair of regularization functions in the context of regularized LRA is the $\ell_1$ and the squared $\ell_2$ norms, penalizing respectively $X_1$ and $X_2$. In this case~\cref{prop:implicit_equiv} shows that the implicit regularization is a group norm on the rows or columns of rank-one matrices $L_q$, therefore inducing sparsity on the implicit problem as intended in the explicit formulation.

\subsection{Choosing the hyperparameters}\label{sec:choicemu}
For a given set of hyperparameters $\{\mu_i\}_{i\leq n}$, proportional scaling of all hyperparameters does not alter the minimum of the cost function, although it impacts the specific column scaling of the optimal parameter matrices.
Thus it appears that hyperparameters $\{\mu_i\}_{i\leq n}$ may be chosen equal: $\mu_i =: \mu$. This choice simplifies the problem formulation and ensures that, at optimality, the equality $p_ig_i(X_i[:,q]) = p_jg_j(X_j[:,q])$ holds for all $i<j\leq n$ and for all $q\leq r$. Consequently, the solutions are balanced independently of the hyperparameter value $\mu$. When the functions $g_i$ are identical (\textit{e.g.}, all $\ell_1$ norm), this choice results in uniform penalization and scaling of all parameter matrices. Furthermore, as $\mu$ goes to zero, all regularization terms vanish. This is practically desirable because the limit case $\mu=0$ is consistent with the unregularized problem. Conversely, fixing a single hyperparameter $\mu_j$ while allowing the others to approach zero renders the problem ill-posed, as demonstrated in~\cref{sec:illposed}. When the dimensions $m_1$ and $m_2$ are vastly different, it may be of interest to choose $\mu_i = \frac{1}{m_i}\mu$ to scale the parameter matrices and avoid large discrepancies in their entries numerical values.

\subsection{Related literature}\label{sec:sota}
The implicit regularization effects of explicit regularization in scale-invariant models have been previously recognized in the literature. 
An early contribution by Benichoux \textit{et al.}~\cite{benichouxFundamentalPitfallBlind2013} highlighted the pitfalls of employing $\ell_1$ and $\ell_2$ regularizations in blind deconvolution. In particular, they demonstrated that solutions to such regularized model can be degenerate and non-informative. Their analysis is however tailored to convolutive mixtures, and similar degeneracies were not observed in the context of low-rank approximations.

Another closely related work studies the effect of scale-invariance for sNMF~\cite{Papalexakis2013From}. Their study derived the same implicit problem as Equation~\eqref{eq:l1l1nmfimplicit} and identified the pitfall of independently tuning hyperparameters. They also notice empirically the existence of "scaling swamps" which we analyse in~\cref{sec:balance_speed}. 
However, their work is confined to sNMF, lacks a discussion of implicit regularization effects, and does not generalize the findings to broader classes of regularized low-rank approximations.

Further contributions within the sNMF framework include the work of Marmin \textit{et. al.} that established an equivalence between sNMF with normalization constraints and a scale-invariant problem~\cite{marminMajorizationMinimizationSparseNonnegative2023}. This work then proceeds to minimize the scale-invariant loss. Tan and Févotte, by generalizing the work of M{\o}rup~\cite{morupAutomaticRelevanceDetermination2009}, have studied automatic rank selection using $\ell_1$ or $\ell_2$ penalizations, where the scales are controlled via the variance of the parameter matrices~\cite{tanAutomaticRelevanceDetermination2013}. Srebro \textit{et. al.} demonstrated that ridge penalization in two-factor low-rank matrix factorization is equivalent to nuclear-norm minimization, a foundational result that initiated broader investigations into implicit regularization effects~\cite{srebroMaximumMarginMatrixFactorization2004}. 
Despite their significance, these works are confined to specific models and do not capture the systemic influence of scale-invariance on regularized LRA or address the potentially misleading intuitions associated with explicit regularization.

Beyond sNMF, several works have examined implicit regularization in regularized low-rank factorization problems. Uschmajew observed that ridge regularization balances solutions in canonical polyadic decompsition~\cite{Uschmajew2012Local}, while Roald \textit{et. al.} showed that regularizations hyperparameters can be collectively tuned in regularized low-rank approximations with homogeneous regularizations~\cite{roaldAOADMMApproachConstraining2022}. Parallel insights have emerged in the context of deep learning, where scale-invariance also plays a significant role. 
Ridge penalization on the linear layers of a feedforward ReLU network is equivalent to a group sparse LASSO problem~\cite{NeyshaburTS14}, as recalled in a recent note by Tibshirani~\cite{tibshirani2021equivalences}. Practical implications of explicit and implicit regularizations in deep neural networks have also been explored~\cite{ergenImplicitConvexRegularizers2020, stock2018equinormalization, gribonvalScalingAllYou}, see this recent overview for further references~\cite{parhiDeepLearningMeets2023}.

Finally, it is worth noting that~\cref{prop:opt_scal} is reminiscent of the quadratic variational formulation of norms~\cite{bachOptimizationSparsityInducingPenalties2011} albeit adapted to the multifactor setting. 

\section{Algorithms}\label{sec_algorithms}
In this section, we focus on developing practical optimization strategies tailored to regularized low-rank approximations with elementwise $\beta$-divergence loss functions. While the HRSI model is more general, a broad derivation of optimization strategies for it would be impractical.


More precisely, we consider nonnegative decompositions of input tensors or matrices, using the elementwise $\beta$-divergence~\cite{fevotte2009nonnegative} as the loss function $f$. The regularization terms are chosen from a family of functions that can be upper-bounded by separable quadratics, which includes $\ell_p^p$ norms.
Let us recall that the $\beta$-divergence between two matrices $A$ and $B$ is defined as
\begin{equation*}
    D_\beta(A,B) = \sum_{i,j} d_\beta(A[i,j],B[i,j]), 
\end{equation*}
where, for scalars $x$ and $y$, 
\begin{equation} \label{eq:betadivergence}  
d_{\beta}(x,y) = \left\{ 
\begin{array}{cc}
  \frac{x}{y} - \log \frac{x}{y} - 1  & \text{ for }  \beta = 0, \\
 x \log \frac{x}{y} - x + y & \text{ for } \beta = 1, \\ 
\frac{1}{\beta (\beta-1)} \left(x^\beta + (\beta-1) y^\beta - \beta x y^{\beta-1}\right) &  \text{ for } \beta \neq 0,1,  
 \end{array}
\right.
\end{equation} 
see \cite{basu1998robust, eguchi2001robustifying}. 
When $\beta = 2$, the $\beta$-divergence corresponds to the least-squares measurement. For $\beta=1$ and $\beta=0$, it corresponds to the Kullback-Leibler (KL) divergence and the Itakura-Saito (IS) divergence, respectively. \footnote{Note that, with convention  $a\log 0=-\infty$ for $a>0$ and $0\log 0=0$, the KL-divergence is well-defined.} The choice of $\beta$ is usually driven by two main motivations:
\begin{enumerate}
    \item \textbf{Degree of homogeneity:} The $\beta$-divergence $d_{\beta}(x,y)$ is homogeneous of degree $\beta$, meaning that $d_{\beta}(\lambda x, \lambda y) = \lambda^{\beta}(x,y) $. This implies that factorizations obtained with $\beta > 0$ (such as the Euclidean distance or the KL divergence) will give more weight to the largest data values and less precision is expected in estimating of the low-power components. The IS divergence ($\beta =0$) is scale-invariant, \textit{i.e.}, $d_{IS}(\lambda x|\lambda y)=d_{IS}(x|y)$. This is the only divergence in the $\beta$-divergences family with this property, meaning that low-power entries are treated as equally important as large entries in the divergence computation. 
    This property is useful in applications such as audio source separation, where low-power frequency bands contribute perceptually as much as high-power frequency bands.
    For further details, see~\cite{Lefevre_phd} and~\cite{gillis2020bk}. 
    \item \textbf{Noise statistics:}  The choice of \( \beta \) is also influenced by the assumed noise statistics in the data. Specifically, there is an equivalence between the choice of the parameter \( \beta \) and the maximum likelihood estimators of factors in the decomposition, given specific statistical distributions for the noise. In particular, 
    the values \( \beta = 2 \), \( \beta = 1 \), and \( \beta = 0 \) correspond to the assumptions of Gaussian, Poisson, and exponential noise distributions, respectively. For further details, the reader is referred to \cite{gillis2020bk} and references therein.
\end{enumerate}

Our approach is based on three main components: Cyclic Block-coordinate Descent (BCD), Majorization-Mimization (MM), and Optimal balancing. The latter consists in performing a balancing step for all blocks to minimize the regularization terms, as described in~\cref{subsec:balancingPrinciple}.
In the following, we discuss these three components in more detail. We then introduce our general algorithm, referred to as the \textit{Meta-Algorithm}, and conclude by presenting its convergence guarantees.

\subsection{BCD: which splitting ?}\label{sec:bcdwhich}
We begin by splitting the variables into disjoint blocks\footnote{their union is the initial and full set of variables} and sequentially minimize the objective function over each block. The updates for each block will be similar, and the objective function in \eqref{eq:hrsi} is separable columnwise. Therefore to simplify the notations, we consider the update of the block of variables $X_i[:,q]$ (one column of $X_i$, dropping the $i$ index), denoted $x\in\mathbb{R}_+^{m}$, while the other blocks of variables are fixed. Recall that $x[k]$ represents the entry in vector $x$ at index $k$. We are interested in solving the following subproblem:
\begin{equation}\label{eq:subProblem_in_x}
    \min_{x \in \mathbb{R}_+^{m}} f(x) + \mu g(x).
\end{equation}

\subsection{Majorization-Minimization (MM)}\label{sec:MM_framework}

To address Problem \eqref{eq:subProblem_in_x}, we employ the widely-used MM framework. In this approach, the objective function $\phi$ is iteratively globally majorized by a locally tight approximation $\Bar{\phi}$, which can be efficiently minimized. In the context of regularized LRA, such majorizers are typically designed to be separable, enabling cost-efficient update rules. The originality in our approach lies in the presence of the regularizer $g$. 

We now detail how to derive a separable majorizer $\Bar{\phi}(x|\tilde{x})$ tight at a given point $\tilde{x}$ for the objective function $\phi(x):=f(x) + \mu g(x)$. 
The first term $f(x)$ is majorized following the methodology introduced in \cite{fevotte2011algorithms}, where the beta-divergence is decomposed as the sum of a convex and a concave function. The convex part is majorized using Jensen's inequality, while the concave part is upper-bounded by its first-order Taylor approximation. The resulting separable majorizer for $f(x)$ built at the current iterate $\widetilde{x}$ is denoted $\Bar{f}(x|\widetilde{x})$. For the second term, we make the following assumptions:

\begin{assumption}\cite{doi:10.1137/20M1377278}\label{ass1}
For the smooth and positive-definite regularization function $g: x \in \mathbb{R}_+^{m} \rightarrow \mathbb{R}_+$, and any $\Tilde{x} \in \mathbb{R}_+^{m}$, there exists nonnegative constants $C_k$ with $1 \leq k \leq m$ such that the inequality
\begin{equation}\label{eq:MajorizerRegu}
    g(x) \leq \Bar{g}(x|\Tilde{x}):=g(\Tilde{x}) + \langle \nabla g(\Tilde{x}), x - \Tilde{x}\rangle + \sum_{k=1}^m \frac{C_k}{2} (x[k] - \Tilde{x}[k])^2
\end{equation}
holds for all $x \in \mathbb{R}_+^{m}$. 
\end{assumption}
In other words, the regularization function can be upper-bound by a \textit{tight separable} majorizer given in Equation~\eqref{eq:MajorizerRegu}. This property holds for all $\ell^p_p$ norms employed in this work as regularizers. Consequently, a separable majorizer $\Bar{\phi}(x|\widetilde{x})$ for $\phi(x)$ can be constructed at any $\Tilde{x} \in \mathcal{X}$ by combining $\Bar{f}(x|\widetilde{x})$ and $\Bar{g}(x|\Tilde{x})$. This majorizer is then minimized to compute the update $\hat{x}$.
The existence and uniqueness of the minimizer $\hat{x}$ can be established under mild conditions~\cite{doi:10.1137/20M1377278}.


\paragraph{Computation of the update rules for beta divergence and $\ell_p^p$ regularizations}
The MM updates for $\beta$ in $[0,2]$ and regularization functions that are $\ell_p^p$ norms, both of which are important particular cases, are derived in~\cref{app:update_rules} and~\cref{app:beta_zero_l2reg}.
One of our contributions is to derive update rules for $\beta=1$ and $\beta=0$ with squared $\ell_2$ regularizations leading to convergence guarantees for the iterates. These cases have not been addressed in the existing literature. As an illustrative example, for the specific pair $\beta=1$ and $p=2$, the update for each entry of the block of variable under consideration, denoted $x$, has the general form
\begin{equation}\label{eq:updatexkentry}
    \begin{aligned}
       x[k]\leftarrow\frac{\sqrt{c^2+8\mu x[k]t}-c}{4\mu},
    \end{aligned}
\end{equation}
where $\mu > 0$ and $c$ and $t$ are nonnegative real-valued functions of the input data and the other blocks of variables held fixed. 

More generally, a closed form expression for the update $\hat{x}$ can be derived if $\beta$ belongs to $\{0,1,3/2,2\}$ according to \cite{doi:10.1137/20M1377278}.  Outside this set of values for $\beta$, an iterative scheme such as Newton-Raphson is required to compute the minimizer. 
Note that while MM updates for $\beta=2$ can be derived within the proposed framework, in this case, the literature suggests the superiority of an alternating least squares approach called HALS~\cite{Gillis2012Accelerated,marmoret2020uncovering}.

\subsection{Balancing improves convergence speed}\label{sec:balance_speed}
One of the contributions of this work is the introduction of a numerical balancing of the columns of low-rank components, building upon the results presented in~\cref{sec:Sparsity_scale_inva_models}. While the balancing procedure has been established, its practical utility remains to be clarified. To this end, we present an example where the existence of ``scaling swamps''~\cite{Papalexakis2013From} for the Alternating Least Squares (ALS) algorithm can be proven.

For a given real value $y$ and a regularization parameter $0<\lambda\ll y$, consider the scalar optimization problem
\begin{equation}\label{eq:toy_als}
    \underset{x_1,x_2}{\min} f(x_1,x_2) \text{  where  } f(x_1,x_2)= (y-x_1x_2)^2 + \lambda(x_1^2 + x_2^2).
\end{equation}
Using the results from~\cref{sec:Sparsity_scale_inva_models}, we can immediately infer that the minimizers $x_1^\ast$ and $x_2^\ast$ must be balanced. This yields optimal solutions $x_1^\ast=x_2^\ast = \sqrt{y-\lambda}$ (for $\lambda\leq y)$. 

Setting aside the fact that we know the optimal solution, a classical iterative BCD method from the regularized LRA literature for solving this problem would be the ALS algorithm. 
Denoting the iterates at iteration $k$ as $x_1^{(k)}$ and $x_2^{(k)}$, the update rules are straightforwardly derived as 
\begin{equation}
        x_1^{(k+1)} = \frac{x_2^{(k)}y}{{x^2_2}^{(k)}+\lambda} \; \; \text{and} \; \; 
        x_2^{(k+1)} = \frac{x_1^{(k+1)}y}{{x^2_1}^{(k+1)}+\lambda} .
\end{equation}
We show in~\cref{app:proof_als_sublin} that when $\lambda$ is small relative to $y$, and for sufficiently large $k$,
\begin{equation}
  \frac{x_1^{(k+1)} - \sqrt{y-\lambda}}{x_1^{(k)} - \sqrt{y-\lambda}} \approx 1-4\frac{\lambda}{y}
\end{equation}
with a similar expression holding for $x_2$.
Thus, ALS converges almost sub-linearly when $x_{1}$ and $x_2$ are near the optimum, particularly when $\frac{\lambda}{y}\ll 1$. The smaller the ratio $\frac{\lambda}{y}$, the slower the convergence. A numerical simulation of this toy problem solved with the ALS algorithm experimentally confirmed these observations, see~\cref{fig:als_scale}.

Similar slow convergence behaviors are observed in more realistic regularized LRA problems, as shown in~\cref{sec:Showcases}. To address this, the Meta-Algorithm below includes a step to optimally balance the iterates after each outer iteration.

\begin{figure}[!ht]
    \centering
    \includegraphics[width=15cm]{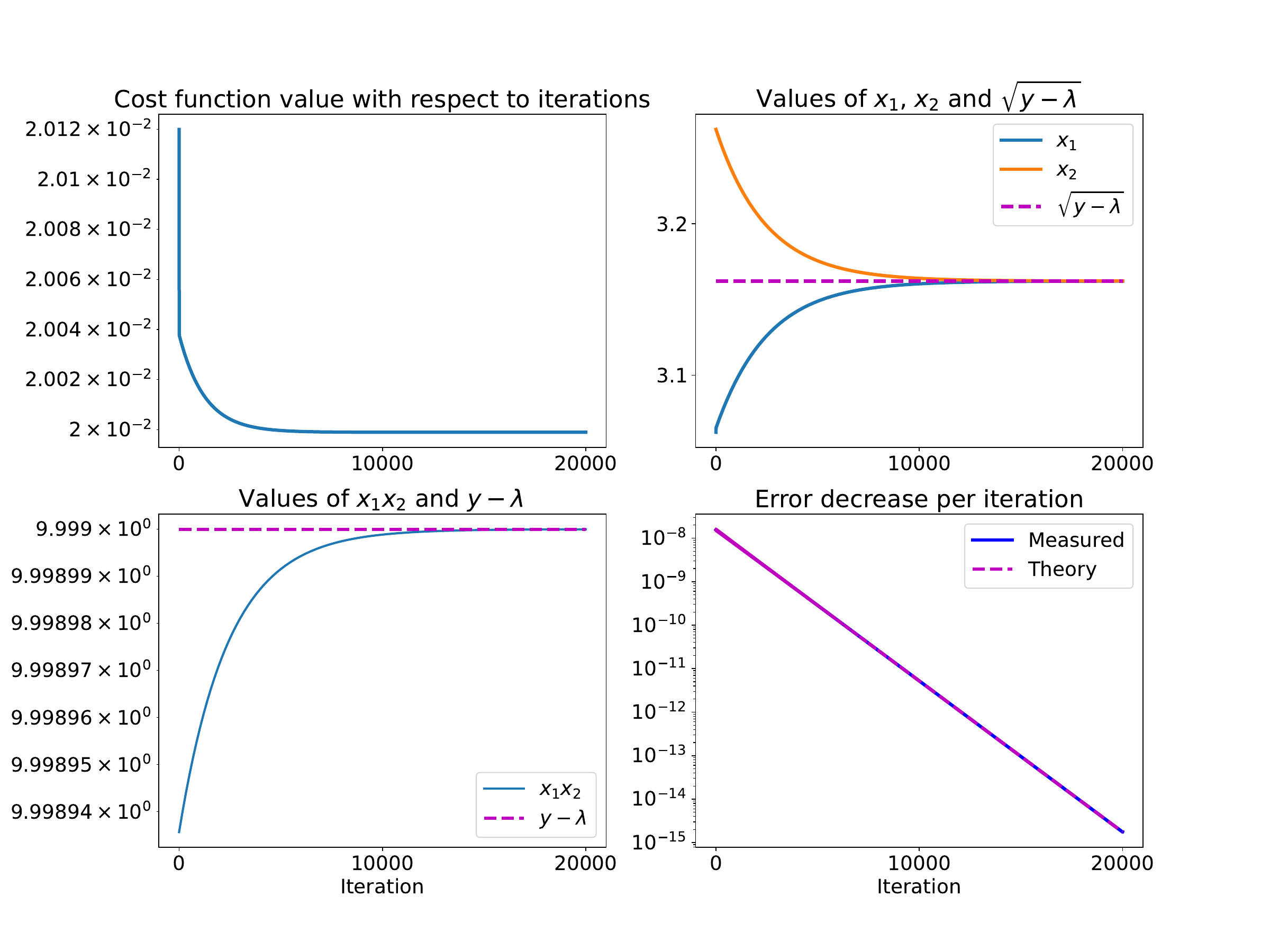}
    \caption{Study of the convergence of ALS for the toy problem~\eqref{eq:toy_als}. We choose $y=10$ and $\lambda=10^{-3}$. The number of iterations of ALS is set to 20000. 
    Top left: cost function $f$ with respect to iteration index. Top right: values of $x_1$ (orange), $x_2$ (blue) and $\sqrt{y-\lambda}$ (magenta). Bottom left: value of $x_1x_2$ (blue) and $y-\lambda$ (magenta). Bottom right: empirical value of the cost function variation (blue) and theoretical value $16\frac{\lambda^2}{y}e_k^2$ (magenta). 
    Several observations can be made: The value of $x_1x_2$ rapidly converges towards $y-\lambda$ (bottom left graph). However the cost converged after at least 5000 iterations, and the individual values of $x_1$ and $x_2$ are slow to converge to the optimum $\sqrt{y-\lambda}$ with respect to how fast $x_1x_2$ converged to $y-\lambda$. We can observe a linear convergence rate with a multiplicative constant very close to $1$ in the bottom right graph, and a close match between the theoretical approximate cost decrease and the practical observation of that decrease.
    }
    \label{fig:als_scale}
\end{figure}

\subsection{Meta-Algorithm}\label{sec:meta_algo}
The proposed method for solving problem~\eqref{eq:hrsi} is summarized in~\cref{alg:meta-Algo}. The algorithm is iterative, alternating, and consists of two main steps within each block update. First the block $\{ X_i\}_{i \leq n}$ is updated using MM. Second, optimal balancing of the block of variables is performed. This balancing step is computationally inexpensive and necessarily decreases the objective function, see~\cref{prop:nonIncreaseRB}. Thus, it does not compromise the convergence guarantee offered by the MM-based updates, ensuring the convergence of the sequence of objective function values. However, the convergence guarantees for the sequence of iterates will depend on the choice of the value $\beta$ in Problem~\eqref{eq:GenOptiProblem}. These points are discussed in more detail in~\cref{sec:specificity_beta} below.

\algsetup{indent=2em}
\begin{algorithm}[ht!]
\caption{Meta-Algorithm \label{alg:meta-Algo}}
\begin{algorithmic}[1] 
\REQUIRE An input tensor $\mathcal{T}\in \mathbb{R}^{m_1 \times \cdots \times m_n}$, 
an initialization $X_i^{(0)}$ and regularization hyperparameters $\mu_i > 0$,
the factorization rank $r$, a maximum number of iterations, maxiter.
\STATE $ \eta \in \text{Argmin}_{\eta \geq 0} f(\{\eta X_i^{(0)}\}_{i\leq n}) + \sum_{i=1}^{n} \mu_{i} \sum_{q=1}^{r} \eta^{p_i} g_i(X_i^{(0)}[:,q])$ \emph{  \% Rescaling}
\STATE $X_i^{(1)} \leftarrow \eta X^{(0)}$ for all $i=1,...,n$.
\STATE \emph{\% Main optimization loop}:
\FOR{$k$ = 1 : maxiter}
    \FOR{$i$ = 1 : $n$ and $q$ = 1 : $r$}
            \STATE  $X_i[:,q]^{(k+1)} \leftarrow \argmin_{x \in \mathcal{X}_i} \Bar{\phi}(x|X_i[:,q]^{(k)})$ \emph{  \% Update of parameter matrices}
    \ENDFOR
    \STATE  $\{X_i^{(k+1)}[:,q]\}_{i\leq n}$ balanced using~\cref{alg:balancing-Algo} for all $q\leq r$ \emph{  \% Optimal balancing}
\ENDFOR
\RETURN $\{X_i[:,q]\}_{i\leq n}$
\end{algorithmic}  
\end{algorithm} 
\paragraph{Initialization}

In~\cref{alg:meta-Algo}, we propose an initialization stage where the blocks of variables are scaled (lines 1-2). This scaling is essential to prevent the "zero-locking" phenomenon. To illustrate this issue, let us consider a matrix case $n=2$.
If the regularization functions $g_1$ and $g_2$ are chosen as $\ell_1$ norms, the updates for the columns of parameter matrix $X_1$ reduce to solving a series of LASSO problems, with the mixing matrix given by $X_2^T$. For such problems, it is well known that if $\|M X_2 \|_{\infty} \leq \mu_1$, the solution is necessarily $X_1=0$. In this scenario, $(0,0)$ being a saddle point it is very likely that~\cref{alg:meta-Algo} will return $X_1=X_2=0$. The scaling step mitigates this issue by ensuring that the initial values of $X_1$ and $X_2$ are not excessively small.

\subsection{Convergence guarantees}\label{sec:conv_gua_Algo1}
We establish two convergence results for the proposed Meta-Algorithm: (1) the convergence of the cost function, and (2) under further assumptions, the convergence of the iterates to a stationary point. While these results are not unexpected, they are not straightforward corollaries of existing literature. Detailed proofs, technical details, and precise formulations of the assumptions are provided in~\cref{AppendC}. 

Before we give the main convergence results, it is convenient to write down a general formulation for our optimization problems. Let $z=(\{X_i[:,q]\}_{i \leq n, q \leq r})$ denote the concatenation of the variables $X_i[:,q]$ and let $\phi$ represent the objective function from Problem~\eqref{eq:hrsi}. We consider the following class of multiblock  optimization problems:
\begin{equation}
    \label{eq:GenOptiProblem_1}
    \min_{z_i\in\mathcal{Z}_i} \phi(z_1,\ldots,z_s),
\end{equation}
where $z$ is decomposed into $s$ blocks $ z=(z_1,\ldots,z_s)$, $\mathcal{Z}_i\subseteq\mathbb{R}^{m_i}$  is a closed convex set, and $\mathcal{Z}=\mathcal{Z}_1\times\ldots\mathcal{Z}_s \subseteq \text{dom} (\phi)$. The function $\phi:\mathbb{R}^m\to\mathbb{R}\cup\{+\infty\}$, with $m=\sum_{i=1}^s m_i$, is continuously differentiable over its domain, lower bounded and has bounded level sets.

\begin{theorem}\label{theo:monotoneNI_main} Let $z^{(k)}=(\{X_i^{(k)}[:,q]\}_{i \leq n, q \leq r}) \geq 0$, and let $\Bar{\phi}_i$ be block-wise tight majorizers\footnote{See~\cref{def:majorizer}} for $\phi: \mathcal{Z}\rightarrow \mathbb{R}_+$ from Problem~\eqref{eq:GenOptiProblem_1},  with $\mathcal{Z}:=\mathcal{Z}_1\times\ldots\mathcal{Z}_s \subseteq \text{dom} (\phi)$ and $1 \leq i \leq nr$. Then $\phi$ is monotone non-increasing under the updates of~\cref{alg:meta-Algo}, and the sequence of cost values converges. 
\end{theorem}

\begin{theorem}\label{theo:conv_iterates_main}
    Let block-wise tight majorizers $\Bar{\phi}_i(y_i,z)$ be quasi-convex in $y_i$ for $i=1,...,nr$. Additionally, assume that the directional derivatives of the majorizers and the objective coincide locally for each block and that the majorizer is continuous. Further, assume that the MM updates and the optimal scaling have unique solutions. 
    Under these conditions, every limit point $z^{\infty}$ of the iterates $\{z^{(k)}\}_{k \in \mathbb{N}}$ generated by~\cref{alg:meta-Algo} is a coordinate-wise minimum of Problem~\eqref{eq:GenOptiProblem_1}. Furthermore, if $\phi$ is regular at $z^{\infty}$, then $z^{\infty}$ is a stationary point of Problem~\eqref{eq:GenOptiProblem_1}.
\end{theorem}
The proofs of these theorems are adapted from the convergence results to stationary points obtained in \cite[Theorem~2]{RHL_BSUM} for the Block Successive Upperbound Minimization (BSUM) framework. They leverage the fact that the balancing procedure always decreases the cost $\phi$. 

All the hypotheses of~\cref{theo:monotoneNI_main} and~\cref{theo:conv_iterates_main}, except for directional differentiability ( see~\cref{sec:specificity_beta}), are easily satisfied by many regularized LRA models, in particular those discussed in~\cref{sec:Showcases}.

\subsection{Handling near-zero entries with the $\beta$-divergence} \label{sec:specificity_beta}
To ensure the convergence of the iterates generated by~\cref{alg:meta-Algo} to critical points of Problem~\eqref{eq:GenOptiProblem_1}, one of the required assumptions is the directional differentiability of the function $\phi$, see~\cref{AppendC}. However, this property does not hold at zero as soon as $\beta < 2$, see for instance the discussions in~\cite{HienNicolasKLNMF}.
To guarantee convergence for the iterates, 
a common solution is to impose constraints $X_i \geq \epsilon$ for all $i \in [1,n]$, where $\epsilon > 0$~\cite{HienNicolasKLNMF}. 
Under these constraints, the updates of parameter matrices in~\cref{alg:meta-Algo} (Step 6) become
\begin{equation}\label{eq:modifiedUpdates}
    X_i[:,q]^{(k+1)} \leftarrow \text{max} \Big(\epsilon , \argmin_{x} \Bar{\phi}(x|X_i[:,q]^{(k)})\Big), 
\end{equation}
see~\cite{Takahashi2014} where this approach has been successfully employed to derive multiplicative updates algorithms for NMF with global convergence guarantees.

Given the modified updates, ensuring the convergence guarantees stated in~\cref{theo:conv_iterates_main} is not straightforward. The reason lies in the optimal balancing step of~\cref{alg:meta-Algo}. Indeed, if one entry of $X_i[:,q]^{(k+1)}$ reached $\epsilon$, the balancing procedure may decrease this entry below $\epsilon$, making it non-feasible. 
We discuss two strategies to address this issue:
\begin{itemize}
    \item \textbf{Suboptimal scaling with projection} If a column of a parameter matrix becomes elementwise smaller than $\epsilon$, all the corresponding columns in other parameter matrices are set to $\epsilon$ as well. When at least one entry in each column is strictly larger than $\epsilon$, we first set all values smaller than $\epsilon$ to zero, optimally scale the resulting projected columns, and then project back the outcome such that all entries are larger than $\epsilon$. This approach leads to a suboptimal scaling and may therefore break convergence, but we anticipate that these theoretical considerations will not affect the algorithm behavior noticeably. This is the recommended approach in practice, used in~\cref{sec:Showcases}.
    \item \textbf{Stop balancing to guarantee convergence} A second approach is to stop balancing as soon as any entry in the parameter matrices is smaller than $\epsilon$. 
    This strategy ensures the convergence of the iterates to critical points of the $\epsilon$-perturbed Problem~\eqref{eq:GenOptiProblem_1} through BSUM, with the balancing steps primarily serving to accelerate the convergence of the algorithms during the initial iterations. Balancing may even only be applied upon initialization. This approach appears to be the most suitable from a theoretical perspective, but in practice, the balancing may not be performed for many iterations which is undesirable. 
    Indeed, anticipating on~\cref{sec:Showcases}, performing optimal balancing along all iterations is practically profitable.
    
\end{itemize}

Finally, for the case $\beta \geq 2$, the $\beta$-divergence exhibits continuous differentiability over the feasible set, thereby ensuring the convergence of~\cref{alg:meta-Algo} as outlined in~\cref{theo:conv_iterates_main}. However, thresholding the entries in parameter matrices $X_i$ to $\epsilon$ when $\beta\geq 2$ may still be of practical interest. This is because, with multiplicative updates, any entry in a parameter matrix that reaches zero cannot be further modified by the algorithm.
This can prevent convergence in practice, especially with poor initialization or insufficient machine precision. 
In the proposed implementations of the meta-algorithm, whenever nonnegativity is enforced, a small $\epsilon$ is used to lower-bound the entries of all parameter matrices.

\section{Showcases}\label{sec:Showcases}
In this section, we study three specific instances of HRSI: sNMF, rNCPD, and sNTD. We begin by relating each model to HRSI, outlining the balancing update and the implicit equivalent formulation, and detailing the implementation of the proposed alternating MM algorithm. Next, we evaluate the performance of the proposed method with and without balancing.

Additionally, we provide an implementation of the proposed algorithm for each showcase problem, along with reproducible experiments, in a dedicated repository\footnote{\url{https://github.com/vleplat/NTD-Algorithms}}. For multilinear algebra, we rely on Tensorly~\cite{kossaifi2019tensorly}, and we use the package shootout\footnote{\url{https://github.com/cohenjer/shootout}} for running the experiments.

\subsection{Explicit formulations and relation to HRSI}

We consider sNMF, rNCPD, and sNTD respectively, explicitly formulated as follows:
\begin{gather}
\label{eq:l1SparseNMF}\tag{sNMF}
        \min_{X_1 \geq0, X_2\geq0 }  \text{KL}(M,X_1X_2^T) + \mu_1 \|X_1\|_1 +  \mu_2 \| X_2 \|_1, \\
        \label{eq:rNCPD}\tag{rNCPD}
    \underset{X_i\geq 0}{\argmin}  \| \mathcal{T} - \mathcal{I}_r \times_1 X_1 \times_2 X_2 \times_3 X_3 \|^2_F +  \mu \left(\|X_1\|_F^2 + \|X_2\|_F^2 + \|X_3\|_F^2\right), \\
\label{eq:sNTD_model}\tag{sNTD}
    \underset{X_i\geq 0, \mathcal{G}\geq 0}{\argmin}  \text{KL}(\mathcal{T}, \mathcal{G} \times_1 X_1 \times_2 X_2 \times_3 X_3) +  \mu \left(\|\mathcal{G}\|_1 + \|X_1\|_F^2 + \|X_2\|_F^2 + \|X_3\|_F^2\right),
\end{gather}
where sNTD has sparsity enforced on the core tensor $\mathcal{G}$. For sNMF, two regularization hyperparameters $\mu_1\neq \mu_2$ are introduced to test the observations made in~\cref{sec:scale-optimality}. 

As discussed in~\cref{sec:HRSI_intro}, sNMF, and similarly rNCPD, are instances of HRSI. sNTD may also be formulated as a rank-one HRSI model, by considering a vectorized formulation:
\begin{equation}\label{eq:sNTDvec}
    \underset{x_1\geq 0, x_2\geq 0, x_3\geq 0, g\geq 0}{\argmin} \text{KL}(\mathcal{T},\{x_1,x_2,x_2,g\}) + \mu (\|g\|_1 + \|x_1\|_F^2 + \|x_2\|_F^2 + \|x_3\|_F^2),
\end{equation}
where $x_i=\vecn(X_i)$ and $g=\vecn(\mathcal{G})$ for a chosen arbitrary vectorization procedure, and where $\text{KL}$ is the KL divergence against tensor $\mathcal{T}$ with vectorized inputs.

\subsubsection{Optimal balancing and implicit formulations}
We can apply the results of~\cref{sec:Sparsity_scale_inva_models} straightforwardly for each model to obtain the following balancing updates for each column indexed by $q\leq r$ and each parameter matrix $X_i$:
\begin{align}\label{eq:sNMFbalanc}
    &X_i[:,q] \leftarrow \frac{\sqrt{\mu_1\mu_2\|X_1[:,q]\|_1\|X_2[:,q]\|_1}}{\mu_i\|X_i[:,q]\|_1}X_i[:,q], \tag{sNMFb} \\\label{eq:rNCPD_balanc}
    &X_i[:,q] \leftarrow \left( \sqrt{2} \|X_i[:,q]\|_2^{-\frac{1}{3}}\prod_{j\neq i}\|X_j[:,q]\|_2^{\frac{2}{3}} \right)X_i[:,q]. \tag{rNCPDb}
\end{align}
Balancing for sNTD is done with respect to the vectorized inputs and therefore there is virtually only one column to scale for each parameter matrix:
\begin{equation}\label{eq:sNTD_balanc}\tag{sNTDb}
X_i \leftarrow \sqrt{\frac{2^{\frac{3}{5}}\left(\|\mathcal{G}\|_1\prod_{j=1}^{3}\|X_j\|_F\right)^{2/5}}{2 \|X_i\|^2_F}} X_i 
\; \text{and} \;
\mathcal{G} \leftarrow \frac{2^{\frac{3}{5}}\left(\|\mathcal{G}\|_1\prod_{j=1}^{3}\|X_j\|_F\right)^{2/5}}{\|\mathcal{G}\|_1}  \mathcal{G}.
\end{equation}
    
Furthermore, the essentially equivalent implicit HRSI problem for sNMF and rNCPD can respectively be cast as
\begin{align}\label{eq:l1l1nmfimplicit2}
    &\underset{L_q\in\mathbb{R}_+^{m_1\times m_2},\; \text{rank}(L_q)\leq 1}{\min} ~ \text{KL}(M, \sum_{q=1}^{r}L_q) + 2\sqrt{\mu_1\mu_2} \sum_{q=1}^{r} \sqrt{\|L_q\|_1}. \\
&\underset{\{\mathcal{L}_q\}_{q\leq r},~\text{rank}(\mathcal{L}_q)\leq 1}{\inf} ~ \|\mathcal{T} -\sum_{q=1}^{r} \mathcal{L}_q \|_F^2 + 3\mu \sum_{q=1}^{r} \|\mathcal{L}_q\|_F^{\frac{2}{3}}
\end{align}
with $\mathcal{L}_q = X_1[:,q]\otimes X_2[:,q]\otimes X_3[:,q]$.
In a nutshell, sNMF in its implicit formulation is regularized with a mixed norm that promotes sparsity at both the entry and the component levels, while rNCPD is regularized with a group sparse norm on the rank-one components.

The sNTD model exhibits rotational invariances that our approach does not address, as they do not match the assumptions of the HRSI model. Based on formulation \eqref{eq:sNTD_model}, we can derive an implicit formulation; however, balancing is applied to all components simultaneously, which obscures component-wise group sparse regularization. In \cref{app:NTD_sinkhornbalancing}, we demonstrate that a Sinkhorn-like procedure can be designed to balance sNTD algorithmically. However, this approach did not yield improvements over the simpler vectorization-based balancing and deserves further exploration.

\subsubsection{Majorization-Minimization update rules}
For sNMF, the MM updates take the form of ``multiplicative'' updates:
\begin{equation}\label{eq:sNMFupdateH}
   \hat{X}_1=\max\left(X_1 \odot \frac{X_2 \frac{M}{X_2^TX_1}}{\mu_1 e_{m_1\times r}+e_{m_1}\otimes X^T_{2}e_{m_2}}, \epsilon\right),
\end{equation} 
where $e_{n}$ is a vector or matrix of ones of size $n$ and the maximum is taken elementwise. A similar update is obtained for $X_2$, detailed computations are found in~\cref{app:update_rules}. For sNTD, consider the unfolded tensor along the $i$th mode $\mathcal{T}_{[i]} \in \mathbb{R}_+^{m_i \times \prod_{j\neq i}m_j}$ and the Kronecker products $U_i:=\mathcal{G}_{[i]} \left(\kron_{j\neq i} X_j \right)^T \in \mathbb{R}_+^{r_1 \times \prod_{j\neq i}m_j}$. The MM updates then read
\begin{align} 
       \hat{X_i}&=\max \Big( \frac{\left[C_i^{.2}+S_i\right]^{.\frac{1}{2}}-C_i}{4 \mu}, \epsilon \Big),
\\
    \hat{\mathcal{G}} &= \max \Big(\mathcal{G} \odot \frac{\mathcal{N} \times_1 X_1^T \times_2 X_2^T \times_3 X_3^T}{\mu \mathcal{E}_{r_1\times r_2\times r_3} + \mathcal{E}_{m_1\times m_2\times m_3} \times_1 X_1^T \times_2 X_2^T \times_3 X_3^T}, \epsilon \Big),
\end{align}
where $C_i=EU_i^T$ with $E$ and $\mathcal{E}$  all-one matrices and tensors,
$S_i=8\mu X_i \odot \left( \frac{ \mathcal{T}_{[i]} }{X_iU_i} U_i^T\right) $ and $\mathcal{N} = \frac{ \mathcal{T}}{\mathcal{G} \times_1 X_1 \times_2 X_2 \times_3 X_3 }$. The matrices and tensors of ones are not contracted as such in practice; instead, we use summations of columns in the parameter matrices.

Regarding rNCPD, we solve the problem using HALS~\cite{Gillis2012Accelerated} as explained in~\cref{sec_algorithms}.

\subsubsection{Synthetic experiments}
\paragraph{Experiments description}
We study the impact of the balancing procedure in the meta-algorithm, compared to the same algorithm without explicit balancing. We observe the performance of both strategies in terms of loss value at the last iteration.
In addition, we explore a third option in which balancing is performed only at initialization. Some experiments conducted here are adapted from our previous work~\cite{cohenRegularisationImpliciteFactorisations2023}.

The process for generating synthetic data is as follows:
\begin{enumerate}
    \item We construct the ground truth low-rank matrix/tensor \( M \) with dimensions \( m_i = 30 \) using low-rank models sNMF, rCPD, or sNTD.
    \item The rank \( r \) of the model is fixed at 4, and the estimated rank \( r_e \) for the approximation is also set to 4.
    \item For the generation of noisy data:
   \begin{itemize}
       \item For sNMF and sNTD, the noisy data matrix/tensor is sampled from the Poisson distribution \( \mathcal{P}(\alpha M) \).
       \item For rNCPD, it is sampled from a Gaussian independent and identically distributed (i.i.d.) distribution with variance \( \alpha \) and mean \( M \).
   \end{itemize}
   \item The parameter \( \alpha \) is selected to ensure that the average Signal-to-Noise Ratio (SNR) is 40 dB for sNMF and sNTD, and 200 dB for rNCPD.
   \item Finally, the noisy data is normalized using its Frobenius norm.
\end{enumerate}
The ground-truth parameter matrices \( X_i \) are sampled elementwise from i.i.d. Uniform distributions over the interval \([0, 1]\). For sNMF, these matrices are sparsified using hard thresholding, resulting in \( 30\% \) of their entries being set to zero. In the case of sNTD, the core tensor is sampled similarly and sparsified such that \( 70\% \) of its entries are null.

The hyperparameters are selected as follows: We set \( \epsilon = 10^{-16} \). The initial parameter matrices are drawn like the true ones; however, the first matrix \( X_1 \) is scaled by \( 100 \). This results in unbalanced parameter matrices, which serves to emphasize the effect of the balancing procedure.
We optimally scale the initialization with respect to the \( \ell_2 \) loss, as described in \cref{alg:meta-Algo}. The number of outer iterations is set to \( 500 \) for both sNMF and sNTD, while it is set to \( 50 \) for rNCPD. The number of inner iterations is set to \( 10 \). Each experiment is repeated \( 50 \) times. The rNCPD problem is addressed using the HALS algorithm, which is known for its fast convergence, thereby requiring fewer outer iterations.

We conduct several additional tests specific to each model. For sNMF, to validate that the individual values of \( \mu_1 \) and \( \mu_2 \) do not influence the relative sparsity levels of \( X_1 \) and \( X_2 \), we examine the ratio \( \frac{\|X_1\|_0}{\|X_2\|_0} \) in relation to the regularization hyperparameters. An entry in a matrix is considered null if it is below the threshold \( 2\epsilon \). Moreover, we fix the regularization hyperparameter \( \mu_1 = 1 \), which allows us to study the relative sparsity of the parameter matrices concerning the individual values of the regularization hyperparameters. 
For rNCPD, to investigate the number of zero components, we overestimate the number of components for the approximation \( r_e \) (estimated rank) to \( 6 \). We anticipate that the ridge penalization will function as a rank penalization, effectively pruning two components. A component is deemed null if the weight associated with it (the product of the \( \ell_2 \) norms of each parameter matrix along each column) is lower than \( 1000\epsilon \) (there are \( 30 \times 6 = 180 \) entries). 
For sNTD, the first estimated rank \( r_1 \) is also set to \( 6 \) instead of \( 4 \), and we analyze both the number of nonzero first-mode slices and the overall sparsity in the core tensor \( \mathcal{G} \).


\paragraph{Experiments Results}

Results are shown in~\cref{fig:figs_NMF_loss_fms,fig:figs_CP_loss_fms,fig:figs_NTD_loss_fms}.
We validate several observations from the theoretical derivations:
\begin{itemize}
    \item For small values of \( \mu \), as predicted in \cref{sec:balance_speed}, the meta-algorithm without explicit balancing struggles to converge in terms of loss value. However, as shown in the supplementary material \cref{fig:fms}, the parameter matrices are still well estimated without balancing. This indicates that while the unbalanced algorithm identifies good candidate solutions for small \( \mu \), it requires many iterations to balance these estimates. Explicit balancing reduces the regularization terms, resulting in significantly lower costs after running the algorithm. Balancing the initial guess offers less improvement compared to balancing at each iteration, except for sNMF.
    \item The observations derived from the implicit formulations are validated numerically. The relative sparsity levels in the sNMF parameter matrices remain close to \( 1 \), regardless of the ratio \( \frac{\mu_2}{\mu_1} \). For a wide range of regularization hyperparameter values, the estimated number of components in rNCPD matches the true count of four, indicating that ridge regularization effectively performs automatic rank selection. In sNTD, since all algorithms should converge to similar solutions, the unbalanced method produces solutions denser than expected. Thus, balancing aids in identifying zero entries. The decomposition rank for the first mode was set larger than the true rank, and the balanced algorithm more reliably prunes the extra components.
\end{itemize}

\begin{figure}
    \centering
    \includegraphics[width=7.7cm]{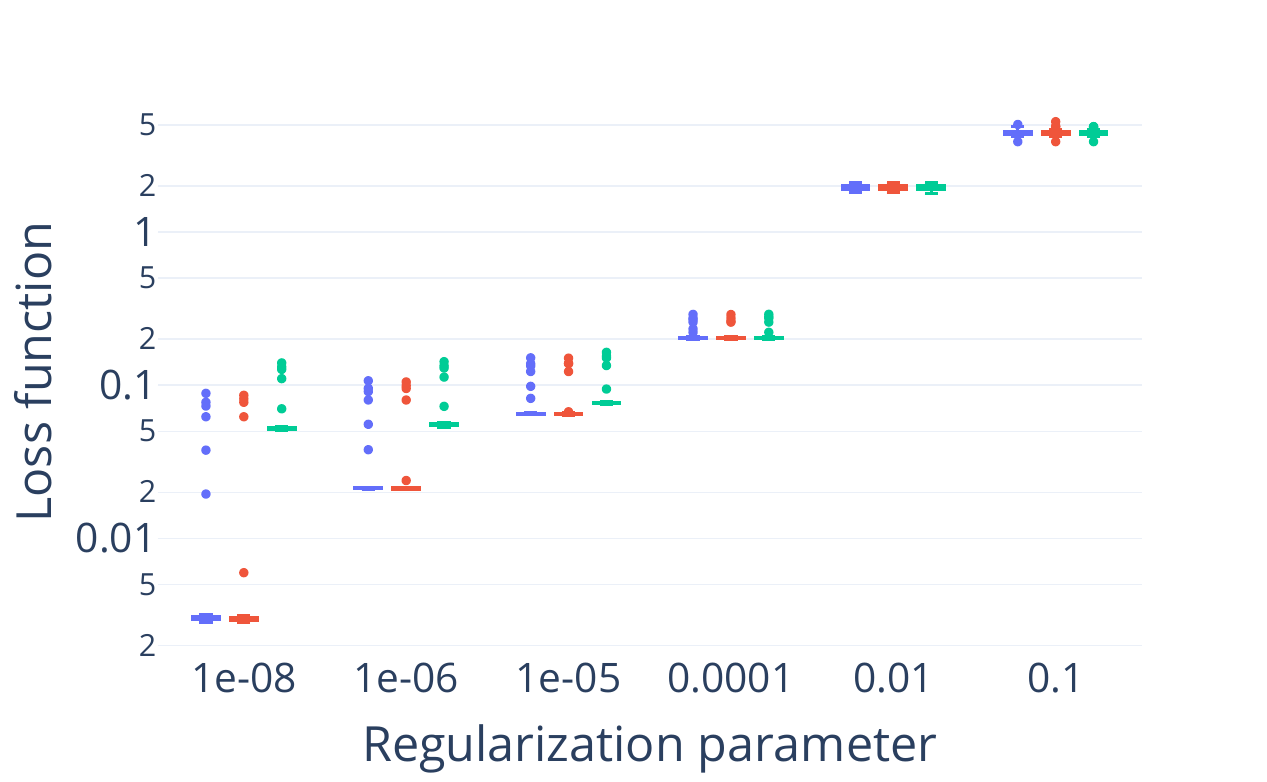}
    \includegraphics[width=7.7cm]{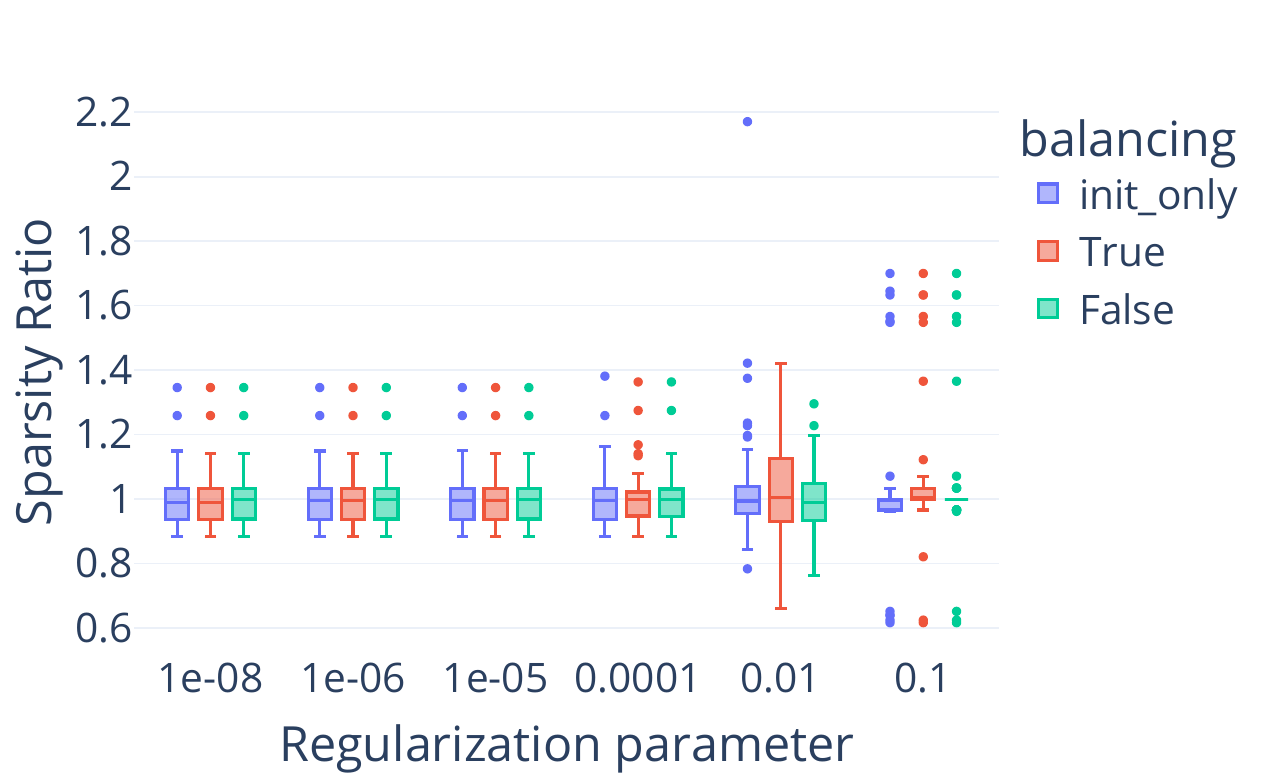}
    \caption{Results for the simulated experiment on sNMF.
    The left plot shows the loss after the stopping criterion is reached, plotted against the regularization parameter $\mu_2$ while $\mu_1=1$. The right plot displays the ratio of the sparsity of the estimated matrix $X_1$ to that of the estimated matrix $X_2$. Results are presented for three scenarios: balancing used only at initialization (blue), balancing applied at each iteration (red), and no balancing (green). Notably, the left plot indicates that balancing reduces the loss function value upon reaching the convergence criterion, particularly for the weakest regularizations. In the right plot, it is observed that tuning the regularization hyperparameter $\mu_2$ while fixing $\mu_1=1$ does not affect the sparsity ratio of the parameter matrices $X_2$ and $X_1$, confirming that the sparsity of matrix $X_2$ cannot be adjusted independently of matrix $X_1$.
    }
    \label{fig:figs_NMF_loss_fms}
\end{figure}

\begin{figure}
    \centering
    \includegraphics[width=7.7cm]{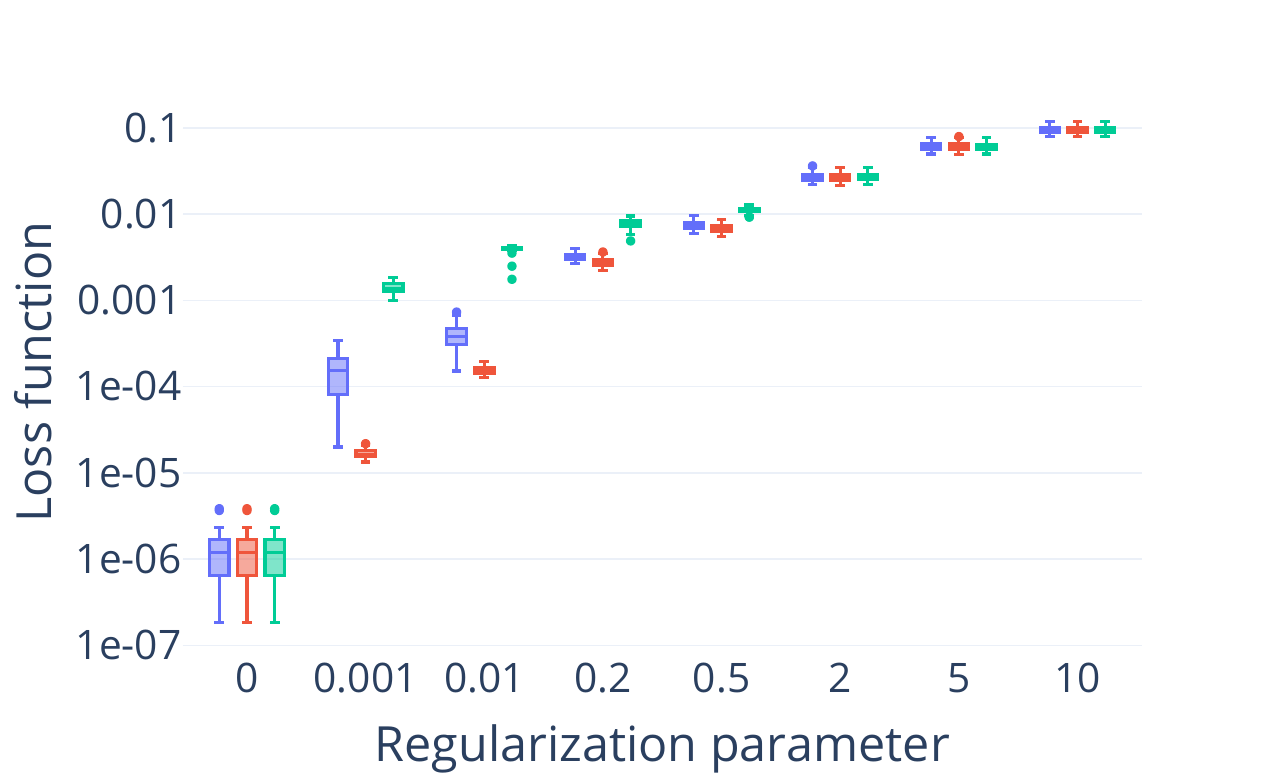}
    \includegraphics[width=7.7cm]{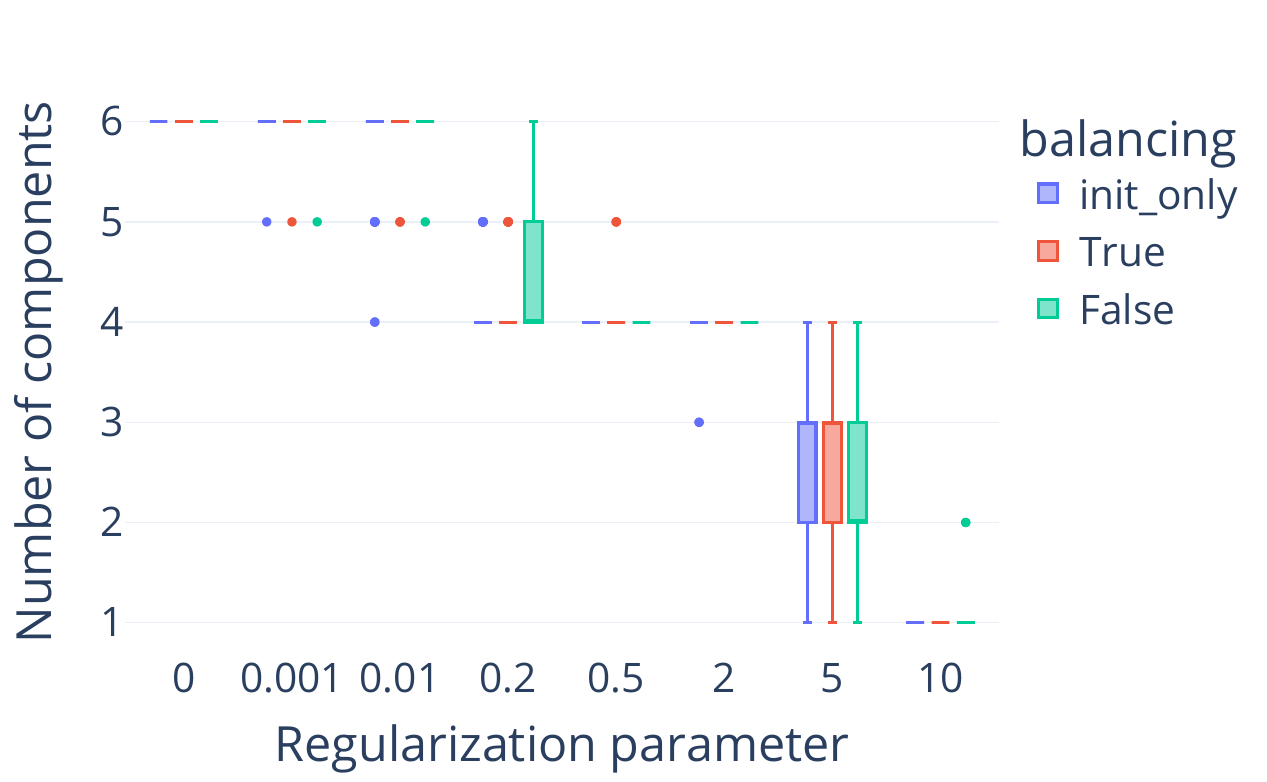}
    \caption{Results for the simulated experiment on ridge Nonnegative CPD. The normalized loss after the stopping criterion is reached is shown with respect to the regularization parameter $\mu$ on the left plot, while 
    the number of components within the estimated rNCPD model with respect to the regularization parameter is shown on the right plot. Results when balancing is used at initialization only, at each iteration, and not used are shown respectively in blue, red, and green. One may observe on the left plot that balancing helps reduce the loss function value when the convergence criterion is reached. This reduction is significant in particular for small regularization parameter $\mu$ and when balancing is performed at each iteration, although balancing only the initialization already improves over no balancing. On the right plot, one may observe that for a wide range of ridge regularization hyperparameter values, the number of estimated components in rNCPD matches true rank $r=4$. Therefore ridge regularization has an implicit group sparse action on the rank-one components as predicted in theory. This phenomenon happens regardless of the balancing procedure.
    }
    \label{fig:figs_CP_loss_fms}
\end{figure}

\begin{figure}
    \centering
    \includegraphics[width=7.7cm]{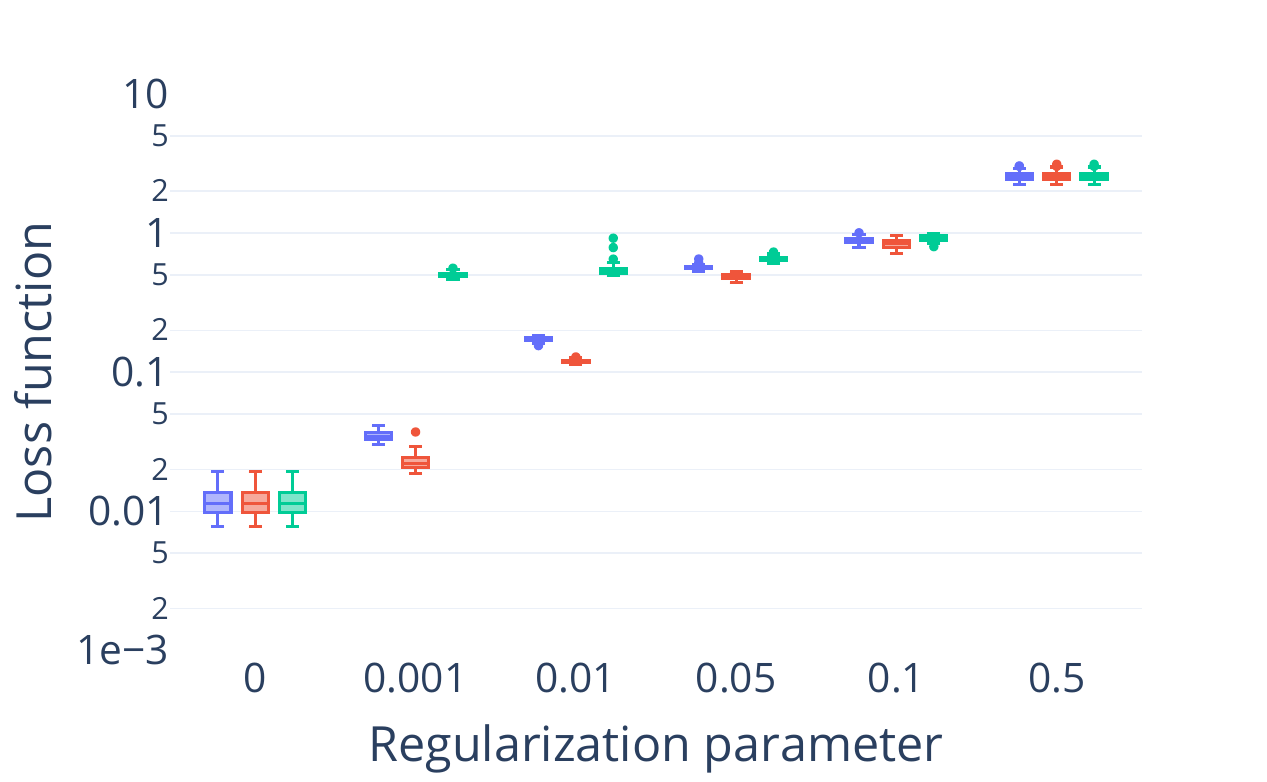}
    \includegraphics[width=7.7cm]{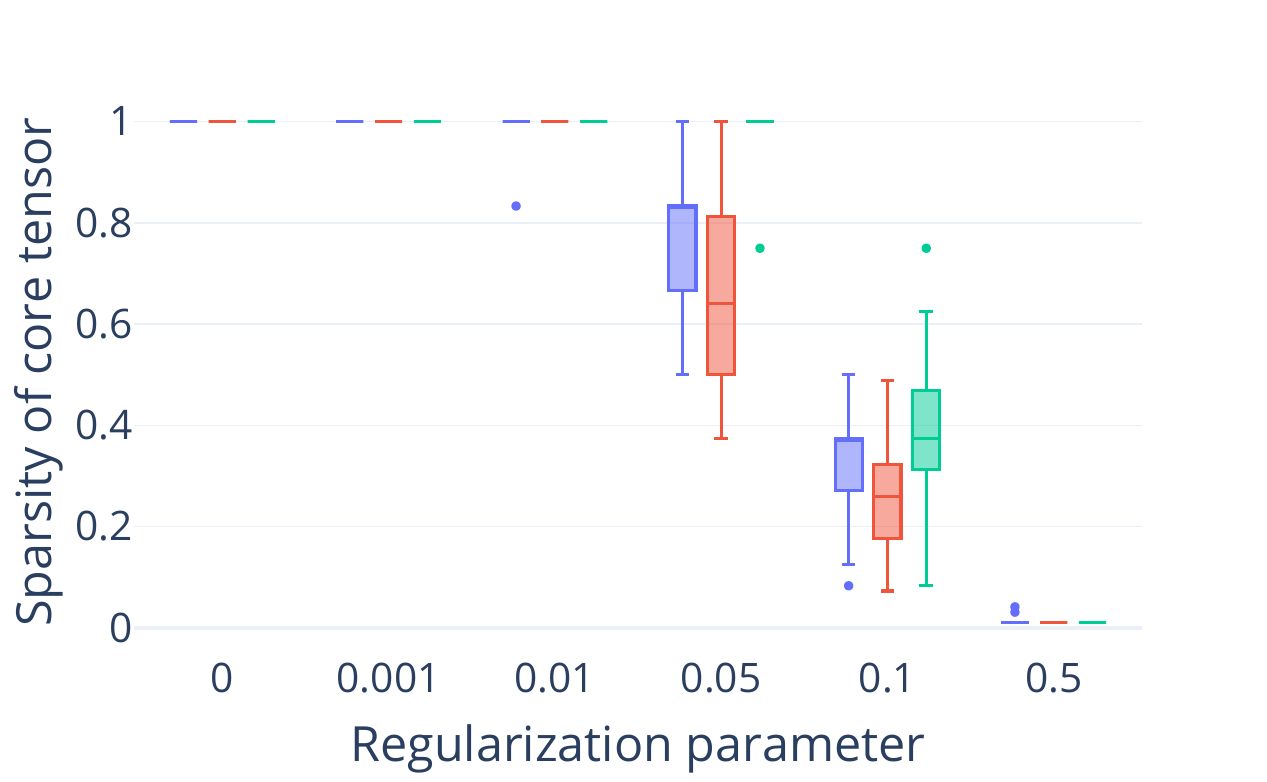}
    \includegraphics[width=7.7cm]{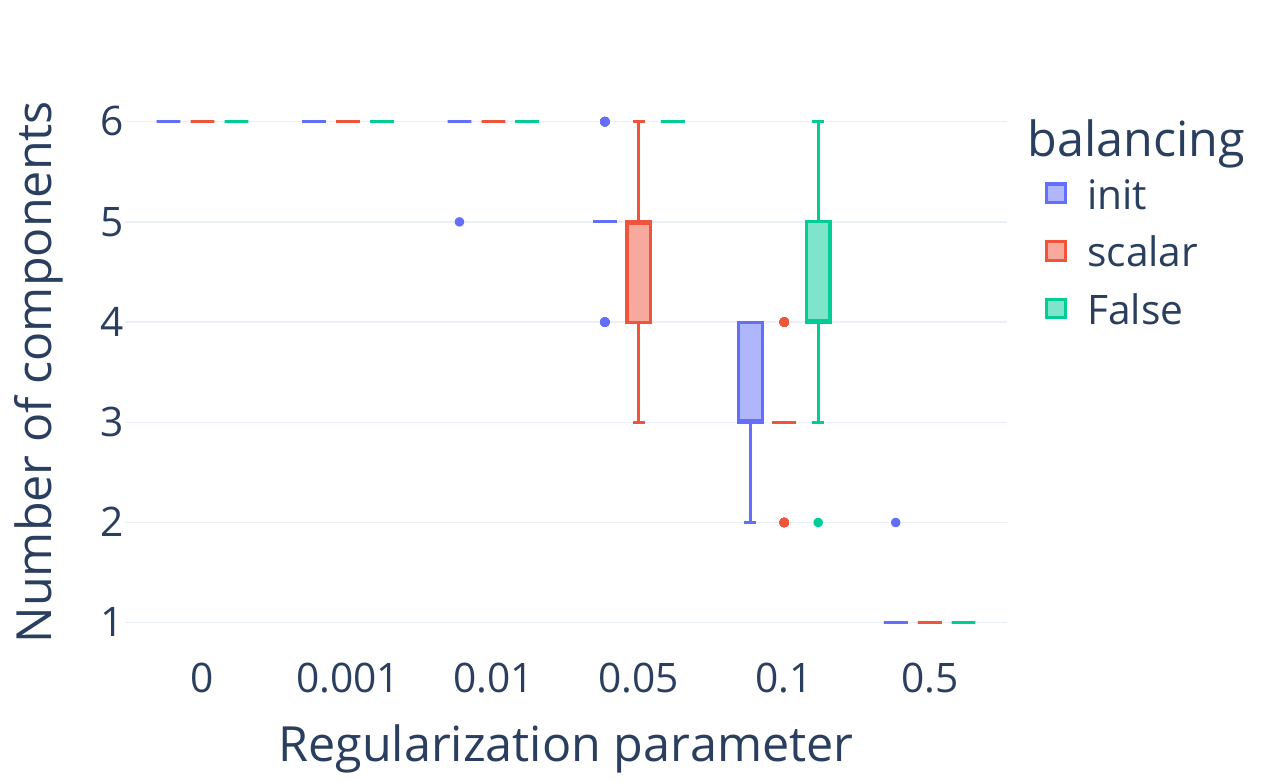}
    \caption{Results for the simulated experiment on sparse NTD.
    The normalized loss after 500 outer iterations is presented in the top left plot, showing its relationship with the regularization parameter $\mu$. The top right plot displays the scaled sparsity of the estimated core tensor against the regularization parameters, while the bottom plot illustrates the number of nonzero components along the first mode of the estimated core. Results are shown in blue for balancing used only at initialization, in red for balancing at each iteration (using scalar balancing), and in green for no balancing. The top left plot indicates that balancing effectively reduces the loss function value upon reaching the convergence criterion, with a notable reduction for small regularization parameter $\mu$, especially when balancing is applied at each iteration. Even balancing only at initialization improves results compared to no balancing. In the top right and bottom plots, $\ell_1$ regularization on the core tensor promotes sparsity and helps select the appropriate number of components, $r_1=4$. Balancing—whether at initialization or during iterations—broadens the range of regularization values for which components are effectively pruned.
    }
    \label{fig:figs_NTD_loss_fms}
\end{figure}

\subsubsection{sNTD for Music Segmentation}

We showcase sNTD with balancing in a music redundancy detection task. The use of NTD for analyzing music was proposed in~\cite{smithNonnegativeTensorFactorization2018} and further developed in~\cite{marmoret2020uncovering}. For an audio recording of a song segmented into bars, a tensor spectrogram is constructed by stacking time-frequency spectrograms computed from each bar of the audio signal. Each column of the time-frequency spectrogram matrix captures frequency information about the one-dimensional audio signal within a small time window. The resulting data forms a tensor spectrogram $\tT$ of size $F \times T \times B$, where $F$ is the number of frequency bins, $T$ is the number of time bins, and $B$ is the number of bars in the song. By applying NTD to the tensor spectrogram, redundancies across the bars can be detected, allowing for the factorization of spectral and short-term temporal content into a collection of patterns. Notably, the redundancy information is contained in the last mode parameter matrix $X_3$, which can be utilized for downstream tasks such as automatic segmentation.
Each column $X_3[:,q]$ of matrix $X_3$ corresponds to a specific pattern, defined by the product of the other parameters $X_1$, $X_2$, and the slice $\mathcal{G}[:,:,q]$ of tensor $\mathcal{G}$. Each row of matrix $X_3$ maps these patterns to the bar indices. Matrix $X_3$ visualizes the relationship between patterns and bars, with the expectation that bars containing similar content will utilize the same patterns. Consequently, matrix $X_3$ should exhibit sparsity, as only a limited number of patterns are employed for each bar.

We consider sNTD with sparsity imposed on the parameter matrix $X_3$ instead of $\mathcal{G}$. The derivations for this variant are similar to those of sNTD with a sparse core tensor and are therefore omitted. The song selected for this experiment is "Come Together" by the Beatles, which has been used previously to demonstrate NTD's ability to detect patterns and redundancies in music. After processing the song as described in~\cite{marmoret2020uncovering}, we obtain a tensor of size $F=80$, $T=96$, and $B=89$. The core dimensions are set to $32\times 12\times 10$, indicating that the song should be factorized into at most $10$ patterns. The algorithm runs for $50$ iterations, with $10$ inner iterations per run. The regularization parameter is chosen from the grid $[0, 10, 100, 300, 500]$. Initialization is performed using HOSVD~\cite{DeLathauwer2004Dimensionality}. While the choice of initialization method is important, we opted for the method proposed in~\cite{marmoret2020uncovering} without further exploration.

\Cref{fig:figs_NTD_audio} displays the loss function for two variants of the sNTD algorithm: one without balancing and one with scalar balancing, along with the estimated parameter matrices $X_3^T$ for each setup. The expert, ground truth segmentation of the song is indicated by green vertical lines. Notably, the balanced algorithms converge faster than the unbalanced variant. Additionally, the sparsity imposed on $X_3$ significantly affects the quality of redundancy detection. Ideally, the pattern indices used for each bar should align closely with the expert segmentation, resulting in group sparsity between the green vertical lines in the rows of the matrices shown in~\cref{fig:figs_NTD_audio}. This alignment is particularly evident with a regularization parameter of three hundred. At this level, balancing visibly influences the final sparsity pattern of matrix $X_3$, although it is arguably difficult to assess the superiority of the balanced algorithm from visual inspection only. 

\begin{figure}
    \centering
    \includegraphics[width=14cm]{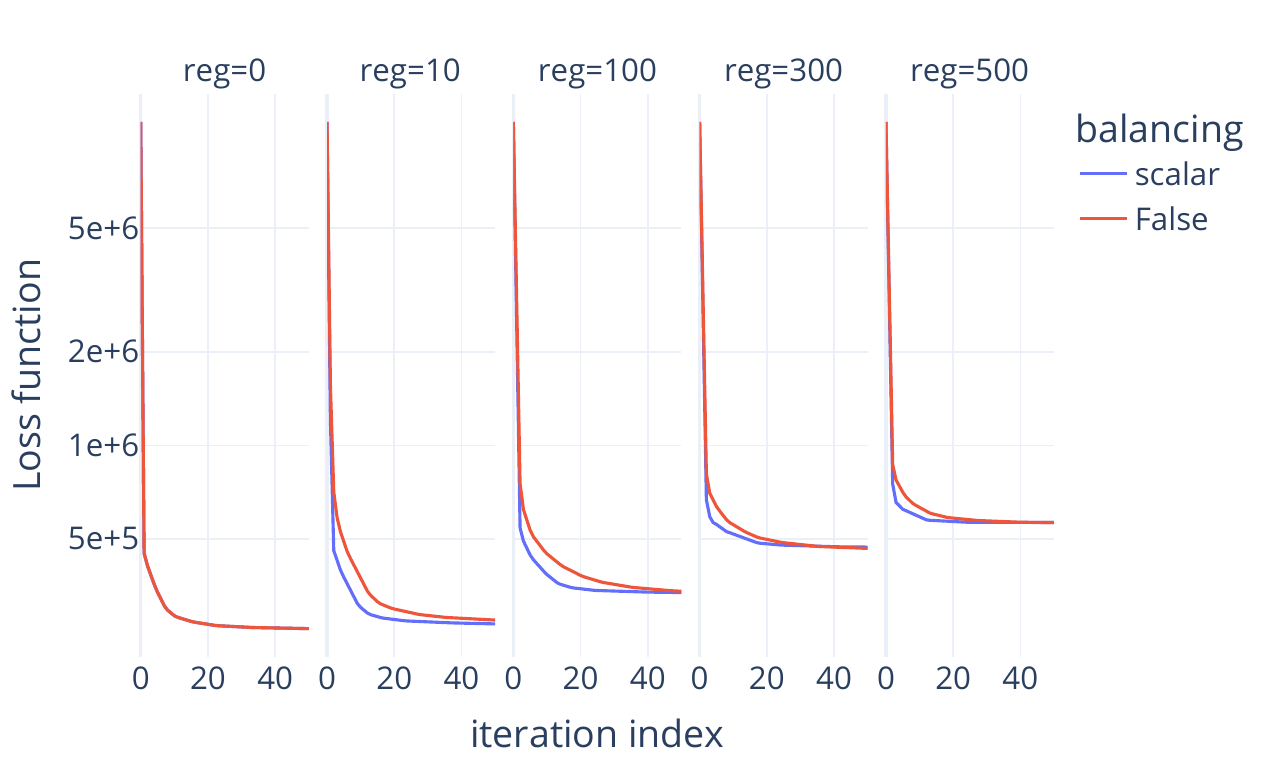}
    \includegraphics[width=14cm]{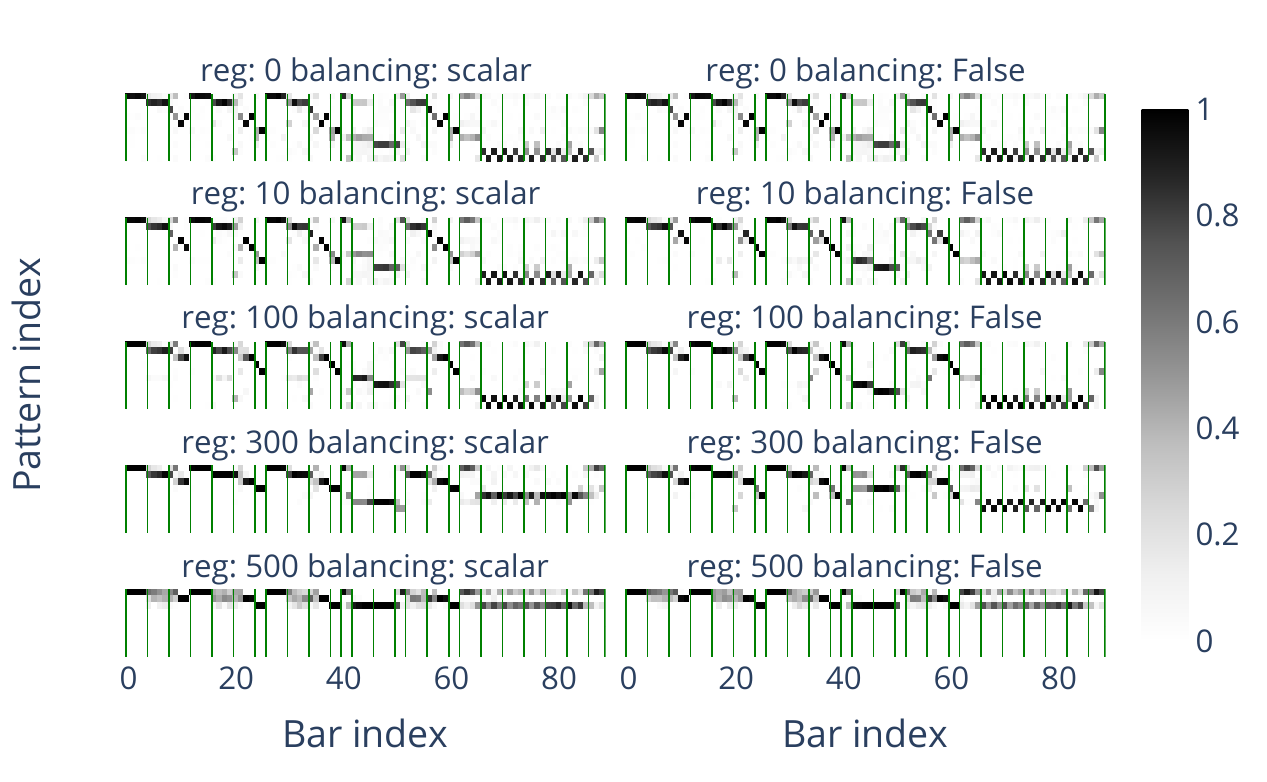}
    \caption{Results for the sNTD experiment on music redundancy detection. \textbf{Top:} the loss function for each tested algorithm is shown across iterations for various values of the regularization hyperparameters. \textbf{Bottom:} The reconstructed matrix $X_3$ is shown, transposed, for each tested algorithm and various values of the regularization hyperparameter. Green vertical bars indicate an expert segmentation of the song (Come-Together). The number of patterns
(rows) is hard to estimate a priori, therefore ideally the estimated matrix $X_3$ should have sparse rows to prune unnecessary patterns. Moreover, it should be group sparse between the green bars to allow for easier segmentation by downstream methods such as dynamic programming or K-means.
    }
    \label{fig:figs_NTD_audio}
\end{figure}

\section{Conclusions and perspectives}\label{sec:conclu}

 The design of efficient algorithms for computing low-rank approximations is an active research field, in particular in the presence of constraints and penalizations. In this work, we have shed light on the hidden effects of regularizations in nonnegative low-rank approximations, by studying a broader class of problems we term Homogeneous Regularized Scale-Invariant problems. We have shown that the scale-ambiguity inherent to these problems induces a balancing effect of penalizations imposed on several parameter matrices. This balancing effect has both practical and theoretical implications. In theory, implicit balancing induces implicit regularizations effects. For instance, imposing ridge penalizations on all parameter matrices in a low-rank approximation is equivalent to using a sparsity-inducing penalization on rank-one components of the model. In practice, enforcing this implicit balancing helps alternating algorithms efficiently minimize the cost function without hurting convergence guarantees. This is demonstrated formally on a toy problem, and practically in three different showcases, namely sparse Nonnegative Matrix Factorization, ridge Nonnegative Canonical Polyadic Decomposition, and sparse Nonnegative Tucker Decomposition.

Our findings open several avenues for future exploration. First, the implicit equivalent formulation of regularized low-rank approximations could be leveraged to derive guarantees regarding the solution space of these models. For instance, it remains unclear under what conditions $\ell_1$-regularized Nonnegative Matrix Factorization yields zero, sparse, or dense solutions. Additionally, our framework is currently limited to positive homogeneous and positive-definite regularizations applied column-wise to parameter matrices of scale-invariant models. While positive homogeneity and positive definiteness allow for a wide range of regularizations, including $\ell_p^p$ norms, some regularizations and constraints commonly used in the literature do not satisfy these criteria. Thus, it would be valuable to investigate how relaxing these assumptions affects our findings. 

\section*{Acknowledgements}
The authors would like to thank Rémi Gribonval for an early assessment of this work and several pointers toward relevant literature about neural network training. The authors also thank Cédric Févotte and Henrique De Morais Goulart for helpful discussions around the implicit and explicit formulations, as well as pointers toward relevant literature in sNMF.

\newpage 

\appendix

\section{HRSI with unregularized modes is ill-posed}\label{AppendB}
First, let us consider without loss of generality the case $n=2$, $r=1$ and $\mu_2=0$ for Problem \eqref{eq:l1prop11}:
\begin{equation}\label{eq:l1prop}
    \phi(X_1,X_2) = f(X_1,X_2) + \mu_1 g(X_1)
\end{equation}
Function $\phi$ is bounded below by zero because $f$ is nonnegative, therefore it must admit an infimum over $\mathbb{R}^{m_1\times r}\times \mathbb{R}^{m_2\times r}$. Let us denote by $p$ this infimum.

We are going to show that, given $l:=\inf_{X_1,X_2} f(X_1,X_2)$, it holds that $p=l$. We know that $p\geq l$ since both $g$ and $\mu_1$ are nonnegative. To prove the converse inequality, note that from Equation~\eqref{eq:l1prop},
for any $\lambda<1$, it holds that $\phi(X_1,X_2)>\phi(\lambda X_1, \frac{1}{\lambda}X_2)$. Therefore for any admissible $(X_1,X_2)$,
\begin{equation}
    p \leq \lim_{\lambda\to 0} \phi(\lambda X_1, \frac{1}{\lambda} X_2).
\end{equation}

Now we may notice that $\phi(\lambda X_1,\frac{1}{\lambda}X_2) = f(X_1,X_2) + o(\lambda)$, and therefore prolonging by continuity the right hand side, we get that $\lim_{\lambda\to 0} \phi(\lambda X_1, \frac{1}{\lambda}X_2) = f(X_1,X_2)$.  This shows that $p \leq l$, which concludes our proof that $p=l$.

One can easily note that for the infimum $p$ to be attained, say at $(X_1^\ast, X_2^\ast)$, one must have $\phi(X_1^\ast,X_2^\ast)=f(X_1^\ast, X_2^\ast)$ and therefore $g(X_1^\ast)=0$, \textit{i.e.} $X_1^\ast=0$ and $p=f(0,X_2^\ast)$ since $g_1$ is positive definite. By contraposition, as soon as $f(0,X_2^\ast)$ is not the infimum of $f$ on the set of matrices with rank at most $r$, the infimum $p$ cannot be attained. Finally, note that $f(0,X_2^\ast)=f(0,0)$ by scale-invariance.

\section{Closed-form expression to solutions of the balancing procedure}\label{app:balancing_eq}
We are interested in computing the solution to the following problem:
\begin{equation}\label{eq:scale_cost_appendix}
    \underset{\forall i\leq n,\;\lambda_i}{\min} \sum_{i=1}^{n} \lambda_i^{p_i} a_i  \text{ such that } \prod_{i=1}^{n}\lambda_i = 1
\end{equation}
Note that we dropped the nonnegativity constraints which are not required to reach our closed-form solution. A particular solution will be naturally nonnegative as long as all products $p_ia_i$ are nonnegative. We also suppose that each constant $a_i$ is strictly positive (see~\cref{sec:illposed} for a justification).

We use the KKT conditions to compute the solutions, as is traditionally done for the related variational formulation of the geometric mean. Indeed,
\begin{equation}\label{eq:geo_mean_def_var}
    \underset{\{\lambda_i\}_{i\leq n}}{\argmin} \sum_{i=1}^{n} \lambda_i  \text{ such that } \prod_{i=1}^{n}\lambda_i = \prod_{i=1}^{n} x_i
\end{equation}
is an alternative definition of the geometric mean $\beta=\left(\prod_{i\leq n}x_i\right)^{\frac{1}{n}}$.

Let us define the Lagrangian for our problem,
\begin{equation}\label{eq:Lag_geomean_var}
    \mathcal{L}(\{\lambda_i\}_{i\leq n}, \nu) = \sum_{i\leq n} \lambda_i^{p_i}a_i + \nu (1-\prod_{i\leq n}\lambda_i) ~.
\end{equation}
Optimality conditions can be written as follows:
\begin{equation}
    \begin{cases}
        \nabla_{\lambda_i} \mathcal{L}(\{\lambda_i^\ast\}_{i\leq n}, \nu^\ast) = p_ia_i {\lambda_i^\ast}^{p_i-1} - \nu^\ast \prod_{j\neq i} \lambda_j^\ast = 0 \\
        1 = \prod_{i\leq n} \lambda_i^\ast
    \end{cases}
\end{equation}
Using the second equality, we can inject $\lambda^\ast_i$ in the right hand side of the first equation, which yields the following necessary condition
\begin{equation}
    \lambda^\ast_i = \left(\frac{\nu^\ast}{p_ia_i}\right)^{\frac{1}{p_i}}~.
\end{equation}
There may be a sign ambiguity depending on the power value $p_i$, in which case we can chose the positive solution since the original problem involved nonnegativity constraints on $\lambda_i$. Moreover, we have used the facts that $p_ia_i$ are nonzero, and that $\lambda_i^\ast$ are also nonzero. This last condition can easily be checked since the constraint $\prod_{i\leq n} \lambda_i^\ast = 1$ prevents any optimal $\lambda_i^\ast$ to be null.

To compute the optimal dual parameter, we can use again the primal feasibility condition:
\begin{equation}
    \prod_{i\leq n}\left(\frac{\nu^\ast}{p_ia_i}\right)^{\frac{1}{p_i}} = 1
\end{equation}
which, after some simple computation, shows that $\nu^\ast$ is exactly the weighted geometric mean of products $p_ia_i$ with weights $\frac{1}{p_i}$,
\begin{equation}
    \nu^\ast = \left(\prod_{i\leq n} \left(p_ia_i\right)^{\frac{1}{p_i}}\right)^{\frac{1}{\sum_{i\leq n} \frac{1}{p_i}}}
\end{equation}

Note that when applying this balancing operation to an optimal solution ${X_i^\ast}$ of HRSI, the scales $\lambda_i$ must equal one, therefore at optimality, 
\begin{equation}
    p_ia_i = \nu^\ast
\end{equation}

If the solution $X_i^\ast$ is null, then $\nu^\ast$ is also null and the result holds.

\section{Proof of sublinear convergence rate for ALS on the scaling problem}\label{app:proof_als_sublin}

We will study the variation of loss function when $x_1$ is updated, $x_2$ fixed, across one iteration. The problem is symmetric therefore we will generalize for all updates of either $x_1$ or $x_2$. 

Let us introduce a few notations. Let $k$ the iteration index, we set $x_i:=x_i^{(k)}, \tilde{x}_1 = x_1^{(k+1)}$. We denote the residuals $r_{k} = y - x_1x_2$ and $r_{k+1} = y - \tilde{x}_1x_2 $, and the regularizations $g_k = \lambda (x_1^2 + x_2^2)$ and $g_{k+1} = \lambda (\tilde{x}_1^2 + x_2^2)$. The cost function value at iteration $k$ is $f_k = r_k^2 + g_k$. Finally, we denote $\delta r_k = r_{k} - r_{k+1}$ the variation in residuals.
Throughout the following derivations, we suppose that $\lambda$ is much smaller than $y$, but also $x_1^2$ and $x_2^2$. Since both $x_1^2$ and $x_2^2$ will eventually converge towards $y-\lambda$, this is equivalent to assuming that $\lambda$ is much smaller than $y$, and $k$ is large enough so that $x_1$ and $x_2$ are relatively close to $\sqrt{y-\lambda}$.

We study the variation of the cost $\delta f_k = f_k - f_{k+1}$. We observe that
\begin{align}
    f_{k+1} &= (-\delta r_k + r_k)^2 + g_{k+1} - g_k + g_k \\
            &= f_k - 2r_k\delta r_k + (\delta r_k)^2 + (g_{k+1} - g_k),
\end{align}
and therefore $\delta f_k =  2r_k\delta r_k - (\delta r_k)^2 + (g_{k} - g_{k+1}) $.
We can expend the update of the regularization terms:
\begin{align}
    g_{k} - g_{k+1} &= \lambda \left( x_1^2 - \tilde{x}_1^2 \right) \\
                    &= \frac{\lambda}{x_2^2} \left( (y^2-\tilde{x}_1^2x_2^2) - (y^2 - x_1^2x_2^2) \right) \\
                    &= \frac{\lambda}{x_2^2} \left( r_{k+1} (y+\tilde{x}_1x_2) - r_{k} (y + x_1x_2) \right)
\end{align} 
Recall that $\tilde{x}_1 = \frac{y}{x_2} \frac{1}{1+\frac{\lambda}{x_2^2}}$\footnote{ALS algorithm updates each variable as follows $\tilde{x}_1=\argmin_{x_1} (y-x_1x_2)^2 + \lambda(x_1^2 + x_2^2)$, solving the KKT condition leads to $\tilde{x}_1=\frac{y x_2}{\lambda + x_2^2}=\frac{y}{x_2} \frac{1}{1+\frac{\lambda}{x_2^2}}$.}. We perform a first approximation here of $\tilde{x}_1(\lambda)$ by using its Taylor expansion around $\lambda=0$, which hold approximately when $\frac{\lambda}{x_2^2}\ll 1$. Then we get 
\begin{equation}
    \tilde{x}_1 = \frac{y}{x_2}(1-\frac{\lambda}{x_2^2} + \mathcal{O}(\lambda^2)) = \frac{y}{x_2} - \frac{y\lambda}{x_2^3} + \mathcal{O}(\lambda^2),
\end{equation}
This means that 
\begin{equation}
    y+\tilde{x}_1x_2  = 2y - \frac{y\lambda}{x_2^2} + \mathcal{O}(\lambda^2).
\end{equation}
Furthermore, under the same hypothesis, we can easily see that the residuals is of order of magnitude $\lambda$:
\begin{align}
    r_{k+1} &= y-\tilde{x}_1x_2 \\
            &= \frac{y\lambda}{x_2^2} + \mathcal{O}(\lambda^2)
\end{align}
and therefore, we can relate the sums $y+x_1x_2$ and $y+\tilde{x}_1x_2$ to the residuals as
\begin{align}
    y + \tilde{x}_1x_2 = 2y - r_{k+1} + \mathcal{O}(\lambda^2) \\
    y + x_1x_2 = 2y - r_{k} + \mathcal{O}(\lambda^2)~.
\end{align}
This yields the following formula for the variation of regularization terms
\begin{align}
    g_{k} - g_{k+1} &= \frac{\lambda}{x_2^2} \left( 2y(r_{k+1} - r_{k}) + (r_k^2 - r_{k+1}^2)  + \mathcal{O}(\lambda^3)\right) \\
                    &= - 2r_{k+1}\delta r_k + \mathcal{O}(\lambda^3)~.
\end{align}
We have assumed that $\delta r_k$ is also of the order of magnitude $\lambda$, which is shown in the derivations below.
Finally, the variation in loss function can be written as
\begin{align}
    \delta f_k &= -2r_{k+1}\delta r_k + 2r_k\delta r_k - (\delta r_k)^2 + \mathcal{O}(\lambda^3) \\
               &= 2\delta r_k (r_{k} - r_{k+1}) - (\delta r_k)^2 + \mathcal{O}(\lambda^3) \\
               &= (\delta r_k)^2 + \mathcal{O}(\lambda^3)\label{eq:dfdrk}
\end{align}
To further detail the expression of $\delta r_k$, we may first relate $\tilde{x}_1$ with $x_1$ since
\begin{equation}
    \tilde{x}_1 = \frac{y}{x_2} \frac{x_2^2}{x_2^2 + \lambda} = \frac{y}{x_2}\left(1-\frac{\lambda}{x_2^2+\lambda}\right).
\end{equation}
Viewing $x_2$ as the least squares update from $x_1$, we further obtain


\begin{equation}
    \tilde{x}_1 = x_1 \frac{1-\frac{\lambda}{x_2^2 + \lambda}}{1-\frac{\lambda}{x_1^2 + \lambda}}
\end{equation}
Therefore we get
\begin{align}
    \delta r_k^2
    &= x_2^2 (x_1 - \tilde{x}_1)^2  \\
    &= x_1^2 x_2^2 \left(1 -\frac{1-\frac{\lambda}{x_1^2-\lambda}}{1-\frac{\lambda}{x_2^2+\lambda}} \right)^2 \\
    &= x_1^2x_2^2\lambda^2 \left(\frac{\frac{1}{x_2^2+\lambda} - \frac{1}{x_1^2+\lambda}}{1-\frac{\lambda}{x_2^2+\lambda}}\right)^2
\end{align}
The denominator can be neglected since it will generate terms in the order of $\lambda^4$ as soon as $\lambda$ is small with respect to $x_2^2$, which is our working hypothesis. This yields
\begin{align}
    \delta r_k^2
    &= x_1^2x_2^2\lambda^2 \left(\frac{x_1^2-x_2^2}{(x_1^2+\lambda)(x_2^2+\lambda)} \right)^2 + \mathcal{O}(\lambda^3) \\
    &= \frac{x_1^2x_2^2(x_1+x_2)^2}{(x_1^2+\lambda)^2(x_2^2+\lambda)^2}\lambda^2(x_1-x_2)^2 + \mathcal{O}(\lambda^3)
\end{align}
Let us introduce the errors $e_1,e_2$ from $x_1,x_2$ to their target $y-\lambda$:
\begin{align}
    x_1 = \sqrt{y-\lambda} + e_1 \\
    x_2 = \sqrt{y-\lambda} + e_2 
\end{align}
We may note that $e_1$ and $e_2$ are of opposite sign as soon as $k>1$, and for simplicity of presentation, we will assume that $e_1\approx -e_2 = e_k$.
We further assume that $k$ is large enough so that $e_1,e_2$ are small enough (of order of magnitude $\mathcal{O}(\lambda)$) to approximate $(x_1+x_2)^2$ as $4y$ 
and $(x_1+\lambda)^2,(x_2+\lambda)^2$ as $y$ up to an error $\mathcal{O}(\lambda)$. The term $x_1-x_2$ is of order $\mathcal{O}(\lambda)$. Therefore for large enough $k$,
\begin{align}
    \delta f_k &= y^2\frac{4y}{y^2y^2}\lambda^2(x_1-x_2)^2 +\mathcal{O}(\lambda^3) \\
    &= 4\frac{\lambda^2}{y}(x_1-x_2)^2 + \mathcal{O}(\lambda^3) \\
    &= 16\frac{\lambda^2}{y}e_k^2 + \mathcal{O}(\lambda^3)
\end{align}

We are left to characterize the convergence speed of $e_k$ towards zero. We can see that $e_k$ converges linearly with however a multiplicative constant that goes to one as $\frac{\lambda}{y}$ grows smaller. Indeed,
\begin{align}
    e_{k+1} &= \tilde{x}_1 - x^\ast \\
    &=  x_1 - x^\ast - x_1 \left( 1 - \frac{1-\frac{\lambda}{x_2^2 + \lambda}}{1-\frac{\lambda}{x_1^2 + \lambda}}  \right) \\
    &= x_1 - x^\ast - x_1\lambda\left( \frac{x_1^2-x_2^2}{(x_1^2+\lambda)(x_2^2+\lambda)}\right) + \mathcal{O}(\lambda^2) \\
    &= x_1 - x^\ast - \frac{2\lambda}{y}\left(x_1-x_2\right) + \mathcal{O}(\lambda^2)
\end{align}
where the last two lines are obtained by reusing the previous derivations for $\delta r_k$. We find that

\begin{equation}
    e_{k+1} = e_k - e_k 4\frac{\lambda}{y} + \mathcal{O}(\lambda^2)
\end{equation}
We see that $\frac{e_{k+1}}{e_k}$ is approximately constant, equal to $1-4\frac{\lambda}{y}$, which proves that ALS is a sublinear algorithm when $\frac{\lambda}{y}$ goes to zero.

\section{Convergence of iterates for \cref{alg:meta-Algo}}\label{AppendC}
In this supplementary section we discuss in details the convergence guarantees offered by~\cref{alg:meta-Algo} w.r.t the sequence of objective functions and the sequence of iterates. 

For convenience, let us briefly recall the formulation of our optimization problems as a general class of multiblock  optimization problems; denoting $z=(\{X_i[:,q]\}_{i \leq n, q \leq r})$ (concatenation of the variables) and $\phi$ the objective function from Problem~\eqref{eq:hrsi}, we consider the problems
\begin{equation}
    \label{eq:GenOptiProblem}
    \min_{z_i\in\mathcal{Z}_i} \phi(z_1,\ldots,z_s),
\end{equation}
where $z$ can be decomposed into $s$ blocks $ z=(z_1,\ldots,z_s)$, $\mathcal{Z}_i\subseteq\mathbb{R}^{m_i}$  is a closed convex set, $\mathcal{Z}=\mathcal{Z}_1\times\ldots\mathcal{Z}_s \subseteq \text{dom} (\phi)$, $\phi:\mathbb{R}^m\to\mathbb{R}\cup\{+\infty\}$, $m=\sum_{i=1}^s m_i$, is a continuous differentiable function over its domain, and $\phi$ is lower bounded and has bounded level sets. 

In~\cref{theo:monotoneNI} we show that~\cref{alg:meta-Algo} generates a sequence of iterates $\{z^i\}_{i=1}^{\infty}$ such that the sequence of objective values $\{ \phi(z^i) \}$ is monotone non-increasing.
First, let us recall the formal definition of a  tight majorizer for convenience.

\begin{definition}\label{def:majorizer}
A function $\Bar{\phi}_i: \mathcal{Z}_i \times \mathcal{Z} \rightarrow \mathbb{R}$ is said to be a tight majorizer of $\phi(z)$ if
\begin{enumerate}
    \item $\Bar{\phi}_i(z_i,z)=\phi(z), \forall z \in \mathcal{Z}$,
    \item $\Bar{\phi}_i(y_i,z) \geq \phi(z_1,...,z_{i-1},y_i,z_{i+1},...,z_s), \forall y_i \in \mathcal{Z}_i, \forall z \in \mathcal{Z}$,
\end{enumerate}
where $z_i$ denotes the $i$-th block of variables of $z$ with $1 \leq i \leq s$.
\end{definition}

At iteration $k+1$,~\cref{alg:meta-Algo} fixes the last values of block of variables $j \neq i$ and updates block $z_i$ as follows:
$$z_i^{(k+1)} = \argmin_{z_i \in \mathcal{Z}_i}\Bar{\phi}_i(z_i,z_1^{(k+1)},...,z_{i-1}^{(k+1)},z_{i}^{(k),\ast},...,z_{s}^{(k),\ast}) $$
where $z_{s}^{(k),\ast}$ denotes the $s$-th block of variable computed at iteration $k$ and optimally scaled using  Equations \eqref{eq:scale_aast_final}. After updating all the blocks as described above,~\cref{alg:meta-Algo} performs optimal scaling on each block. We first present~\cref{prop:nonIncreaseRB} which will be useful for the proof of~\cref{theo:monotoneNI}. 
\begin{proposition}\label{prop:nonIncreaseRB}
    Let $\{X_i\}_{i \leq n}$ be the current decomposition, the balanced columns $X_i[:,q]^\ast$ computed according to~\cref{alg:balancing-Algo} necessarily decrease the objective function value of Problem~\eqref{eq:GenOptiProblem}.
\end{proposition}

\begin{proof}
    The proof directly follows the results presented in~\cref{subsec:balancingPrinciple}, indeed we have:
    \begin{equation}
        \begin{aligned}
            \phi(\{X_i\}_{i \leq n}) & \underset{\forall i\leq n,\; \Lambda_i=I_{r\times r}}{=}f(\{X_i\Lambda_i\}_{i\leq n}) + \sum_{i=1}^{n} \mu_{i} \sum_{q=1}^{r} g_i(\Lambda_i[q,q]X_i[:,q]) \\
            & \geq f(\{X_i\Lambda^\ast_i\}_{i\leq n}) + \sum_{i=1}^{n} \mu_{i} \sum_{q=1}^{r} g_i(\Lambda^\ast_i[q,q]X_i[:,q]) = \phi(\{X_i^\ast\}_{i \leq n}).
        \end{aligned}
    \end{equation}
    where $\Lambda^\ast_i[q,q]=\left(\frac{\beta_q}{p_i\mu_ig_i(X_i[:,q])}\right)^{1/p_i}$ from Equations \eqref{eq:scale_aast_final} with $\beta_q$ as defined in~\cref{prop:opt_scal}. The second inequality holds by definition of $\Lambda_i^\ast[q,q]$, that is the optimal and unique solution of the scaling Problem \eqref{eq:scale_pb_big}.
\end{proof}

\begin{theorem}[also~\cref{theo:monotoneNI_main}]\label{theo:monotoneNI} Let $z^{(k)}=(\{X_i^{(k)}[:,q]\}_{i \leq n, q \leq r}) \geq 0$, and let $\Bar{\phi}_i$ be tight majorizers, as per~\cref{def:majorizer}, for $\phi: \mathcal{Z}\rightarrow \mathbb{R}_+$ from Problem~\eqref{eq:GenOptiProblem},  with $\mathcal{Z}:=\mathcal{Z}_1\times\ldots\mathcal{Z}_s \subseteq \text{dom} (\phi)$ and $1 \leq i \leq s=nr$. Then $\phi$ is monotone non-increasing under the updates of~\cref{alg:meta-Algo}. Moreover, the sequences $\{ \phi(z^{(k),\ast}) \}$ and $\{ \phi(z^{(k)}) \}$ converge.
\end{theorem}

\begin{proof}
    From~\cref{def:majorizer}, we have:
    \begin{equation*}
        \begin{aligned}
            \phi(z^{(k),\ast}) & = \Bar{\phi}_1(z_1^{(k),\ast},z_1^{(k),\ast},...,z_{s}^{(k),\ast})  \geq  \Bar{\phi}_1(z_1^{(k+1)},z_1^{(k),\ast},...,z_{s}^{(k),\ast}) \\
            & \geq \phi(z_1^{(k+1)},z_2^{(k),\ast},...,z_{s}^{(k),\ast}) = \Bar{\phi}_2(z_2^{(k),\ast},z_1^{(k+1)},z_2^{(k),\ast},...,z_{s}^{(k),\ast})\\
            & \geq  \Bar{\phi}_2(z_2^{(k+1)},z_1^{(k+1)},z_2^{(k),\ast},...,z_{s}^{(k),\ast}) \geq \phi(z_1^{(k+1)},z_2^{(k+1)},...,z_{s}^{(k),\ast}) \geq ... \\
            & \geq \phi(z_1^{(k+1)},z_2^{(k+1)},...,z_{s}^{(k+1)}) \underset{\cref{prop:nonIncreaseRB}}{\geq} \phi(z_1^{(k+1),\ast},z_2^{(k+1),\ast},...,z_{s}^{(k+1),\ast}) 
        \end{aligned}
    \end{equation*}
    where $z_i^{(k+1),\ast}$ are computed using Equations \eqref{eq:scale_aast_final} for $1 \leq i \leq s=nr$.
    
    By assumption, $\phi$ from Problem~\eqref{eq:GenOptiProblem} is a finite nonnegative sum of functions that are all bounded below onto the feasible set $\mathcal{Z}$, hence $\phi$ is also bounded below onto $\mathcal{Z}$. Therefore, both the non-increasing monotone sequences $\{ \phi(z^{(k),\ast}) \}$ and $\{ \phi(z^{(k)}) \}$ converge.
\end{proof}

\paragraph{Convergence of the sequence of iterates}

In the next Theorem, we establish the convergence of the iterates generated by~\cref{alg:meta-Algo} towards coordinate-wise minimum, or critical points of Problem~\eqref{eq:GenOptiProblem} under mild assumptions.
\begin{theorem}[also~\cref{theo:conv_iterates_main}]\label{theo:conv_iterates}
    Let majorizers $\Bar{\phi}_i(y_i,z)$, as defined in~\cref{def:majorizer}, be quasi-convex in $y_i$ for $i=1,...,s$. Additionally, assume that Assumption~\ref{assumptions: majorizer} are satisfied. Consider the subproblems in Equation~\eqref{eq:subproblems}, and assume that both these subproblems and the optimal scaling, obtained using Equations~\eqref{eq:scale_aast_final}, have unique solutions for any point $$(z_1^{(k+1)},...,z_{i-1}^{(k+1)},z_{i}^{(k),\ast},...,z_{s}^{(k),\ast}) \in \mathcal{Z}.$$
    Under these conditions, every limit point $z^{\infty}$ of the iterates $\{z^{(k)}\}_{k \in \mathbb{N}}$ generated by~\cref{alg:meta-Algo} is a coordinate-wise minimum of Problem~\eqref{eq:GenOptiProblem}. Furthermore, if $\phi(\cdot)$ is regular at $z^{\infty}$, then $z^{\infty}$ is a stationary point of Problem~\eqref{eq:GenOptiProblem}.

\end{theorem}
The proof is handled in~\cref{app:proof_iterates} below, and is an adaptation of the convergence to stationary points obtained in \cite[Theorem~2]{RHL_BSUM} for  block successive upperbound minimization (BSUM) framework.

\section{Convergence proof for the iterates of the Meta-Algorithm}\label{app:proof_iterates}
To better analyze the convergence of the iterates generated by~\cref{alg:meta-Algo}, it is important to review some key concepts presented in \cite{RHL_BSUM}:
\begin{itemize}
    \item \textit{Directional derivative:} For a function $\phi: \mathcal{Z} \rightarrow \mathbb{R}$, where $\mathcal{Z} \subseteq \mathbb{R}^m$ is a convex set, the directional derivative of $\phi$ at point $z$ along a vector $d$ is defined as:
     \begin{equation} \phi^{'}(z;d) := \lim_{\lambda \rightarrow 0} \frac{\phi(z + \lambda d) - \phi(z)}{\lambda} \end{equation}
    \item \textit{Stationary points of a function:} For a function $\phi: \mathcal{Z} \rightarrow \mathbb{R}$, where $\mathcal{Z} \subseteq \mathbb{R}^m$ is a convex set, the point $z^\star$ is considered a stationary point of $\phi(.)$ if $\phi^{'}(z^\star;d) \geq 0$ for all $d$ such that $z^\star + d \in \mathcal{Z}$.
    \item \textit{Quasi-convex function:} A function $\phi$ is said to be quasi-convex if it satisfies the following inequality:
\begin{equation} \phi(\theta x + (1-\theta) y) \leq \max ( \phi(x), \phi(y)) \quad \forall \theta \in (0,1), \quad \forall x,y \in \text{dom} \phi \end{equation}
    \item \textit{Regularity of a function at a point}: The function $\phi:\mathcal{Z} \rightarrow \mathbb{R}$ is regular at the point $z \in \text{dom} \phi$ \textit{w.r.t.} coordinates $m_1,...,m_s$ if $\phi^{'}(z;d) \geq 0$ for all $d=(d_1,...,d_s)$ with $\phi^{'}(z;d_i^0) \geq 0$, where $d_i^0:=(0,...,d_i,...,0)$ and $d_i \in \mathbb{R}^{m_i}$ for all $i$.
    \item \textit{Coordinate-wise minimum of a function} $z \in \text{dom} \phi \in \mathbb{R}^m$ is the coordinate-wise minimum of $\phi$ w.r.t. the coordinates in $\mathbb{R}^{m_1},...,\mathbb{R}^{m_s}$ if
    $$\phi(z + d^0_i) \geq \phi(z) \quad \forall d_i \in \mathbb{R}^{m_i} \quad \text{with } z+d^0_i \in \text{dom} \phi \quad \forall i$$
    or equivalently:
    $$\phi^{'}(z;d^0_i) \geq 0 \quad \forall d_i \in \mathbb{R}^{m_i} \quad \text{with } z+d^0_i \in \text{dom} \phi \quad \forall i$$
    \item \textit{Block Successive MM:}~\cref{alg:meta-Algo} updates only one block of variables at a time. Specifically, in iteration $k+1$, the selected block (referred to as block $i$) is updated by solving the subproblem:
\begin{equation}\label{eq:subproblems} \begin{aligned} & z_i^{(k+1)} = \argmin_{z_i \in \mathcal{Z}_i}\Bar{\phi}_i(z_i,z_1^{(k+1)},...,z_{i-1}^{(k+1)},z_{i}^{(k),\ast},...,z_{s}^{(k),\ast})  \end{aligned} \end{equation}

where $\Bar{\phi}_i(.,z)$ is a tight majorizer of $\phi(z)$ as defined in Definition \cref{def:majorizer}.
   
\end{itemize}

To discuss the convergence of~\cref{alg:meta-Algo}, additional assumptions about the majorizers $\Bar{\phi}_i$ are necessary. These assumptions are as follows:

\begin{assumption}\label{assumptions: majorizer} 
\hfill
\begin{enumerate} \item $\Bar{\phi}_i^{'}(y_i,z;d_i)\rvert_{y_i=z_i}=\phi^{'}(z;d_i^0) \quad \forall d_i^0=(0,...,d_i,...,0)$ such that $z_i + d_i \in \mathcal{Z}_i$, $\forall i$. This assumption implies that the directional derivatives of the majorizers and the objective function coincide at the current iterate for each block of variables. \item $\Bar{\phi}_i(y_i,z)$ is continuous in $(y_i,z) \quad \forall i$. \end{enumerate}
\end{assumption}

Following the approach outlined in~\cite{RHL_BSUM}, the convergence outcomes of~\cref{alg:meta-Algo} are applicable in the asymptotic regime as $k \rightarrow \infty$. This convergence relies on the aforementioned assumptions, which include the directional differentiability and the continuity of the function $\phi$. In addition to this, the proof leverages two critical assumptions: firstly, the quasi-convexity of the majorizers $\Bar{\phi}_i(y_i,z)$, and secondly, the uniqueness of their minimizers.
Here-under, we give the proof of~\cref{theo:conv_iterates}.

\begin{proof}
    We follow the main of steps of the general proof of BSUM~\cite{RHL_BSUM}, and slightly adapt them to the specificity of our Algorithm. First, by~\cref{theo:monotoneNI}, we have:
    \begin{equation}\label{eq:NonIncreasingSeq}
        \phi(z^{(0)}) \geq \phi(z^{(1)}) \geq  \phi(z^{(1),\ast}) \geq \phi(z^{(2)}) \geq  \phi(z^{(2),\ast})\geq \cdots
    \end{equation}
    Consider a limit point $z^{\infty}$, by combining Equation \eqref{eq:NonIncreasingSeq} and the continuity of $\phi$, we have:
    \begin{equation}\label{eq:limitPointNNincSeq}
        \lim_{k \rightarrow \infty} \phi(z^{(k)}) = \phi(z^{\infty}) = \lim_{k \rightarrow \infty} \phi(z^{(k),\ast})
    \end{equation}

    In the ensuing discussion, we focus on the subsequence denoted as $\{ z^{(k_j),\ast}\}$, which converges to the limit point $z^{\infty}$. Moreover, given the finite number of variable blocks, there must exist at least one block that undergoes infinite updates within the subsequences $\{k_j\}$. Without loss of generality, we assume block $s$ is updated infinitely often. Consequently, by constraining our attention to these subsequences, we can express this as:

    \begin{equation}
        z_s^{(k_j +1)} = \argmin_{z_s \in \mathcal{Z}_s}\Bar{\phi}_s(z_s,z_1^{(k_j +1)},z_{2}^{(k_j+1)},...,z_{s}^{(k_j),\ast})
    \end{equation}

    To establish the convergence $z^{(k_j + 1)} \rightarrow z^{\infty}$ as $j \rightarrow \infty$, or equivalently, to demonstrate $z_i^{(k_j + 1)} \rightarrow z_i^{\infty}$ for all $i$ as $j \rightarrow \infty$, we consider the block $i=1$. Assume the contrary, supposing $z_1^{(k_j + 1)}$ does not converge to $z_1^{\infty}$. Consequently, by narrowing our focus to a subsequence, there exists $\Bar{\gamma} > 0$ such that:
    \[
    \Bar{\gamma} \leq \gamma^{(k_j)} = \|z_1^{(k_j +1)} - z_1^{(k_j),\ast} \| \quad \forall k_j
    \]

    This implies that as $j \rightarrow \infty$, $z_1^{(k_j +1)}$ consistently deviates from $z_1^{(k_j),\ast}$, which converges to $z_1^{\infty}$. Consequently, by transitivity, $z_1^{(k_j +1)}$ does not converge to $z_1^{\infty}$.

    We introduce the sequence $s^{(k_j)}$, defined as the normalized difference between $z_1^{(k_j +1)}$ and $z_1^{(k_j),\ast}$:

   \begin{equation}\label{eq:Skj_def}
         s^{(k_j)}:=\frac{z_1^{(k_j +1)} - z_1^{(k_j),\ast}}{\gamma^{(k_j)}}
    \end{equation}

    Notably, we observe that $\| s^{(k_j)} \| = 1$ holds for all $\{k_j\}$; thus, $s^{(k_j)}$ is a member of a compact set and consequently possesses a limit point denoted as $\bar{s}$. By further narrowing our focus to a subsequence that converges to $\bar{s}$, we can express the ensuing set of inequalities:

\begin{equation}\label{eq:Ineqseq1}
    \begin{aligned}
        \phi(z^{(k_{j+1}),\ast}) &\leq \phi(z^{(k_{j+1})}) \leq \phi(z^{(k_{j}+1),\ast}) \leq \phi(z^{(k_{j}+1)})
    \end{aligned}
\end{equation}
    which hold thanks to~\cref{eq:NonIncreasingSeq} (consequence of~\cref{theo:monotoneNI}) and given that $k_{j+1} \geq k_{j}+1$. Moreover :
    \begin{equation}\label{eq:Ineqseq2}
        \begin{aligned}
            \phi(z^{(k_{j}+1)}) & \leq \bar{\phi}_1(z_1^{(k_{j}+1)},z^{(k_{j}),\ast}) \quad \text{(\cref{def:majorizer})}\\
                           & = \bar{\phi}_1( z_1^{(k_j),\ast} + \gamma^{(k_j)} s^{(k_j)}      ,z^{(k_{j}),\ast}) \quad \text{(Using \eqref{eq:Skj_def})} \\
                           & \leq \bar{\phi}_1( z_1^{(k_j),\ast} + \epsilon \Bar{\gamma} s^{(k_j)}      ,z^{(k_{j}),\ast}) \quad \text{($\epsilon \in [0,1]$ and $z_1^{(k_j + 1)}= \argmin \bar{\phi}_1(.,z^{(k_{j}),\ast})$)} \\
                           & := \bar{\phi}_1( z_1^{(k_j),\ast} + \theta (z_1^{(k_j +1)} - z_1^{(k_j),\ast})      ,z^{(k_{j}),\ast}) \quad \text{(With $\theta=\epsilon \frac{\Bar{\gamma}}{\gamma^{(k_j)}} \in [0,1]$)} \\
                           & \leq \max \left(\bar{\phi}_1( z_1^{(k_j +1)} ,z^{(k_{j}),\ast}) , \bar{\phi}_1( z^{(k_{j}),\ast},z^{(k_{j}),\ast}) \right) \quad \text{(quasi-convexity of $ \bar{\phi}_1(.,z^{(k_{j}),\ast})$)} \\
                           & = \bar{\phi}_1( z^{(k_{j}),\ast},z^{(k_{j}),\ast}) \quad \text{(by definition of the arg min $z_1^{(k_j +1)}$)} \\
                           & = \phi(z^{(k_{j}),\ast}) \quad \text{(Majorizer is tight at $z^{(k_{j}),\ast}$ )}
        \end{aligned}
    \end{equation}
    Letting $j \rightarrow \infty$ and combining Equations \eqref{eq:Ineqseq1} and \eqref{eq:Ineqseq2}, we have:
    \begin{equation}
        \phi(z^{\infty}) \leq \Bar{\phi}_1(z_1^{\infty} + \epsilon \Bar{\gamma} \Bar{s}) \leq \phi(z^{\infty})
    \end{equation}
    or equivalently:
    \begin{equation}\label{eq:setLimitPoint}
        \phi(z^{\infty}) = \Bar{\phi}_1(z_1^{\infty} + \epsilon \Bar{\gamma} \Bar{s})
    \end{equation}
    Furthermore:
    \begin{equation}
        \begin{aligned}
            \bar{\phi}_1(z_1^{(k_{j+1}),\ast},z^{(k_{j+1}),\ast}) = \phi(z^{(k_{j+1}),\ast})  \leq \phi(z^{(k_{j}+1),\ast}) & \leq \phi(z^{(k_{j}+1)}) \\
            & \leq \bar{\phi}_1(z_1^{(k_{j}+1)},z^{(k_{j}),\ast}) \\
            & \leq \bar{\phi}_1(z_1,z^{(k_{j}),\ast}) \quad \forall z_1 \in \mathcal{Z}_1
        \end{aligned}
    \end{equation}
    Thanks to the continuity of majorizers, and letting $j \rightarrow \infty$, we obtain:
    $$\bar{\phi}_1(z_1^{\infty},z^{\infty}) \leq \bar{\phi}_1(z_1,z^{\infty}) \quad \forall z_1 \in \mathcal{Z}_1 $$
    which implies that $z_1^{\infty}$ is the minimizer of $\bar{\phi}_1(.,z^{\infty})$. By hypothesis, we assumed that the minimizer of each majorizer is unique, which is a contradiction with Equation \eqref{eq:setLimitPoint}. Therefore, the contrary assumption is not true, i.e., $z^{(k_j + 1)} \rightarrow z^{\infty}$ for $j \rightarrow \infty$.

    Since $z_1^{(k_j + 1 )} = \argmin_{z_1 \in \mathcal{Z}_1 } \Bar{\phi}_1 (z_1, z^{(k_{j}),\ast})$, then we have:
    $$\Bar{\phi}_1 (z_1^{(k_j + 1 )}, z^{(k_{j}),\ast}) \leq \Bar{\phi}_1 (z_1, z^{(k_{j}),\ast}) \quad \forall z_1 \in \mathcal{Z}_1$$

    The computation of the limit $j \rightarrow \infty$ gives:
    $$\Bar{\phi}_1 (z_1^{\infty}, z^{\infty}) \leq \Bar{\phi}_1 (z_1, z^{\infty}) \quad \forall z_1 \in \mathcal{Z}_1 $$

    The rest of the proof is identical to \cite{RHL_BSUM}, we detail the remaining steps for the sake of completeness.
    The above inequality implies that all the directional derivatives of $\Bar{\phi}_1(.,z^{\infty})$ w.r.t. block 1 of variables, and evaluated at $z_1^{\infty}$ are nonnegative, that is:
    $$\Bar{\phi}^{'}(z_1,z^{\infty};d_1)\bigg\rvert_{z_1 = z_1^{\infty}} \geq 0 \quad \forall d_1 \in \mathbb{R}^{m_1} \text{ such that } z_1^{\infty} + d_1 \in \mathcal{Z}_1$$

    Following the same rationale for the other blocks of variables, we obtain:
    \begin{equation}\label{eq:Majora_dirDerivatives}
        \Bar{\phi}^{'}(z_i,z^{\infty};d_i)\bigg\rvert_{z_i = z_i^{\infty}} \geq 0 \quad \forall d_i \in \mathbb{R}^{m_i} \text{ such that } z_i^{\infty} + d_i \in \mathcal{Z}_i, \forall i=1,...,s
    \end{equation}
    Using Assumption~\ref{assumptions: majorizer} and Equations \eqref{eq:Majora_dirDerivatives}, we obtain:
    \begin{equation}\label{eq:Func_dirDerivatives}
        \phi^{'}(z^{\infty};d_i^0) \geq 0 \quad \forall d_i^0=(0,...,d_i,...,0) \text{ such that } z_i^{\infty} + d_i \in \mathcal{Z}_i, \forall i=1,...,s
    \end{equation}
    Equations \eqref{eq:Func_dirDerivatives} implies that $z^{\infty}$ is a coordinate-wise minimum of $\phi$.
    Finally, assuming that $\phi$ is regular on its domain, including $z^{\infty}$, then we can write:
    \begin{equation}\label{eq:Func_dirDerivatives_reg}
        \phi^{'}(z^{\infty};d) \geq 0 \quad \forall d=(d_1,...,d_s) \text{ such that } z^{\infty} + d \in \mathcal{Z}
    \end{equation}
    Under this additional regularity assumption of $\phi$, Equation \eqref{eq:Func_dirDerivatives_reg} implies that the limit point $z^{\infty}$ is a stationary point of $\phi$.
    This concludes the proof.
    \end{proof}

\section{MM update rules for the meta-algorithm}\label{app:update_rules}

We consider a generic minimization problem of the form
\begin{equation}\label{eq:sNMFcolpb}
  \underset{X_1 \geq 0}{\argmin}D_{\beta}(T|X_1U) + \mu \sum_{q \leq r}\|X_1[:,q]\|_p^p
\end{equation}
where $T=\mathcal{T}_{[1]} \in \mathbb{R}_+^{m_1 \times m_2 \times m_3 }$, $U=\mathcal{G}_{[1]}(X_2 \kron X_3)^T \in \mathbb{R}_+^{r_1 \times m_2 m_3}$, $D_\beta$ is the beta-divergence and $\mu > 0$. Notably, we considered the decomposition of a third-order tensor $\mathcal{T}$, but the approach can be generalized to any order.

In Problem \eqref{eq:sNMFcolpb}, the regularization term is separable w.r.t. each entry of $X_1[:,q]$ for all $q \leq r$.
For the data fitting term, we use the majorizer proposed in \cite{fevotte2011algorithms}. For the sake of completeness and to ease the understanding of the further developments, we briefly recall it in the following. It consists in first decomposing the $\beta$-divergence between two nonnegative scalars $a$ and $b$ as a sum of a convex, concave and constant terms as follows:

\begin{equation}\label{eq:sNMFJensen}
  d_{\beta}(a|b)= \check{d}_{\beta}(a|b)+\hat{d}_{\beta}(a|b)+\bar{d}_{\beta}(a|b), 
  \end{equation}
where $\check{d}$ is a convex function of $b$, $\hat{d}$ is a concave function of $b$ and $\bar{d}$ is a constant w.r.t. $b$; see~\cref{table:conv_concav_decomp}, and then  
in majorizing the convex part of the $\beta$-divergence using Jensen's inequality and majorizing the concave part by its tangent (first-order Taylor approximation).
  
\begin{center}
\begin{table}[h!]
\begin{center}
\caption{Differentiable convex-concave-constant decomposition of the $\beta$-divergence under the form \eqref{eq:sNMFJensen}~\cite{fevotte2011algorithms}. }
\label{table:conv_concav_decomp}
\begin{tabular}{|c|c|c|c|}
\hline 
      & $\check{d}(a|b)$  &  $\hat{d}(a|b)$ & $\bar{d}(a)$   \\  \hline 
      $\beta=(-\infty,1) \setminus \{0\}$     
 &   $\frac{1}{1 - \beta} a b^{\beta-1}$  & $\frac{1}{\beta} b^{\beta}$  & $\frac{1}{\beta (\beta-1)} a^\beta$  \\
 $\beta=0$      
 & $ab^{-1}$     & $\log(b)$  & $a(\log(a)-1)$  \\
 $\beta \in [1,2] $ 
 & $d_{\beta}(a|b)$ & 0 & 0    \\ \hline  
\end{tabular} 
\end{center}
\end{table}
\end{center}

Since $\mu \sum_{q \leq r}\|X_1[:,q]\|_p^p=\mu \|X_1\|_p^p$ is convex, the subproblem obtained after majorization (step 1 of MM, see~\cref{sec:MM_framework}) to solve has the form:
\begin{equation}\label{eq:sNMFMaj}
  \underset{X_1 \geq 0}{\argmin}G(X_1|\tilde{X}_1) + \mu \|X_1\|_p^p
\end{equation}

where $G(X_1 | \tilde{X_1})$ is the separable, tight majorizer for $D_\beta(T|X_1U)$ at $\tilde{X_1} \geq 0$ of the following form

\begin{equation}\label{eq:20}
\begin{aligned}
     G(X_1 | \tilde{X_1}) & = \sum_{i \leq m_1, k\leq r} l(X_1[i,k] | \tilde{X_1}) \\
     \text{such that }   l(X_1[i,k] | \tilde{X_1}) & =\sum_{j \leq m_2 m_3} \frac{\Tilde{X}_1[i,k] U[k,j]}{\Tilde{T}[i,j] } \check{d} ( T[i,j] | \Tilde{T}[i,j] \frac{X_1[i,k]}{\Tilde{X}_1[i,k]} ) \\
     & + \sum_{j \leq m_2 m_3} \Big[ U[k,j] \hat{d}^{'}(T[i,j]|\tilde{T}[i,j]) \Big] X_1[i,k] + \text{cst}
\end{aligned}
\end{equation}

according to \cite[Lemma~1]{fevotte2011algorithms} where $\Tilde{T}=\Tilde{X}_1 U$.

As explained in~\cref{sec:MM_framework}, the optimization problem \eqref{eq:sNMFMaj} is a particular instance of the family of problems covered by the framework proposed in \cite{doi:10.1137/20M1377278} in the case no additional equality constraints are required. Assuming $\mu> 0$ and the regularization terms satisfy Assumption~\ref{ass1}, the minimizer of Problem \eqref{eq:sNMFMaj} exists and is unique in $(0,\infty)^{m_1 \times r_1}$ as soon as $\beta < 2$.

Moreover, according to \cite{doi:10.1137/20M1377278}, it is possible to derive closed-form expressions for the minimizer of Problem \eqref{eq:sNMFMaj} for the following values of $\beta$: $\beta \in \{0,1,3/2\}$ (if constants $C>0$ in Assumption~\ref{ass1}) or $\beta \in (-\infty, 1 ] \cup\{5/4,4/3,3/2\}$ (if constants $C=0$ in Assumption~\ref{ass1} when $p=1$ for instance).
Outside these values for $\beta$, an iterative scheme is required to numerically compute the minimizer. We can cite the Newton's method, the Muller's method and the the procedure developed in \cite{rootpoly} which is based on the explicit calculation of the intermediary root of a canonical form of cubic. This procedure is suited for providing highly accurate numerical results in the case of badly conditioned polynomials.

In this section, we will consider the interval $\beta \in [1,2)$. For this setting, we can show that computing the minimizer of Problem \eqref{eq:sNMFMaj} corresponds to solving, for a single element $X_1[i,k]$:
\begin{equation}\label{eq:sNMFpolybetah}
    \begin{aligned}
       & \nabla_{X_1[i,k]}\left[ G(X_1|\tilde{X}_1) + \mu \|X_1\|_p^p \right] = 0\\
        & \Leftrightarrow a X_1[i,k]^{p-1+2-\beta} + b X_1[i,k] - c  = 0 \text{ since } X_1[i,k] \in (0,\infty).
    \end{aligned}
\end{equation}
where the second equation has been obtained by multiplying both sides by $X_1[i,k]^{\beta-2}$, and:
\begin{itemize}
    \item $a=p\mu > 0 $ by hypothesis,
    \item $b=\sum_j U[k,j] \left( \frac{\tilde{T}[i,j]}{\tilde{X}_1[i,k]} \right)^{\beta-1} \geq 0$
    \item $c=\sum_j U[k,j] T[i,j] \left( \frac{\tilde{T}[i,j]}{\tilde{X}_1[i,k]} \right)^{\beta-2} \geq 0$.
\end{itemize}

At this stage, one may observe that finding the solution $X_1$ boils down to solving several polynomial equations in parallel. Nevertheless, depending on the values of $\beta$ and $p$, we can further simplify the update.

\paragraph{Case $\beta$=1, p=1}
Equation \eqref{eq:sNMFpolybetah} becomes $a X_1[i,k] + b X_1[i,k] - c  = 0$. The positive (real) root is computed as follows:
\begin{equation}\label{eq:sNMFupdatehkentry}
    \begin{aligned}
       \hat{X}[i,k] & =\frac{\left.c\right|_{\beta=1}}{a+\left.b\right|_{\beta=1}} \\
                   & = \tilde{X}_1[i,k]\frac{\sum_j U[k,j]   \frac{T[i,j]}{\tilde{T}[i,j]} }{\mu + \sum_j U[k,j]}
    \end{aligned}
\end{equation}
which is nonnegative.
In matrix format,
\begin{equation}\label{eq:sNMFupdatehkentryvec}
       \hat{X}_1 = \tilde{X}_1 \odot \frac{\Big[\frac{[T]}{[\Tilde{X}_1 U ]} U^\top \Big]}{\Big[\mu e_{r1 \times m_1} + e_{m_1 \times m_2 m_3} U^\top \Big]}
\end{equation}
where $\odot$ denotes the Hadamard product, $[A]/[B]$ denotes the elementwise division between matrices $A$ and $B$.
Note that in~\cite{Hoyer2002Non}, the regularization parameters $\mu$ is mistakenly written in the numerator. This error has been corrected in subsequent works without an explicit mention of the error~\cite{leroux2015sparse}.

\paragraph{Case $\beta$=1, p=2}
Equation \eqref{eq:sNMFpolybetah} becomes $a X_1[i,k]^{2} + b X_1[i,k] - c  = 0$. The positive (real) root is computed as follows:
\begin{equation}\label{eq:updatewkentry}
    \begin{aligned}
       \hat{X_1}[i,k]=\frac{\sqrt{\left(\sum_j U[k,j]\right)^2+8\mu\tilde{X}_1[i,k]\sum_j U[k,j]\frac{T[i,j]}{\tilde{T}[i,j]}}-\sum_j U[k,j]}{4\mu}
    \end{aligned}
\end{equation}
with $\mu > 0$.
Note that although the closed-form expression in Equation \eqref{eq:updatewkentry} has a negative term in the numerator of the right-hand side, it can be easily checked that it always remains nonnegative given $T[i,j]$, $U[k,j]$ and $\tilde{X}[i,k]$ are nonnegative for all $i,j,k$. 

Equation \eqref{eq:updatewkentry} can be expressed in matrix form as follows:

\begin{equation}\label{eq:updateW}
    \begin{aligned}
       \hat{X}_1=\frac{\left[C^{.2}+S\right]^{.\frac{1}{2}}-C}{4 \mu}
    \end{aligned}
\end{equation}

where $[A]^{.\alpha}$ denotes the elementwise $\alpha$-exponent, $C=e U^\top$ with $e$ is a all-one matrix of size $m_1 \times m_2 m_3$ and $S=8\mu \tilde{X}_1 \odot \left( \frac{\left[ T \right]}{\left[ \tilde{X}_1U \right]} U^T\right) $.

Some interesting insights are discussed here-under:
\begin{itemize}
    \item Computing the limit $\lim_{\mu \to 0} \hat{X}_1(\mu)$ makes Equation \eqref{eq:updateW} tends to the original Multiplicative Updates introduced by Lee and Seung \cite{Lee1999Learning}. Indeed let us compute this limit in the scalar case from Equation \eqref{eq:updatewkentry} and pose $\alpha=\sum_n u[j,k]$ and $\eta=\tilde{w}[j]\sum_j U[k,j]\frac{T[i,j]}{\tilde{T}[i,j]}$ for convenience:
    \begin{equation}
        \begin{aligned}
           \lim_{\mu \to 0} \frac{\sqrt{\alpha^2+8\mu \eta}-\alpha}{4 \mu} & = \frac{"0"}{"0"} \\
           & \underset{\text{H}}{=}\lim_{\mu \to 0} \frac{\frac{\partial}{\partial_{\mu}} \sqrt{\alpha^2+8\mu \eta}-\alpha}{\frac{\partial}{\partial_{\mu}}4 \mu}\\
           & = \lim_{\mu \to 0} (\alpha^2+8\mu\eta)^{-\frac{1}{2}} \eta = \frac{\eta}{\alpha}=\tilde{X}_1[i,k]  \frac{\sum_j U[k,j]\frac{T[i,j]}{\tilde{T}[i,j]}}{\sum_j U[k,j]}
        \end{aligned}
    \end{equation}
    In matrix form we have:
    \begin{equation}
        \begin{aligned}
           \lim_{\mu \to 0} \hat{X}_1(\mu) = \tilde{X}_1 \odot \frac{\left[ \frac{\left[ T \right]}{\left[ \tilde{X}_1U \right]} U^T \right]}{\left[ e U^T \right]}
        \end{aligned}
    \end{equation}
    which are the MU proposed by Lee and Seung for $\beta=1$.
    \item Regarding the analysis of univariate polynomial equation given in \eqref{eq:sNMFpolybetah}, interestingly the existence of a unique positive real root could have been also established by using Descartes rules of sign. Indeed we can count the number of real positive roots that $p(X_1[i,k])=a X_1[i,k]^{p-1+2-\beta}+ b X_1[i,k] - c$ has (for $\beta \in [1,2] $). More specifically, let $v$ be the number of variations in the sign of the coefficients $a,b,c$ (ignoring coefficients that are zero). Let $n_p$ be the number of real positive roots. Then:
    \begin{enumerate}
        \item $n_p \leq v$,
        \item $v-n_p$ is an even integer.
    \end{enumerate}
     Let us consider the case $\beta=1$ and $p=2$ as an example: given the polynomial $p(X_1[i,k])=a X_1[i,k]^{2} + b X_1[i,k] - c$, assuming that $c$ is positive. Then $v=1$, so $n_p$ is either 0 or 1 by rule 1. But by rule 2, $v-n_p$ must be even, hence $n_p=1$. Similar conclusion can be made for any $\beta \in [1,2] $.

\end{itemize}

Finally, similar rationale can be followed to derive MM updates in closed-form for $\beta=3/2$. The case $\beta=2$ is less interesting to tackle with MM updates rules since each subproblem is $L$-smooth, allowing the use of highly efficient methods such as the Projected Accelerated Gradient Descent algorithm from \cite{nesterov2013introductory}, or the famous HALS algorithm. 
The derivation of the MM updates in the case $\beta = 0$ with $\ell_2^2$ regularizations on the factors is below in~\cref{app:beta_zero_l2reg}.

\section{Factor updates for $\beta=0$ with $\ell_2^2$ regularizations}\label{app:beta_zero_l2reg}
In this section, we present the steps followed to derive MM update when the data fitting term corresponds to the well-known "Itakura-Saito" divergence, that is for $\beta=0$, and with $\ell_2^2$ regularization on  $\{X_i[:,q]\}_{i \leq n}$ for all $q \leq r$. Since the update of all factors can be derived in a similar fashion, we only detail the update for $X_1 \in \mathbb{R}_+^{m_1 \times r_1}$, that is for $i=1$, without loss of generality. The MM update for $X_1$ boils down to tackle the following subproblem:
\begin{equation}\label{eq:probinWIS}
  \underset{X_1 \geq 0}{\argmin}D_{0}(T|X_1U) +  \mu \sum_{q \leq r} \|X_1[:,q]\|_2^2
\end{equation}
where $T=\mathcal{T}_{[1]} \in \mathbb{R}_+^{m_1 \times m_2 \times m_3 }$, $U=\mathcal{G}_{[1]}(X_2 \kron X_3)^T \in \mathbb{R}_+^{r_1 \times m_2 m_3}$ and $\mu > 0$.
We follow the methodology detailed in~\cref{sec:MM_framework}. In Problem \eqref{eq:probinWIS}, the regularization terms are separable \textit{w.r.t.} each entry of $X_1[:,q]$ for all $q \leq r$. For the data fitting term, we use the results from \cite{fevotte2011algorithms}. Since $\sum_{q \leq r} \|X_1[:,q]\|_2^2= \|X_1\|_F^2$, this leads to solve the following problem:
\begin{equation}\label{eq:prob_inw_IS}
    \argmin_{X_1 \geq 0} G(X_1 | \tilde{X_1}) + \mu \|X_1\|_F^2
\end{equation}

where $G(X_1 | \tilde{X_1})$ is the separable, tight majorizer for $D_{0}(T|X_1U)$ at $\tilde{X_1}$ as proposed in \cite[Lemma~1]{fevotte2011algorithms}.

Given that the objective function in Problem~\eqref{eq:prob_inw_IS} is separable with respect to each entry $X_1[i,k]$ of $X_1$ where $1 \leq i \leq m_1$ and $1 \leq k \leq r_1$, we can update each entry separately. To find the value of $X_1[i,k]$, we solve 
\begin{equation}\label{eq:subProbem_ISL2_sep}
    \nabla_{X_1[i,k]}\left[ G(X_1 | \tilde{X_1}) + \mu ||X_1||_F^2 \right] = 0 
\end{equation}

and check if the solution is unique and satisfies the nonnegativity constraints. Considering the discussion in~\cref{sec:MM_framework} and the findings of~\cite[Proposition1]{doi:10.1137/20M1377278}, we are aware that this minimizer exists and is unique over the range of $(0,+\infty)$ due to the condition of $\mu>0$ and the regularization function squared $\ell_2$ norm satisfying Assumption~\ref{ass1} with constants $C=2$. For now, we put this result aside and simply compute the solution of Equation~\eqref{eq:subProbem_ISL2_sep} and then check if it exists and is unique over the positive real line. 

From Lemma~\cite[Lemma~1]{fevotte2011algorithms} and~\cref{table:conv_concav_decomp} for $\beta=0$, one can show that solving \eqref{eq:subProbem_ISL2_sep} reduces to finding $X_1[i,k] > 0$ of the following scalar equation:
\begin{equation}\label{eq:poly_ISl2s}
    - \Bar{a} (X_1[i,k])^{-2} + 2\mu X_1[i,k] + \Bar{c} = 0
\end{equation}
where $\Bar{a} = (\Tilde{X}_1[i,k])^{2} \sum_{j} U[k,j] \frac{T[i,j]}{(\Tilde{T}[i,j])^{2}}, \geq 0$, $\Tilde{T}=\Tilde{X}_1 U$ and $\Bar{c}=\sum_j \frac{U[k,j]}{T[i,j]}$ \footnote{The solution $X_1[i,k]=0$ only exists if $\Bar{c}=0$, that is, by definition of $\Bar{c}$, when the column $U[k,:]$ is zero.}. Posing $x=X_1[i,k]$, Equation \eqref{eq:poly_ISl2s} can be equivalently rewritten as follows:
\begin{equation} 
\label{eq:poly_ISl2s_2} x^3 + p x^2 + q x + r = 0 ,
\end{equation}
where $p=\frac{\Bar{c}}{2\mu} (\geq 0)$, $q=0$ and $r=-\frac{\Bar{a}}{2\mu} (\leq 0)$. Equation \eqref{eq:poly_ISl2s_2} can be rewritten in the so-called normal form by posing $x = z - p/3$ as follows:
\begin{equation} 
\label{beta1over2_poly_2} z^3 + a z  + b = 0 ,
\end{equation}
where $a=\frac{1}{3}(3q - p^2)=\frac{-p^2}{3}$ and $b=\frac{1}{27}(2p^3 - 9pq + 27r)=\frac{1}{27}(2p^3 + 27r)$. Under the condition $\frac{b^2}{4} + \frac{a^3}{27} > 0$, Equation \eqref{beta1over2_poly_2} has one real root and two conjugate imaginary roots. This condition boils down to $\frac{1}{4}(2p^3 + 27r)^2 - p^6 > 0$. By definition of $p$ (non-negative) and $r$ (non-positive) above, this condition holds if $p > 0$ and $r<-\frac{4p^3}{27}$. The first condition on $p$ is satisfied assuming that $\Bar{c}$ is positive. The second condition on $r$ is satisfied by further assuming $\mu > \sqrt{\frac{1}{27} \frac{(\Bar{c})^3}{\Bar{a}}}$
\footnote{In practice, this condition means that $\mu$ should not be chosen too small.}. Therefore, the positive real root has the following closed-form expression:
\begin{equation} 
\label{sol} z = \sqrt[3]{-\frac{b}{2} + \sqrt{\frac{b^2}{4} + \frac{a^3}{27}}} +  \sqrt[3]{-\frac{b}{2} - \sqrt{\frac{b^2}{4} + \frac{a^3}{27}}},
\end{equation}
and hence 
\begin{equation} 
\label{sol_2} X_1[i,k] =z - \frac{\Bar{c}}{6 \mu},
\end{equation} 
With a bit of technicalities, one can show this root is positive, under the assumptions $\Bar{c} > 0$ and $\mu > \sqrt{\frac{1}{27} \frac{(\Bar{c})^3}{\Bar{a}}}$.
In matrix form, we obtain the following update 
$$
X_1 =  Z - \frac{[\Bar{C}]}{[6 \mu ]}, 
$$
where $Z=\Big[ \frac{-B}{2} + \Big[ \frac{[B]^{.2}}{4} + \frac{[A]^{.3}}{27} \Big]^{.(1/2)}\Big]^{.(1/3)} + \Big[ \frac{-B}{2} - \Big[ \frac{[B]^{.2}}{4} + \frac{[A]^{.3}}{27} \Big]^{.(1/2)}\Big]^{.(1/3)}$, $A=-\frac{1}{3} \Big[ \frac{[\Bar{C}]}{[2\mu]}\Big]^{.2}$, $B=\frac{1}{27} \Big( 2 \Big[  \frac{[\Bar{C}]}{[2\mu]}\Big]^{.3} + 27 \frac{[-\Bar{A}]}{[2\mu]}\Big)$, $\Bar{A}= [\tilde X_1]^{.2} \odot 
\left(\frac{[T]}{[\tilde X_1 U]^{.2}} U^\top\right)$ and $\Bar{C}=\frac{[E]}{[\tilde X_1 U]} U^\top$, with $E$ denoting a matrix full of ones of size $m_1 \times m_2 m_3$. \\

Note that we skipped in the discussion above the following cases:
\begin{enumerate}
    \item $\frac{b^2}{4} + \frac{a^3}{27} = 0$, that is when $r=-\frac{4p^3}{27}$, or equivalently $\mu = \sqrt{\frac{1}{27} \frac{(\Bar{c})^3}{\Bar{a}}}$: here, Equation \eqref{beta1over2_poly_2} has one nonnegative real root given by $z= \sqrt{- \frac{a}{3}} = \frac{1}{3} \frac{\Bar{c}}{\mu}$ (if $b<0$) and hence $X_1[i,k]=z - \frac{\Bar{c}}{6 \mu} = \frac{\Bar{c}}{6 \mu} \geq 0$. If $b>0$, the root is given by $z=\frac{\Bar{c}}{6 \mu}$, and then $X_1[i,k]=0$.
    \item $\frac{b^2}{4} + \frac{a^3}{27} < 0$, that is when $\mu < \sqrt{\frac{1}{27} \frac{(\Bar{c})^3}{\Bar{a}}}$: then Equation \eqref{beta1over2_poly_2} has one nonnegative real root given among $z=\frac{1}{3} \frac{\Bar{c}}{\mu} \cos{(\frac{\Phi}{3} + \frac{2 k \pi}{3})}$ with $k=0,1,2$ and where $\cos{\Phi}=\frac{1}{2} \frac{|2 p^3 + 27 r|}{p^3}$.
\end{enumerate}

Those cases are treated in the implementation of the algorithm.



\section{Scaling slices and the core with tensor Sinkhorn}\label{app:NTD_sinkhornbalancing}

At first glance the sNTD does not fit in the HRSI framework. While there is a form of scale ambiguity in NTD, it is not of the form we introduced in Assumption~\textbf{A1} in Section~\cref{sec:HRSI_intro}. The scaling ambiguity works pairwise between the core and each factor, for instance on the first mode,
\begin{equation}
  (\mathcal{G} \times_1 \Lambda) \times_1 (X_1\Lambda^{-1}) \times_2 X_2 \times_3 X_3 =  \mathcal{G} \times_1 X_1 \times_2 X_2 \times_3 X_3
\end{equation}
for a diagonal matrix $\Lambda$ with positive entries. This means that the core matrix has scaling ambiguities affecting its rows, columns, and fibers. It then makes little sense to talk about scaling the columns of $\mathcal{G}$ within the HRSI framework.

The fact that the scaling invariances do not fit the HRSI framework does not change the fact that the regularizations will enforce implicitly a particular choice of the parameters $\mathcal{G}$ and $X_1,X_2,X_3$ at optimality. Therefore it should still be valuable to design a balancing procedure to enhance convergence speed, as we have proposed for sNMF and rCPD. 

The scalar scale invariance detailed in~\cref{sec:Showcases} does not encompass all scaling invariances that occur in NTD. Again, in general, for each mode, the core and factor can be rescaled by inserting a diagonal matrix in-between, e.g. a matrix $\Lambda_{1}$ for the first mode:
$\mathcal{G} \times_1 X_1 = (\mathcal{G} \times_1 \Lambda_1) \times_1 (X_1\Lambda_1^{-1})$, where $\Lambda_1$ has strictly positive diagonal elements. The scaling on each mode affects all the slices of the core, therefore the HRSI framework does not directly apply. Nevertheless, it is possible to derive an algorithmic procedure similar to the Sinkhorn algorithm to scale the factors and the core. We provide detailed derivations for the special case of sNTD, but the proposed method extends to any regularized NTD with homogeneous regularization functions.

To understand the connection between the Sinkhorn algorithm, which aims at scaling a matrix or tensor so that the slices are on the unit simplex~\cite{Sinkhorn_Knopp_1967,BACHARACH,LAMOND1981239}, and the proposed procedure, consider the following: the results of~\cref{sec:Sparsity_scale_inva_models} applied to scale invariance on the first mode inform that at optimality, $\beta_{X_1}[q]:=2\|X_1^\ast[:,q]\|_2^2 = \|\mathcal{G}^\ast_{[1]}[:,q]\|_1$, where $\mathcal{G}_{[i]}$ is the mode $i$ unfolding of the core tensor $\mathcal{G}$. This observation implies that solutions of sparse NTD must in particular be solutions to a system of equations of the form:
\begin{align}\label{eq:scale_NTD_systemW}
\beta_{X_1}[q] = \|\mathcal{G}_{[1]}[:,q]\|_1  \\\label{eq:scale_NTD_systemH}
\beta_{X_2}[q] = \|\mathcal{G}_{[2]}[:,q]\|_1  \\\label{eq:scale_NTD_systemQ}
\beta_{X_3}[q] = \|\mathcal{G}_{[3]}[:,q]\|_1 
\end{align}
Optimally balancing sparse NTD can therefore be reduced to two operations: first, compute the optimal regularizations $\beta_{X_i}$, then transform the core and the factors so that\eqref{eq:scale_NTD_systemW}-\eqref{eq:scale_NTD_systemQ} holds. Supposing we know the $\beta_{X_i}$ values, computing the scalings of the slices of the core tensor in each mode is exactly what the tensor Sinkhorn algorithm is designed for~\cite{Csiszar_1975,pmlr-v70-sugiyama17a}. However, since the $\beta_{X_i}$ values are not known, we cannot directly apply a tensor Sinkhorn algorithm.
Tensor Sinkhorn usually involves an alternating normalization procedure, which we shall mimic in the following.

We propose a scaling heuristic which, starting from initial diagonal matrices $\Lambda_1,\Lambda_2,\Lambda_3$ set to identity matrices, alternatively balances each mode optimally:
\begin{align}\label{eq:opti_prob_scal_sNTD}
    \underset{\Lambda_1 \text{ diagonal}}{\argmin} \|\mathcal{G} \times_1 \Lambda_1 \times_2 \Lambda_2 \times_3 \Lambda_2 \|_1 + \|X_1 \Lambda_1^{-1} |_F^2  \\
    \underset{\Lambda_3 \text{ diagonal}}{\argmin} \|\mathcal{G} \times_1 \Lambda_1 \times_2 \Lambda_2 \times_3 \Lambda_3 \|_1 + \|X_2 \Lambda_2^{-1} \|_F^2 \\
    \underset{\Lambda_3 \text{ diagonal}}{\argmin} \|\mathcal{G} \times_1 \Lambda_1 \times_2 \Lambda_2 \times_3 \Lambda_3 \|_1 + \|X_3 \Lambda_3^{-1}\|_F^2 ~.
\end{align}
This procedure computes estimates of the $\beta_{X_i}$ marginals iteratively for each mode, and scales the core tensor and the factors accordingly on the fly. It can be seen as a modified tensor Sinkhorn algorithm where the $\beta_{X_i}$ marginals are learned iteratively from the input factors $X_1,X_2,X_3$ and core $\mathcal{G}$ in order to satisfy Equations~\eqref{eq:scale_NTD_systemW}-~\eqref{eq:scale_NTD_systemQ}. Each minimization step boils down to a balancing with two factors in the spirit of what is showcased for sNMF in~\cref{sec:Showcases}. Balancing the first mode yields
\begin{equation}
    \mathcal{G}[q,:,:] \leftarrow \frac{\beta_{X_1}[q]}{\|\Lambda_2 \mathcal{G}[q,:,:] \Lambda_3 \|_1}\mathcal{G}[q,:,:]  \; \text{and} \; X_1[:,q] \leftarrow \sqrt{\frac{\beta_{X_1}[q]}{2\|X_1[:,q]\|_2^2}}X_1[:,q]
\end{equation}
with $\beta_{X_1} = \left(\sqrt{2}\|X_1\|_F\|\mathcal{G} \times_2 \Lambda_2 \times_3 \Lambda_3 \|_1 \right)^{2/3}$. The other two modes are updated in a similar way. Note that in the alternating scaling algorithm, the scales are not explicitly stored. Rather they are embedded in the core at each iteration because of the scaling update.

As the cost decreases at each iteration and is bounded below, the algorithm is guaranteed to converge. The update rule is well-defined if the factors and the core have non-zero columns/slices. Additionally, since each subproblem has a unique minimizer that is attained, and given that the objective function to minimize in Equation~\eqref{eq:opti_prob_scal_sNTD} is continuously differentiable, every limit point of this block-coordinate descent algorithm is a critical point (see \cite[Proposition~2.7.1]{Bertsekas_1999}). If the objective function is a proper closed function and is continuous over its domain (not necessarily continuously differentiable), then any limit point of the sequence is a coordinate-wise minimum (see \cite[Theorem~14.3]{Beck_2017}).

While this guarantees that the scales of the factors and core will be improved, unlike for the tensor Sinkhorn, we are not able to guarantee that the proposed heuristic computes an optimal scaling although unreported experiments hint towards this direction. The proposed adaptive tensor Sinkhorn algorithm is summarized for the sNTD problem in~\cref{alg:sinkhorn-Algo}. 
Finally, we handle the null components and the entries threshold $\epsilon$ for the scaling as explained in Section~\cref{sec:specificity_beta}.

We did not report the results of the Sinkhorn scaling since they were overall similar to what is obtained with the scalar scaling. The scalar scaling is moreover much easier to implement.

\begin{algorithm}[ht!]
\caption{Balancing algorithm for sNTD \label{alg:sinkhorn-Algo}}
\begin{algorithmic}[1] 
\REQUIRE $X_1,X_2,X_3,\mathcal{G}$, maximal number of iterations itermax.
\STATE Initialize $\beta_{X_1}[q]$, $\beta_{X_2}[q]$ and $\beta_{X_3}[q]$ respectively with $\|X_1[:,q]\|_2^2, \|X_2[:,q]\|_2^2, \|X_3[:,q]\|_2^2$ for all $q$.
\FOR{$i$ = 1 : itermax} 
    \FOR{$q$ = 1:$r_1$}
        \STATE $\beta_{X_1}[q] = \left(\sqrt{2\beta_{X_1}[q]}\|\mathcal{G}[q,:,:]\|_1\right)^{\frac{2}{3}}$
        \STATE $\mathcal{G}[q,:,:] = \frac{\beta_{X_1}[q]}{\|\mathcal{G}[q,:,:]\|_1}\mathcal{G}[q,:,:]$ or $\mathcal{G}[q,:,:] = 0$ if $\beta_{X_1}[q]=0$
    \ENDFOR
    \FOR{$q$ = 1:$r_2$}
        \STATE $\beta_{X_2}[q] = \left(\sqrt{2\beta_{X_2}[q]}\|\mathcal{G}[:,q,:]\|_1\right)^{\frac{2}{3}}$
        \STATE $\mathcal{G}[:,q,:] = \frac{\beta_{X_2}[q]}{\|\mathcal{G}[:,q,:]\|_1}\mathcal{G}[:,q,:]$ or $\mathcal{G}[:,q,:] = 0$ if $\beta_{X_2}[q]=0$
    \ENDFOR
    \FOR{$q$ = 1:$r_3$}
        \STATE $\beta_{X_3}[q] = \left(\sqrt{2\beta_{X_3}[q]}\|\mathcal{G}[:,:,q]\|_1\right)^{\frac{2}{3}}$
        \STATE $\mathcal{G}[:,:,q] = \frac{\beta_{X_3}[q]}{\|\mathcal{G}[:,:,q]\|_1}\mathcal{G}[:,:,q]$ or $\mathcal{G}[:,:,q] = 0$ if $\beta_{X_3}[q]=0$
    \ENDFOR
\ENDFOR
\FOR{$q$ = 1:$r_1$}
    \STATE $X_1[:,q] = \sqrt{\frac{\beta_{X_1}[q]}{2\|X_1[:,q]\|_F^2}}X_1[:,q]$ or $X_1[:,q]=0$ if $\|X_1[:,q]\|^2_2=0$
\ENDFOR
\FOR{$q$ = 1:$r_2$}
    \STATE $X_2[:,q] = \sqrt{\frac{\beta_{X_2}[q]}{2\|X_2[:,q]\|_F^2}}X_2[:,q]$ or $X_2[:,q]=0$ if $\|X_2[:,q]\|^2_2=0$
\ENDFOR
\FOR{$q$ = 1:$r_3$}
    \STATE $X_3[:,q] = \sqrt{\frac{\beta_{X_3}[q]}{2\|X_3[:,q]\|_F^2}}X_3[:,q]$ or $X_3[:,q]=0$ if $\|X_3[:,q]\|^2_2=0$
\ENDFOR
\RETURN $X_1,X_2,X_3,\mathcal{G}$
\end{algorithmic}  
\end{algorithm} 

\section{Additional experiments: Factor match score evaluation}
\Cref{fig:fms} shows the Factor Match Score (FMS)~\cite{Acar2011Scalable} for each experiments conducted in~\cref{sec:Showcases}. There is significant variability in the results, but overall balancing does not generally improve FMS. Therefore when scaling-swamps occur, typically the loss is not well minimized but the estimated components are well estimated up to scaling.

\begin{figure}
    \centering
    \includegraphics[width=5.4cm]{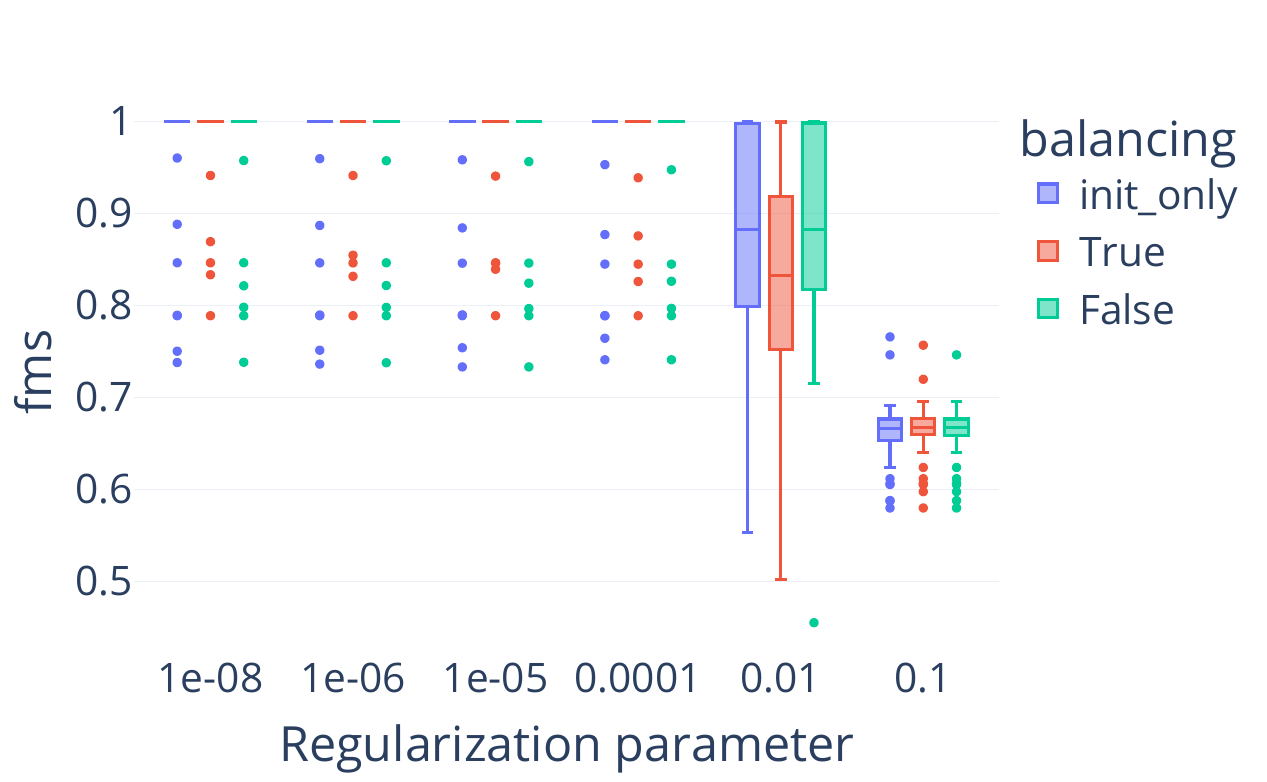}
    \includegraphics[width=5.4cm]{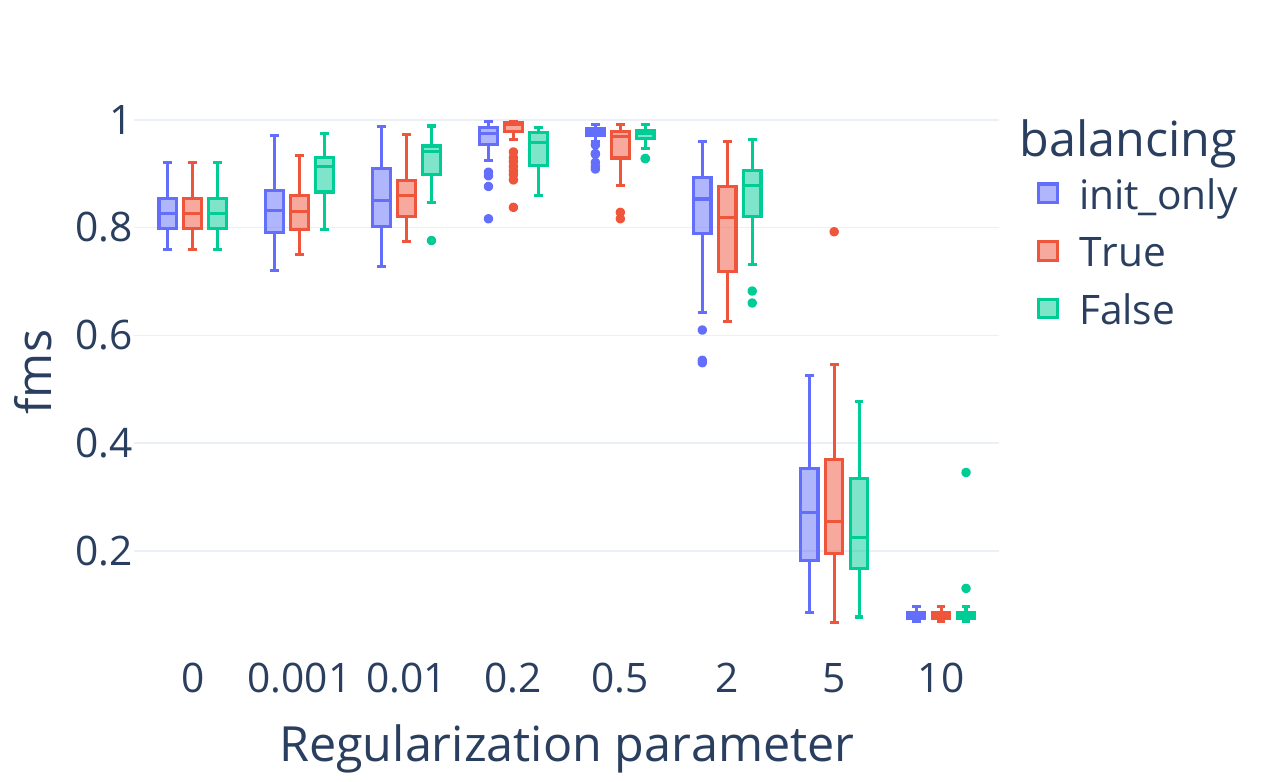}
    \includegraphics[width=5.4cm]{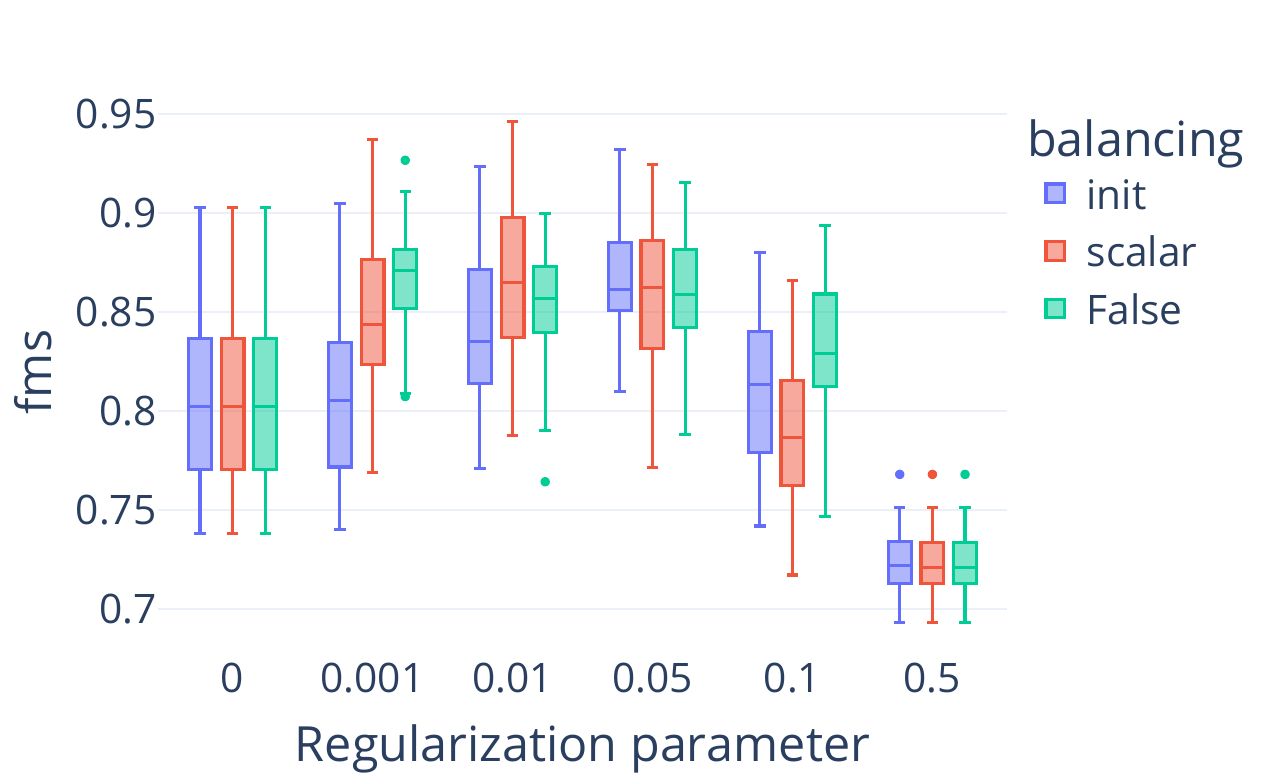}
    \caption{Left: sNMF experiment, center: rNCPD experiment, right: sNTD experiment.}
    \label{fig:fms}
\end{figure}

\bibliographystyle{unsrtnat}
\bibliography{references}  

\begin{thebibliography}{64}
\providecommand{\natexlab}[1]{#1}
\providecommand{\url}[1]{\texttt{#1}}
\expandafter\ifx\csname urlstyle\endcsname\relax
  \providecommand{\doi}[1]{doi: #1}\else
  \providecommand{\doi}{doi: \begingroup \urlstyle{rm}\Url}\fi

\bibitem[Harshman(1970)]{harshman1970foundations}
Richard~A. Harshman.
\newblock Foundations of the {PARAFAC} procedure: Models and conditions for an "explanatory" multimodal factor analysis.
\newblock \emph{UCLA working papers in phonetics}, 16, 1970.

\bibitem[Comon and Jutten(2010)]{Comon2010Handbook}
Pierre Comon and Christian Jutten.
\newblock \emph{Handbook of Blind Source Separation: Independent component analysis and applications}.
\newblock Academic press, 2010.

\bibitem[Lee and Seung(1999)]{Lee1999Learning}
Daniel~D. Lee and H.~Sebastien Seung.
\newblock Learning the parts of objects by non-negative matrix factorization.
\newblock \emph{Nature}, 401\penalty0 (6755):\penalty0 788--791, 1999.

\bibitem[Gillis(2020)]{gillis2020bk}
Nicolas Gillis.
\newblock \emph{Nonnegative Matrix Factorization}.
\newblock SIAM, Philadelphia, 2020.

\bibitem[M{\o}rup et~al.(2008)M{\o}rup, Hansen, and Arnfred]{morup2008algorithms}
Morten M{\o}rup, Lars~Kai Hansen, and Sidse~M Arnfred.
\newblock Algorithms for sparse nonnegative {T}ucker decompositions.
\newblock \emph{Neural computation}, 20\penalty0 (8):\penalty0 2112--2131, 2008.

\bibitem[Kolda and Bader(2009)]{Kolda2009Tensor}
Tamara~G. Kolda and Brett~W. Bader.
\newblock Tensor decompositions and applications.
\newblock \emph{SIAM Review}, 51\penalty0 (3):\penalty0 455--500, 2009.

\bibitem[Tibshirani(2011)]{tibshiraniRegressionShrinkageSelection2011}
Robert Tibshirani.
\newblock Regression {{Shrinkage}} and {{Selection}} via {{The Lasso}}: {{A Retrospective}}.
\newblock \emph{Journal of the Royal Statistical Society Series B: Statistical Methodology}, 73\penalty0 (3):\penalty0 273--282, 2011.
\newblock ISSN 1369-7412, 1467-9868.
\newblock \doi{10.1111/j.1467-9868.2011.00771.x}.
\newblock URL \url{https://academic.oup.com/jrsssb/article/73/3/273/7034363}.

\bibitem[Bogdan et~al.(2015)Bogdan, van~den Berg, Sabatti, Su, and Candès]{bogdanSLOPEADAPTIVEVARIABLE2015}
Małgorzata Bogdan, Ewout van~den Berg, Chiara Sabatti, Weijie Su, and Emmanuel~J. Candès.
\newblock {SLOPE}—adaptive variable selection via convex optimization.
\newblock \emph{The annals of applied statistics}, 9\penalty0 (3):\penalty0 1103--1140, 2015.
\newblock ISSN 1932-6157.
\newblock \doi{10.1214/15-AOAS842}.
\newblock URL \url{https://www.ncbi.nlm.nih.gov/pmc/articles/PMC4689150/}.

\bibitem[Boyer et~al.(2019)Boyer, Chambolle, Castro, Duval, De~Gournay, and Weiss]{boyerRepresenterTheoremsConvex2019}
Claire Boyer, Antonin Chambolle, Yohann~De Castro, Vincent Duval, Frédéric De~Gournay, and Pierre Weiss.
\newblock On {{Representer Theorems}} and {{Convex Regularization}}.
\newblock \emph{SIAM Journal on Optimization}, 29\penalty0 (2):\penalty0 1260--1281, 2019.
\newblock ISSN 1052-6234, 1095-7189.
\newblock \doi{10.1137/18M1200750}.
\newblock URL \url{https://epubs.siam.org/doi/10.1137/18M1200750}.

\bibitem[Unser(2021)]{unserUnifyingRepresenterTheorem2021}
Michael Unser.
\newblock A {{Unifying Representer Theorem}} for {{Inverse Problems}} and {{Machine Learning}}.
\newblock \emph{Foundations of Computational Mathematics}, 21\penalty0 (4):\penalty0 941--960, 2021.
\newblock ISSN 1615-3383.
\newblock \doi{10.1007/s10208-020-09472-x}.
\newblock URL \url{https://doi.org/10.1007/s10208-020-09472-x}.

\bibitem[Tibshirani(2013)]{tibshiraniLassoProblemUniqueness2013}
Ryan~J. Tibshirani.
\newblock The lasso problem and uniqueness.
\newblock \emph{Electronic Journal of Statistics}, 7:\penalty0 1456--1490, 2013.
\newblock ISSN 1935-7524, 1935-7524.
\newblock \doi{10.1214/13-EJS815}.
\newblock URL \url{https://projecteuclid.org/journals/electronic-journal-of-statistics/volume-7/issue-none/The-lasso-problem-and-uniqueness/10.1214/13-EJS815.full}.

\bibitem[Efron et~al.(2004)Efron, Hastie, Johnstone, and Tibshirani]{efronLeastAngleRegression2004}
Bradley Efron, Trevor Hastie, Iain Johnstone, and Robert Tibshirani.
\newblock Least angle regression.
\newblock \emph{The Annals of Statistics}, 32\penalty0 (2):\penalty0 407--499, 2004.
\newblock ISSN 0090-5364, 2168-8966.
\newblock \doi{10.1214/009053604000000067}.
\newblock URL \url{https://projecteuclid.org/journals/annals-of-statistics/volume-32/issue-2/Least-angle-regression/10.1214/009053604000000067.full}.

\bibitem[Cichocki et~al.(2006)Cichocki, Zdunek, and Amari]{cichockiCsiszarDivergencesNonnegative2006}
Andrzej Cichocki, Rafal Zdunek, and Shun-ichi Amari.
\newblock \emph{Csiszár’s {{Divergences}} for {{Non-negative Matrix Factorization}}: {{Family}} of {{New Algorithms}}}, volume 3889.
\newblock Springer, LNCS-3889, 2006.
\newblock \doi{10.1007/11679363_5}.

\bibitem[F{\'e}votte and Idier(2011)]{fevotte2011algorithms}
Cédric F{\'e}votte and Jérome Idier.
\newblock Algorithms for nonnegative matrix factorization with the $\beta$-divergence.
\newblock \emph{Neural computation}, 23\penalty0 (9):\penalty0 2421--2456, 2011.

\bibitem[Le~Roux and Hershey(2015)]{leroux2015sparse}
Felix J.~Weninger Le~Roux, Jonathan and John~R. Hershey.
\newblock Sparse {NMF}--half-baked or well done?
\newblock \emph{Mitsubishi Electric Research Labs (MERL), Cambridge, MA, USA, Tech. Rep., no. TR2015-023}, 11:\penalty0 13--15, 2015.

\bibitem[Bioucas-Dias et~al.(2012)Bioucas-Dias, Plaza, Dobigeon, Parente, Du, Gader, and Chanussot]{Bioucas-Dias2012Hyperspectral}
J.M. Bioucas-Dias, A.~Plaza, N.~Dobigeon, M.~Parente, Q.~Du, P.~Gader, and J.~Chanussot.
\newblock Hyperspectral unmixing overview: {G}eometrical, statistical, and sparse regression-based approaches.
\newblock \emph{IEEE Journal of Selected Topics in Applied Earth Observations and Remote Sensing}, 5\penalty0 (2):\penalty0 354--379, 2012.

\bibitem[Caredda et~al.(2023)Caredda, Cohen, {Mahieu-Williame}, Sablong, Sdika, Guyotat, and Montcel]{careddaSeparableSpectralUnmixing2023}
Charly Caredda, J{\'e}r{\'e}my~E. Cohen, Laurent {Mahieu-Williame}, Raph{\"a}el Sablong, Mich{\"a}el Sdika, Jacques Guyotat, and Bruno Montcel.
\newblock Separable spectral unmixing based on the learning of periodic absorbance changes: {{Application}} to functional brain mapping using {{RGB}} imaging.
\newblock In \emph{Translational {{Biophotonics}}: {{Diagnostics}} and {{Therapeutics III}} (2023), Paper {{126271D}}}, page 126271D. Optica Publishing Group, June 2023.
\newblock \doi{10.1117/12.2670506}.

\bibitem[Marmoret et~al.(2020)Marmoret, Cohen, Bertin, and Bimbot]{marmoret2020uncovering}
Axel Marmoret, Jérémy~E. Cohen, Nancy Bertin, and Frédéric Bimbot.
\newblock Uncovering audio patterns in music with nonnegative {T}ucker decomposition for structural segmentation.
\newblock In \emph{ISMIR 2020-21st International Society for Music Information Retrieval}, 2020.

\bibitem[Bro(1998)]{Bro1998Multi}
R.~Bro.
\newblock \emph{\textit{Multi-way Analysis in the Food Industry: Models, Algorithms, and Applications}}.
\newblock PhD thesis, University of Amsterdam, The Netherlands, 1998.

\bibitem[Roald et~al.(2022)Roald, Schenker, Calhoun, Adalı, Bro, Cohen, and Acar]{roaldAOADMMApproachConstraining2022}
Marie Roald, Carla Schenker, Vince~D. Calhoun, Tülay Adalı, Rasmus Bro, Jeremy~E. Cohen, and Evrim Acar.
\newblock An {{AO-ADMM}} approach to constraining {{PARAFAC2}} on all modes.
\newblock \emph{SIAM Journal on Mathematics of Data Science}, 4\penalty0 (3):\penalty0 1191--1222, 2022.
\newblock ISSN 2577-0187.
\newblock \doi{10.1137/21M1450033}.
\newblock URL \url{http://arxiv.org/abs/2110.01278}.

\bibitem[Sidiropoulos et~al.(2000)Sidiropoulos, Bro, and Giannakis]{Sidiropoulos2000Parallel}
N.~D. Sidiropoulos, R.~Bro, and G.~B. Giannakis.
\newblock Parallel factor analysis in sensor array processing.
\newblock \emph{IEEE Trans. Sig. Proc.}, 48\penalty0 (8):\penalty0 2377--2388, aug 2000.

\bibitem[Kruskal(1977)]{Kruskal1977Three}
Joseph~B Kruskal.
\newblock Three-way arrays: rank and uniqueness of trilinear decompositions, with application to arithmetic complexity and statistics.
\newblock \emph{Linear algebra and its applications}, 18\penalty0 (2):\penalty0 95--138, 1977.

\bibitem[Cichocki et~al.(2009)Cichocki, Zdunek, Phan, and Amari]{Cichocki2009Nonnegative}
Andrzej Cichocki, Rafal Zdunek, Anh-Huy Phan, and Shun-Ichi Amari.
\newblock \emph{Nonnegative Matrix and Tensor Factorization}.
\newblock Wiley, 2009.

\bibitem[Lim and Comon(2009)]{Lim2009Nonnegative}
Lek-Heng Lim and Pierre Comon.
\newblock Nonnegative approximations of nonnegative tensors.
\newblock \emph{Journal of chemometrics}, 23\penalty0 (7-8):\penalty0 432--441, 2009.

\bibitem[Mohlenkamp(2019)]{mohlenkampDynamicsSwampsCanonical2019}
Martin~J. Mohlenkamp.
\newblock The {{Dynamics}} of {{Swamps}} in the {{Canonical Tensor Approximation Problem}}.
\newblock \emph{SIAM Journal on Applied Dynamical Systems}, 18\penalty0 (3):\penalty0 1293--1333, January 2019.
\newblock ISSN 1536-0040.
\newblock \doi{10.1137/18M1181389}.

\bibitem[Marmin et~al.(2023)Marmin, Goulart, and F{\'e}votte]{marminMajorizationMinimizationSparseNonnegative2023}
Arthur Marmin, Jos{\'e} Henrique de~Morais Goulart, and C{\'e}dric F{\'e}votte.
\newblock Majorization-{{Minimization}} for {{Sparse Nonnegative Matrix Factorization With}} the \${\textbackslash}beta\$-{{Divergence}}.
\newblock \emph{IEEE Transactions on Signal Processing}, 71:\penalty0 1435--1447, January 2023.
\newblock ISSN 1053-587X.
\newblock \doi{10.1109/TSP.2023.3266939}.

\bibitem[Srebro et~al.(2004)Srebro, Rennie, and Jaakkola]{srebroMaximumMarginMatrixFactorization2004}
Nathan Srebro, Jason Rennie, and Tommi Jaakkola.
\newblock Maximum-{{Margin Matrix Factorization}}.
\newblock In \emph{Advances in {{Neural Information Processing Systems}}}, volume~17. {MIT Press}, 2004.
\newblock URL \url{https://proceedings.neurips.cc/paper/2004/hash/e0688d13958a19e087e123148555e4b4-Abstract.html}.

\bibitem[Yuan and Lin(2006)]{yuanModelSelectionEstimation2006}
Ming Yuan and Yi~Lin.
\newblock Model selection and estimation in regression with grouped variables.
\newblock \emph{Journal of the Royal Statistical Society: Series B (Statistical Methodology)}, 68\penalty0 (1):\penalty0 49--67, 2006.
\newblock ISSN 1467-9868.
\newblock \doi{10.1111/j.1467-9868.2005.00532.x}.
\newblock URL \url{https://onlinelibrary.wiley.com/doi/abs/10.1111/j.1467-9868.2005.00532.x}.

\bibitem[Papalexakis and Bro(2013)]{Papalexakis2013From}
Nicholas D.~Sidiropoulos Papalexakis, Evangelos~E. and Rasmus Bro.
\newblock From k-means to higher-way co-clustering: Multilinear decomposition with sparse latent factors.
\newblock \emph{IEEE transactions on signal processing}, 61\penalty0 (2):\penalty0 493--506, 2013.

\bibitem[Benichoux et~al.(2013)Benichoux, Vincent, and Gribonval]{benichouxFundamentalPitfallBlind2013}
Alexis Benichoux, Emmanuel Vincent, and Remi Gribonval.
\newblock A fundamental pitfall in blind deconvolution with sparse and shift-invariant priors.
\newblock In \emph{{{IEEE International Conference}} on {{Acoustics}}, {{Speech}} and {{Signal Processing}}}, pages 6108--6112, 2013.
\newblock ISBN 978-1-4799-0356-6.
\newblock \doi{10.1109/ICASSP.2013.6638838}.
\newblock URL \url{http://ieeexplore.ieee.org/document/6638838/}.

\bibitem[M{\o}rup and Hansen(2009)]{morupAutomaticRelevanceDetermination2009}
Morten M{\o}rup and Lars~Kai Hansen.
\newblock Automatic relevance determination for multi-way models.
\newblock \emph{Journal of Chemometrics}, 23\penalty0 (7-8):\penalty0 352--363, 2009.
\newblock ISSN 1099-128X.
\newblock \doi{10.1002/cem.1223}.

\bibitem[Tan and F{\'e}votte(2013)]{tanAutomaticRelevanceDetermination2013}
Vincent~Y.F. Tan and C{\'e}dric F{\'e}votte.
\newblock Automatic {{Relevance Determination}} in {{Nonnegative Matrix Factorization}} with the /spl beta/-{{Divergence}}.
\newblock \emph{IEEE Transactions on Pattern Analysis and Machine Intelligence}, 35\penalty0 (7):\penalty0 1592--1605, July 2013.
\newblock ISSN 1939-3539.
\newblock \doi{10.1109/TPAMI.2012.240}.

\bibitem[Uschmajew(2012)]{Uschmajew2012Local}
André Uschmajew.
\newblock Local convergence of the alternating least squares algorithm for canonical tensor approximation.
\newblock \emph{SIAM J. matrix Analysis}, 33\penalty0 (2):\penalty0 639--652, 2012.

\bibitem[Neyshabur et~al.(2015)Neyshabur, Tomioka, and Srebro]{NeyshaburTS14}
Behnam Neyshabur, Ryota Tomioka, and Nathan Srebro.
\newblock In search of the real inductive bias: On the role of implicit regularization in deep learning.
\newblock In \emph{3rd International Conference on Learning Representations, {ICLR} 2015, San Diego, CA, USA, Workshop Track Proceedings}, 2015.
\newblock URL \url{http://arxiv.org/abs/1412.6614}.

\bibitem[Tibshirani(2021)]{tibshirani2021equivalences}
Ryan~J Tibshirani.
\newblock Equivalences between sparse models and neural networks.
\newblock \emph{Working Notes. URL https://www. stat. cmu. edu/ryantibs/papers/sparsitynn. pdf}, 2021.

\bibitem[Ergen and Pilanci(2021)]{ergenImplicitConvexRegularizers2020}
Tolga Ergen and Mert Pilanci.
\newblock Implicit convex regularizers of {\{}cnn{\}} architectures: Convex optimization of two- and three-layer networks in polynomial time.
\newblock In \emph{International Conference on Learning Representations}, 2021.
\newblock URL \url{https://openreview.net/forum?id=0N8jUH4JMv6}.

\bibitem[Stock et~al.(2019)Stock, Graham, Gribonval, and Jégou]{stock2018equinormalization}
Pierre Stock, Benjamin Graham, Rémi Gribonval, and Hervé Jégou.
\newblock Equi-normalization of neural networks.
\newblock In \emph{International Conference on Learning Representations}, 2019.
\newblock URL \url{https://openreview.net/forum?id=r1gEqiC9FX}.

\bibitem[Gribonval et~al.(2023)Gribonval, Mary, and Riccietti]{gribonvalScalingAllYou}
Rémi Gribonval, Theo Mary, and Elisa Riccietti.
\newblock Scaling is all you need: quantization of butterfly matrix products via optimal rank-one quantization.
\newblock In \emph{29ème Colloque sur le traitement du signal et des images (GRETSI), Aug 2023, Grenoble, France.}, pages pp. 497--500, 2023.

\bibitem[Parhi and Nowak(2023)]{parhiDeepLearningMeets2023}
Rahul Parhi and Robert~D. Nowak.
\newblock Deep {{Learning Meets Sparse Regularization}}: {{A}} signal processing perspective.
\newblock \emph{IEEE Signal Processing Magazine}, 40\penalty0 (6):\penalty0 63--74, September 2023.
\newblock ISSN 1558-0792.
\newblock \doi{10.1109/MSP.2023.3286988}.

\bibitem[Bach et~al.(2011)Bach, Jenatton, Mairal, and Obozinski]{bachOptimizationSparsityInducingPenalties2011}
Francis Bach, Rodolphe Jenatton, Julien Mairal, and Guillaume Obozinski.
\newblock Optimization with {{Sparsity-Inducing Penalties}}.
\newblock \emph{Foundations and Trends{\textregistered} in Machine Learning}, 4\penalty0 (1):\penalty0 1--106, 2011.

\bibitem[F{\'e}votte et~al.(2009)F{\'e}votte, Bertin, and Durrieu]{fevotte2009nonnegative}
C.~F{\'e}votte, N.~Bertin, and J-L. Durrieu.
\newblock Nonnegative matrix factorization with the {I}takura-{S}aito divergence: With application to music analysis.
\newblock \emph{Neural computation}, 21\penalty0 (3):\penalty0 793--830, 2009.

\bibitem[Basu et~al.(1998)Basu, Harris, Hjort, and Jones]{basu1998robust}
Ayanendranath Basu, Ian~R Harris, Nils~L Hjort, and MC~Jones.
\newblock Robust and efficient estimation by minimising a density power divergence.
\newblock \emph{Biometrika}, 85\penalty0 (3):\penalty0 549--559, 1998.

\bibitem[Eguchi and Kano(2001)]{eguchi2001robustifying}
Shinto Eguchi and Yutaka Kano.
\newblock Robustifying maximum likelihood estimation.
\newblock \emph{Tokyo Institute of Statistical Mathematics, Tokyo, Japan, Tech. Rep}, 2001.

\bibitem[Lef\`{e}vre(2012)]{Lefevre_phd}
A.~Lef\`{e}vre.
\newblock \emph{M\'{e}thode d'apprentissage de dictionnaire pour la s\'{e}paration de sources audio avec un seul capteur}.
\newblock PhD thesis, Ecole Normale Sup\'{e}rieure de Cachan, 2012.

\bibitem[Leplat et~al.(2021)Leplat, Gillis, and Idier]{doi:10.1137/20M1377278}
Valentin Leplat, Nicolas Gillis, and Jérôme Idier.
\newblock Multiplicative updates for nmf with beta-divergences under disjoint equality constraints.
\newblock \emph{SIAM Journal on Matrix Analysis and Applications}, 42\penalty0 (2):\penalty0 730--752, 2021.
\newblock \doi{10.1137/20M1377278}.

\bibitem[Gillis and Glineur(2012)]{Gillis2012Accelerated}
Nicolas Gillis and Francois Glineur.
\newblock Accelerated multiplicative updates and hierarchical {ALS} algorithms for nonnegative matrix factorization.
\newblock \emph{Neural computation}, 24\penalty0 (4):\penalty0 1085--1105, 2012.

\bibitem[Razaviyayn et~al.(2013)Razaviyayn, Hong, and Luo]{RHL_BSUM}
Meisam Razaviyayn, Mingyi Hong, and Zhi-Quan Luo.
\newblock A unified convergence analysis of block successive minimization methods for nonsmooth optimization.
\newblock \emph{SIAM Journal on Optimization}, 23\penalty0 (2):\penalty0 1126--1153, 2013.

\bibitem[Hien and Gillis(2021)]{HienNicolasKLNMF}
L.~T.~K. Hien and N.~Gillis.
\newblock Algorithms for nonnegative matrix factorization with the kullback-leibler divergence.
\newblock \emph{Journal of Scientific Computing}, \penalty0 (87):\penalty0 93, 2021.

\bibitem[Takahashi and Ryota(2014)]{Takahashi2014}
T.~Takahashi and H.~Ryota.
\newblock Global convergence of modified multiplicative updates for nonnegative matrix factorization.
\newblock \emph{Computational Optimization and Applications}, \penalty0 (57):\penalty0 417–440, 2014.

\bibitem[Kossaifi et~al.(2019)Kossaifi, Panagakis, Anandkumar, and Pantic]{kossaifi2019tensorly}
Jean Kossaifi, Yannis Panagakis, Anima Anandkumar, and Maja Pantic.
\newblock Tensorly: Tensor learning in python.
\newblock \emph{The Journal of Machine Learning Research}, 20\penalty0 (1):\penalty0 925--930, 2019.

\bibitem[Cohen(2023)]{cohenRegularisationImpliciteFactorisations2023}
Jérémy~E Cohen.
\newblock Régularisation implicite des factorisations de faible rang pénalisées.
\newblock \emph{GRETSI 2023, Grenoble}, 2023.

\bibitem[Smith and Goto(2018)]{smithNonnegativeTensorFactorization2018}
Jordan B.~L. Smith and Masataka Goto.
\newblock Nonnegative {{Tensor Factorization}} for {{Source Separation}} of {{Loops}} in {{Audio}}.
\newblock In \emph{2018 {{IEEE International Conference}} on {{Acoustics}}, {{Speech}} and {{Signal Processing}} ({{ICASSP}})}, pages 171--175, 2018.
\newblock \doi{10.1109/ICASSP.2018.8461876}.
\newblock URL \url{https://ieeexplore.ieee.org/abstract/document/8461876?casa_token=SzYH8C--TmQAAAAA:8kHygCQIvSDP8U0pWsE9NTwBigBV4_3IJ0hmEq4rSc9reVFZtlLBvkELXeWaLIbgItiqAZnRBN8Z}.

\bibitem[De~Lathauwer and Vandewalle(2004)]{DeLathauwer2004Dimensionality}
Lieven De~Lathauwer and Joos Vandewalle.
\newblock Dimensionality reduction in higher-order signal processing and rank-(${R}_1, {R}_2,…, {R}_n$) reduction in multilinear algebra.
\newblock \emph{Linear Algebra and its Applications}, 391:\penalty0 31--55, 2004.

\bibitem[Rechtschaffen(2008)]{rootpoly}
Edgar Rechtschaffen.
\newblock Real roots of cubics: explicit formula for quasi-solutions.
\newblock \emph{The Mathematical Gazette}, \penalty0 (524):\penalty0 268–276, 2008.

\bibitem[Hoyer(2002)]{Hoyer2002Non}
Patrik~O. Hoyer.
\newblock Non-negative sparse coding.
\newblock In \emph{IEEE Workshop on Neural Networks for Signal Processing}, pages 557--565. IEEE, 2002.

\bibitem[Nesterov(2013)]{nesterov2013introductory}
Yurii Nesterov.
\newblock \emph{Introductory lectures on convex optimization: A basic course}.
\newblock Springer Science \& Business Media, 2013.

\bibitem[Sinkhorn and Knopp(1967)]{Sinkhorn_Knopp_1967}
Richard~Dennis Sinkhorn and Paul~Joseph Knopp.
\newblock Concerning nonnegative matrices and doubly stochastic matrices.
\newblock \emph{Pacific Journal of Mathematics}, 27\penalty0 (2):\penalty0 343–348, 1967.

\bibitem[Michael(1970)]{BACHARACH}
Bacharach Michael.
\newblock \emph{Biproportional Matrices and Input-Output Change}.
\newblock Cambridge University Press, 1970.

\bibitem[Lamond and Stewart(1981)]{LAMOND1981239}
B.~Lamond and N.F. Stewart.
\newblock Bregman's balancing method.
\newblock \emph{Transportation Research Part B: Methodological}, 15\penalty0 (4):\penalty0 239--248, 1981.
\newblock ISSN 0191-2615.

\bibitem[Csiszar(1975)]{Csiszar_1975}
I.~Csiszar.
\newblock $i$-divergence geometry of probability distributions and minimization problems.
\newblock \emph{The Annals of Probability}, 3\penalty0 (1):\penalty0 146--158, 1975.

\bibitem[Sugiyama et~al.(2017)Sugiyama, Nakahara, and Tsuda]{pmlr-v70-sugiyama17a}
Mahito Sugiyama, Hiroyuki Nakahara, and Koji Tsuda.
\newblock Tensor balancing on statistical manifold.
\newblock In Doina Precup and Yee~Whye Teh, editors, \emph{Proceedings of the 34th International Conference on Machine Learning}, volume~70 of \emph{Proceedings of Machine Learning Research}, pages 3270--3279. PMLR, 06--11 Aug 2017.

\bibitem[Bertsekas(1999)]{Bertsekas_1999}
Dimitri~P. Bertsekas.
\newblock \emph{Nonlinear Programming}.
\newblock Athena Scientific, Belmont, MA, 1999.

\bibitem[Beck(2017)]{Beck_2017}
Amir Beck.
\newblock \emph{First-Order Methods in Optimization}.
\newblock Society for Industrial and Applied Mathematics, Philadelphia, PA, 2017.

\bibitem[Acar et~al.(2011)Acar, Dunlavy, Kolda, and M{\o}rup]{Acar2011Scalable}
Evrim Acar, Daniel~M. Dunlavy, Tamara~G. Kolda, and Morten M{\o}rup.
\newblock Scalable tensor factorizations for incomplete data.
\newblock \emph{Chemometrics and Intelligent Laboratory Systems}, 2011.

\end{thebibliography}






\end{document}